\documentclass{article}

\usepackage{iclr2027_conference}
\usepackage{times}
\iclrfinalcopy

\usepackage[utf8]{inputenc} 
\usepackage[T1]{fontenc}    
\usepackage{hyperref}       
\usepackage{url}            
\usepackage{booktabs}       
\usepackage{amsfonts}       
\usepackage{nicefrac}       
\usepackage{tabularx}       
\usepackage{microtype}      
\usepackage{xcolor}         

\usepackage{appendix}
\usepackage{titletoc}
\usepackage{amsmath,amssymb,amsfonts,amsthm}
\usepackage{mathtools}
\usepackage{bm}
\usepackage{tikz}
\usetikzlibrary{arrows.meta}

\usepackage{graphicx}
\usepackage{subcaption}
\usepackage{booktabs}   
\usepackage{microtype}
\usepackage{needspace}
\usepackage{siunitx} 

\usepackage{hyperref}

\usepackage{enumitem}
\usepackage[capitalize,noabbrev]{cleveref}
\usepackage{placeins}

\usepackage{amsmath, amsthm}

\usepackage{multirow}

\usepackage{placeins}
\usepackage{dblfloatfix}
\usepackage{float}
\usepackage{wrapfig}

\usepackage{siunitx}
\usepackage{algorithm}
\usepackage{algorithmic}

\newcommand\DoToC{%
  \startcontents
  \printcontents{}{2}{\textbf{Contents}\vskip3pt\hrule\vskip5pt}
  \vskip3pt\hrule\vskip5pt
}

\setlist[itemize]{noitemsep, nolistsep}

\theoremstyle{plain}
\newtheorem{theorem}{Theorem}[section]
\newtheorem{proposition}[theorem]{Proposition}
\newtheorem{example}[theorem]{Example}

\theoremstyle{definition}
\newtheorem{definition}[theorem]{Definition}
\newtheorem{assumption}[theorem]{Assumption}
\theoremstyle{remark}

\newcommand{\R}{\mathbb{R}}
\newcommand{\E}{\mathbb{E}}
\newcommand{\Var}{\mathrm{Var}}
\newcommand{\Cov}{\mathrm{Cov}}

\newcommand{\ip}[2]{\left\langle #1,\,#2\right\rangle}

\newcommand{\gr}[1]{\textcolor{black}{#1}}

\newcommand{\gdr}[1]{\textcolor{black}{#1}}
\newcommand{\mk}[1]{\textcolor{black}{#1}}

\title{TASTE: Task-Aware Out-of-Distribution Detection via Stein Operators}

\author{Micha\l\ Kozyra \\
University of Oxford \\
\texttt{michal.kozyra@seh.ox.ac.uk}
\And
Gesine Reinert \\
University of Oxford\\
\texttt{reinert@stats.ox.ac.uk}
}

\begin{document}

\maketitle
\lhead{}
\renewcommand{\headrulewidth}{0pt}
\begin{abstract}

Out-of-distribution  detection methods are often either data-centric,
detecting deviations from the training input distribution irrespective of their
effect on a trained model, or model-centric, relying on classifier outputs
without explicit reference to data geometry.  We propose TASTE (Task-Aware STEin operators): a task-aware framework
which links distribution shift to the input sensitivity of the model. 
The resulting operators admit a 
geometric interpretation as
a projection of distribution shift onto the  sensitivity field of the model, with theoretical guarantees. Beyond detecting \gr{a shift}, 
the 
construction enables its localisation through a coordinate-wise
decomposition, and - for image data - provides interpretable per-pixel
diagnostics. Experiments on controlled Gaussian shifts, MNIST under geometric perturbations, and CIFAR-10\mk{/ImageNet-1K} perturbed benchmarks demonstrate that
TASTE aligns closely with task degradation while showing competitive performance compared to 
 established baselines, with gains that depend on the shift regime.
\end{abstract}

\section{Introduction}

Deep models are typically trained under the assumption that training and
test samples are i.i.d.\ from a fixed joint distribution $P(X,Y)$.  At deployment, this assumption frequently fails: the test-time marginal $Q_X$ may differ due to
changes in, for example,  sensors or 
preprocessing pipelines.
Such
shifts are commonly categorised as \emph{covariate shift} ($Q_X\neq P_X$ with
$P(Y\!\mid\!X)$ unchanged), \emph{label shift} (changes in $P(Y)$), and, in
streaming or non-stationary settings, broader forms of \emph{concept/conditional}
shift affecting $P(Y\!\mid\!X)$
\citep{quinonero-candelaDatasetShiftMachine2008,gamaSurveyConceptDrift2014,luLearningConceptDrift2018}.
A central practical question is therefore not only whether $Q_X\neq P_X$, but
whether the shift \emph{meaningfully affects} the behaviour of the model $f_\theta$ of interest.

{\textbf{Limitations of current OOD approaches.}}
Out-of-distribution (OOD) detection methods usually approach this question from
two complementary angles.  \emph{Data-centric} methods model the training
marginal $P_X$ and flag inputs that appear unlikely, while \emph{model-centric}
methods operate on classifier outputs (e.g.\ softmax confidence) or learned
representations (e.g.\ feature-space distances); see for example 
\cite{hendrycksBaselineDetectingMisclassified2018,liangEnhancingReliabilityOutofdistribution2020,liuEnergybasedOutofdistributionDetection2021,nalisnickDeepGenerativeModels2019}.
Each captures part of the picture: data-centric scores can detect distribution
change but do not specify whether it matters for the task, whereas model-centric
scores can reveal brittleness yet do not explicitly connect $f_\theta$ to the
geometry of the input distribution
\citep{rabanserFailingLoudlyEmpirical2019,yangGeneralizedOutofDistributionDetection2024}.
What is missing is a principled mechanism that links the \emph{geometry of the data} to
the \emph{sensitivity of the model} in a way that remains usable at test time. 

\textbf{Our perspective: Stein operators.} To bridge data geometry and model sensitivity we propose a novel way of using Stein operators.
Stein's method compares a candidate distribution $q$ to a
reference distribution $p$ via an operator $\mathcal{T}_p$ satisfying
$\mathbb{E}_p[\mathcal{T}_p g]=0$ for a rich class of test functions $g$
\citep{stein1972bound,leySteinsMethodComparison2016}.
Departures $\mathbb{E}_q[\mathcal{T}_p g]\neq 0$ certify $q\neq p$.  
Stein operators underpin discrepancy measures and goodness-of-fit
tests such as kernelised Stein discrepancies, and more broadly provide a
computational interface between score information and distributional mismatch\gr{, see the surveys}
\cite{anastasiouSteinsMethodMeets2022, liu2026probabilistic}. 
Crucially, most Stein-based discrepancy tests treat $g$ as a free element of a function
class and optimise or randomise over it to obtain \emph{distribution-level}
tests.  In contrast, in deployment we already have a trained model
$f_\theta(x)\approx \mathbb{E}[Y\mid X=x]$, and we are specifically interested
in how a distribution shift affects \emph{this} function.  This motivates a
model-centric question: What can we say about distribution shift and reliability 
if we fix the
Stein test function to be the deployed predictor $f_\theta$ itself?

\vspace{-1.5cm}
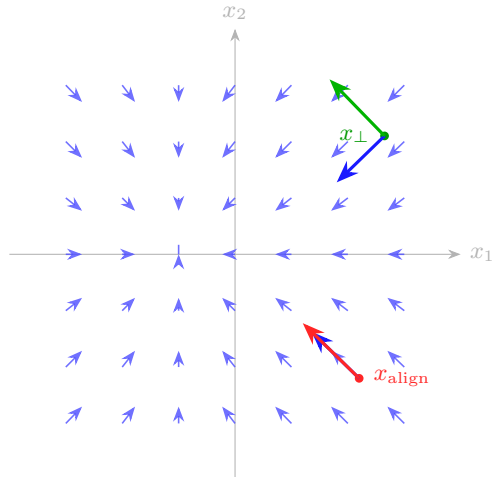
\begin{wrapfigure}{r}{0.5\linewidth}
  \vspace{-2.6em}
  \centering
\centering
\begin{tikzpicture}[scale=0.75, >=Stealth, every node/.style={font=\footnotesize}]

  \draw[->, gray!60] (-3.5,0) -- (3.5,0) node[right] {$x_1$};
  \draw[->, gray!60] (0,-3.5) -- (0,3.5) node[above] {$x_2$};

  \def\fieldscale{0.30}  
  \foreach \x in {-3,-2,-1,0,1,2,3} {
    \foreach \y in {-3,-2,-1,0,1,2,3} {
      \pgfmathsetmacro{\sx}{-tanh(\x+1)}
      \pgfmathsetmacro{\sy}{-tanh(\y)}
      \draw[->, blue!70, line width=0.6pt, opacity=0.80]
        (\x,\y) -- ++({\fieldscale*\sx},{\fieldscale*\sy});
    }
  }


  \coordinate (XR) at (2.2,-2.2);
  \fill[magenta!85] (XR) circle (2.2pt);
  \node[magenta!85, right=2pt] at (XR) {$x_{\mathrm{align}}$};

  \pgfmathsetmacro{\sxR}{-tanh(2.2+1)}
  \pgfmathsetmacro{\syR}{-tanh(-2.2)}
  \draw[->, very thick, blue!90]
    (XR) -- ++({0.85*\sxR},{0.85*\syR});

  \pgfmathsetmacro{\normR}{sqrt((\sxR)*(\sxR)+(\syR)*(\syR))}
  \pgfmathsetmacro{\gxR}{(\sxR)/\normR}
  \pgfmathsetmacro{\gyR}{(\syR)/\normR}
  \draw[->, very thick, magenta!85]
    (XR) -- ++({1.4*\gxR},{1.4*\gyR});

  \coordinate (XG) at (2.65,2.1);
  \fill[brown!60!black] (XG) circle (2.2pt);
  \node[brown!60!black, left=2pt] at (XG) {$x_{\perp}$};

  \pgfmathsetmacro{\sxG}{-tanh(2.65+1)}
  \pgfmathsetmacro{\syG}{-tanh(2.1)}
  \draw[->, very thick, blue!90]
    (XG) -- ++({0.85*\sxG},{0.85*\syG});

  \pgfmathsetmacro{\normG}{sqrt((\sxG)*(\sxG)+(\syG)*(\syG))}
  \pgfmathsetmacro{\pgx}{(\syG)/\normG}
  \pgfmathsetmacro{\pgy}{ -(\sxG)/\normG}
  \draw[->, very thick, brown!70!black]
    (XG) -- ++({1.4*\pgx},{1.4*\pgy});
\end{tikzpicture}
\caption{\textbf{Task-aware intuition behind TASTE.}
Blue \gr{arrows show} 
the score-shift field between the training density $p$ and test density $q$. Without explicit knowledge of $q$, 
TASTE measures 
\gr{alignment between} 
distribution shift 
\gr{and} 
the
input sensitivity of the model $\nabla f
$: shifts aligned with $\nabla f$ (\gr{magenta}) produce a large
response; orthogonal shifts (\gr{brown}) are largely ignored even 
\gr{when} 
large.
}
\label{fig:score-shift-aligned-orth}
\end{wrapfigure}
\vspace{1.5cm}

\textbf{TASTE: Task-Aware Stein functional.}
Fixing the test function $f:\R^d\to\R$  of interest instead of considering 
a wide class of test functions is a novel use of Stein's method. We use it 
to 
develop \textbf{TASTE}, a task-aware framework built around the Langevin Stein operator
\begin{equation}
 \label{eq:langevin-stein-operator-main}
  \mathcal{L}_p f(x)=\Delta f(x)+\nabla\log p(x)^\top\nabla f(x),  
\end{equation}
constructed from a (potentially unnormalised) density/score model $p$ for
training data and a fixed $f$. The 
{\it score function} 
$\gr{s_p} = \nabla\log p$ in \eqref{eq:langevin-stein-operator-main} connects the operator to score-based views of data geometry and
score-matching models \citep{songScoreBasedGenerativeModeling2021}.
For each sample $X$, the value $\mathcal{L}_p f(X)$ gives a {\it per-sample Stein residual}, with mean 0 under  $p$.
We study what we call the {\it task-aware Stein, or TASTE,  functional} 
\begin{align}\label{eq:stein-functional-def-main}
    S_f(p,q):=\mathbb{E}_q[\mathcal{L}_p f(X)]
\end{align}
under a test distribution $q$.
Under mild
regularity assumptions $S_f(p,q)$ admits an explicit ``projection'' form (Eq. \eqref{eq:projection-identity-main} \gr{below}),
which can be interpreted as an $L^2(q)$ inner product between
the sensitivity field $\nabla f$ of the model, and
the shift-score field $\nabla\log(q/p)$.  This identity 
shows 
the
task-aware nature of the signal: 
for fixed $f$, $S_f(p,q)$ 
responds to the components of distribution shift that influence 
$\nabla f$ in expectation;  
see Figure~\ref{fig:score-shift-aligned-orth}.

We focus on the Langevin Stein operator as it provides a scalar functional with a 
geometric interpretation. While simpler first-order Stein operators can also be constructed, they exhibit limitations such as directional blind spots and loss of linear sensitivity; see 
Appendix (App.) \ref{app:first-order-stein}.

\textbf{Practical appeal of TASTE: plug-and-play, model-agnostic, interpretable.}
Our approach is \emph{plug-and-play}: at test time it combines any model
$f$ with any score/density estimator for $p$ and produces an
OOD signal without retraining the classifier or requiring negative samples.  It
is \emph{model-agnostic} in the sense that it only requires access to
input-gradients of $f$ (available via autodiff)
and an estimate of $\nabla\log p$ (the quality of this score estimate influencing performance).  Finally, the operator admits a natural
\emph{per-input-dimension} decomposition; for images this yields per-pixel
residual maps, providing interpretable localisation of task-relevant anomalies
and enabling fine-grained diagnostic visualisations.

\textbf{Contributions.}
Our main contributions are as follows.
\begin{itemize}[noitemsep, leftmargin=*]
  \item \textbf{TASTE: model-centric Stein diagnostics.}
        We formalise a task-aware TASTE functional by applying the Langevin Stein
        operator to a fixed model, and derive a projection identity that
        links distribution shift to the  sensitivity field of the model.

  \item \textbf{Robustness to score-model error.}
        We analyse how inaccuracies in learned score models affect the proposed methodology and provide a 
        correction that removes systematic bias
        using training data.

  \item \textbf{Per-sample TASTE scores and interpretable decompositions.}
        We derive per-sample Stein residuals as OOD scores and show how their
        coordinate-wise decomposition yields per-input (and per-pixel) anomaly maps
        for image data.

  \item \textbf{Empirical validation across diverse shift regimes.}
        We demonstrate that the Stein signal aligns with task degradation in
        controlled shift models and achieves strong performance against standard
        OOD baselines on common benchmark suites (including corruption and
        perturbation benchmarks, mixed in-/out-distribution test sets, and
        adversarially perturbed inputs).
\end{itemize}
\textbf{Scope.}
The goal of this paper is not to propose yet another OOD score in isolation,
but to articulate a general framework in which a fixed $f_\theta$ and
a density/score model for $p(x)$ are linked via a principled operator.  Within
this framework, TASTE functionals act as task-aware probes of distribution
shift, offering a bridge between generative and discriminative views of
uncertainty in deep learning.
\section{Related Work}
\label{app:lit}

The dataset-shift literature distinguishes changes in $P_X$, changes in $P_Y$,
and changes in $P(Y\mid X)$, and studies both detection and adaptation under
non-stationarity
\citep{quinonero-candelaDatasetShiftMachine2008,gamaSurveyConceptDrift2014,luLearningConceptDrift2018}.
In modern OOD detection, the dominant focus is on \emph{detecting} atypical test
inputs and quantifying separation between in- and out-distribution samples,
with a growing emphasis on unified taxonomies and evaluation protocols
\citep{yangGeneralizedOutofDistributionDetection2024,rabanserFailingLoudlyEmpirical2019}.
Our work complements this direction by emphasising \emph{task relevance}: the
goal is not to detect all deviations from $P_X$, but to detect distribution shifts that influence the deployed model.

Data-centric methods use density or score models to assess novelty under $P_X$,
yet the likelihood can be misaligned with semantic novelty in high-dimensional image
data \citep{nalisnickDeepGenerativeModels2019,renLikelihoodRatiosOutofDistribution2019}.
Model-centric methods instead use confidence and calibration heuristics or
representation distances, 
\gr{such as} Maximum Softmax Probability (MSP)-style baselines, and methods that shape
output energies or logits to improve separability
\citep{hendrycksBaselineDetectingMisclassified2018,liangEnhancingReliabilityOutofdistribution2020,liuEnergybasedOutofdistributionDetection2021}.

Recent post-hoc and hybrid methods  such as  ReAct \citep{sun2021react}, ViM \citep{wang2022vim}, ASH
\citep{djurisic2023ash}, PRO \citep{chen2025pro}, 
and SCOPED
\citep{barkley2025scoped} improve OOD
separability through feature shaping, logit correction, perturbation sensitivity,
or score-based geometry.
Our Stein construction is 
a hybrid: it uses a learned score model
to represent data geometry while 
coupling it to the
input-sensitivity of the model, yielding a task-aware signal. 

\gr{M}any applications require localisation: identifying
which features or pixels contribute to abnormality.  In computer vision, pixel-level anomaly localisation has been studied extensively 
\citep{bergmannUninformedStudentsStudentTeacher2020,rothTotalRecallIndustrial2022}.
Separately, feature-level shift localisation aims to identify which coordinates
have shifted between train and test distributions \citep{kulinskiFeatureShiftDetection2021}.
Our 
TASTE residuals provide a task-aware route to such
localisation.

Stein operators 
have been used to
construct discrepancy measures and goodness-of-fit tests such as kernelised
Stein discrepancies,
and computable
measures of sample quality,
see the surveys 
\cite{anastasiouSteinsMethodMeets2022} 
\gr{and \cite{liu2026probabilistic}}.
In contrast to the typical setting 
which obtains a measure of discrepancy by optimising the TASTE functional 
\eqref{eq:stein-functional-def-main} over a set of test functions, TASTE 
fixes the test function, yielding a
model-centric diagnostic tool.

Score matching and diffusion-based generative modelling are  scalable ways to
estimate $\nabla\log p(x)$ for 
high-dimensional data, and can be
interpreted geometrically as learning vector fields aligned with local density
gradients \citep{songScoreBasedGenerativeModeling2021}.  Thus,  score models
are a natural component in operator-based diagnostics that require access to the
data geometry, including our Stein-based construction.

\section{Stein Operators and Data Geometry}
\label{sec:stein-background}

Here we provide the minimal Stein machinery needed for our 
goal 
to connect a \emph{fixed} model
$f$ to the geometry of a reference data distribution $p$ via a differential
operator whose expectation vanishes under $p$; see \cite{anastasiouSteinsMethodMeets2022} for details.
\gr{Our setting is as follows.}
Let $f:\R^d\to\R$ be a scalar test function (or predictor); for a
vector-valued predictor $f:\R^d\to\R^m$ we apply all definitions componentwise.
We assume that $p:\R^d\to(0,\infty)$ is a continuously differentiable probability density with
{score function} $s_p(x) := \nabla \log p(x) \in \R^d$. The score function requires the density $p$ to be available only up to a normalising constant. Expectations under $p$ 
are denoted $\E_p[\cdot]$. 

A so-called {\it Stein operator} 
is an operator $\mathcal{T}_p$ such that the 
\emph{Stein identity} 
\begin{equation}
  \label{eq:stein-identity-main}
  \E_p[\mathcal{L}_p f(X)] = 0
\end{equation}
holds for any sufficiently regular $f$. Eq. \eqref{eq:stein-identity-main} 
\gr{makes} Stein operators 
\gr{very} useful: they produce quantities that equal zero for
in-distribution data (in expectation); \gr{nonzero values indicate} 
distribution shift.
Here we
focus on the {Langevin Stein operator}
\gr{from \eqref{eq:langevin-stein-operator-main}.}
A key property of $\mathcal{L}_p$ is 
\gr{its} divergence form:
\begin{equation}
  \label{eq:langevin-stein-divergence-main}
  \mathcal{L}_p f(x)
  =
  \frac{1}{p(x)}\,\nabla\cdot\big(p(x)\,\nabla f(x)\big).
\end{equation}


Under mild regularity and boundary decay 
conditions, 
\eqref{eq:langevin-stein-divergence-main} implies the \gr{Stein identity \eqref{eq:stein-identity-main},} 
see App.~\ref{app:stein-basics}.  
The divergence view \eqref{eq:langevin-stein-divergence-main} also provides a geometric interpretation.  The vector field
$p(x)\nabla f(x)$ 
\gr{is} a density-weighted \emph{sensitivity flow}
induced by the model.  The Stein identity \eqref{eq:stein-identity-main} states that, under the reference
density $p$, this field cannot generate net flux in expectation.  
\gr{For us,} $\nabla f$ captures local
model sensitivity, while $p$ encodes the data geometry; $\mathcal{L}_p f$
combines both into a single scalar diagnostic.

\section{TASTE: Task-Aware Stein \gr{Residuals}
}
\label{sec:stein-fixed-predictors}

We now 
\emph{fix} $f$ to be a pre-trained model
and ask how Stein operators behave when the test distribution differs from the
training distribution. 
\gr{Let} $q$ be another smooth density on $\R^d$, absolutely continuous with respect
to $p$, and let $S_f(p,q) =\mathbb{E}_q[\mathcal{L}_p f(X)]$ be the   
TASTE functional of the fixed predictor $f$ at
$(p,q)$ given in \eqref{eq:stein-functional-def-main}.
When $q=p$, \eqref{eq:stein-identity-main} implies $S_f(p,p)=0$.  When $q\neq p$,
the functional measures how the test distribution violates the Stein identity
that holds under $p$. 
\gr{The} central identity 
\begin{equation}
  \label{eq:projection-identity-main}
  S_f(p,q)
  =
  -\,\E_q\!\Big[
    \nabla f(X)^\top \nabla\log\frac{q(X)}{p(X)}
  \Big]
\end{equation}
expresses $S_f(p,q)$ 
in terms of the density ratio $q/p$\gr{; for assumptions and the proof see App.\,\ref{app:projection}.}
Defining the \emph{shift score field}
  $u_{p\to q}(x) := \nabla\log\frac{q(x)}{p(x)}$, 
\eqref{eq:projection-identity-main} becomes
 $ S_f(p,q) = -\,\ip{\nabla f}{u_{p\to q}}_{L^2(q)}.$
This is the key ``task-aware'' mechanism: for a fixed predictor $f$, the Stein
functional measures only the component of the distribution shift that aligns with the  sensitivity field $\nabla f$ of the predictor.  Shifts
that are large in density terms but lie mostly in directions orthogonal to
$\nabla f$ can be invisible to $S_f$; conversely, comparatively small shifts
can trigger a strong response if they move mass along directions where $f$ is
sensitive.  



\begin{example} \label{sec:small-change-main}
\gr{As an example}
we \gr{perturb} 
$p$ via exponential tilts $
  q_\varepsilon(x) = \frac{1}{Z_\varepsilon}\,p(x)\,e^{\varepsilon h(x)},$
where $h:\R^d\to\R$ is a shift potential and $\varepsilon$ controls the shift
strength.  Under regularity assumptions on $h$, we have $\nabla\log\frac{q_\varepsilon(x)}{p(x)}
= \varepsilon\,\nabla h(x)$, and substituting into
\eqref{eq:projection-identity-main} yields
\begin{align*}
S_f(p,q_\varepsilon)
&= -\,\varepsilon\,\E_{q_\varepsilon}\!\left[\nabla f(X)^\top \nabla h(X)\right].
\end{align*}
Moreover, expanding expectations under $q_\varepsilon$ around $p$ gives the
first-order response
\begin{equation}
  \label{eq:tilt-expansion-main}
  S_f(p,q_\varepsilon)
  =
  \varepsilon\,\Cov_p\!\big(\mathcal{L}_p f(X),\,h(X)\big)
  + O(\varepsilon^2),
\end{equation}
\gr{see
 App.\,~\ref{app:tilt} for details}.  
Thus, for small smooth
shifts, the 
TASTE functional responds linearly in $\varepsilon$, with  slope
given by a covariance under the training distribution.  In practice, this
regime may provide a useful diagnostic lens even when real shifts are not exactly
tilts.
\end{example}

\subsection{\gr{Adjusted TASTE:} Correcting for using an approximate score model}
\label{sec:score-error-main}

Computing $\mathcal{L}_p f$ requires 
$s_p=\nabla\log p$, yet in
applications $p$ is unknown.  Substituting a learned model $\tilde p$ with
score $\tilde s=\nabla\log\tilde p$ yields 
the approximate operator
$
  \mathcal{L}_{\tilde p} f(x) = \Delta f(x) + \tilde s(x)^\top \nabla f(x).
$
A useful consequence of the projection identity \eqref{eq:projection-identity-main} is that the error induced by
$\tilde s$ is \emph{directional}: it depends on how score-model error aligns
with $\nabla f$ and the shift direction $\gr{l=}q/p$. \gr{Taking}
$g(x) := (\tilde s(x)-s_p(x))^\top \nabla f(x)$,
under mild regularity assumptions, one obtains the decomposition 
\begin{equation}
\label{eq:directional-main}
  \E_q[\mathcal{L}_{\tilde p} f]
  =
  \E_p[\mathcal{L}_{\tilde p} f]
  +
  S_f(p,q)
  +
  \langle g,\ l-1\rangle_{L^2(p)},
\end{equation}
see
App.~\ref{app:directional}. This motivates the \emph{adjusted TASTE functional}
\begin{equation}
  \label{eq:corrected-main}
  \tilde S_f(p,q)
  :=
  \E_q[\mathcal{L}_{\tilde p} f]
  -
  \E_p[\mathcal{L}_{\tilde p} f],
\end{equation}
where $\E_p[\mathcal{L}_{\tilde p} f]$ can be estimated on held-out training
data.

The adjusted 
{TASTE} functional $\tilde S_f(p,q)$ 
gives rise to several observations.
First, it guarantees \emph{zero population response under no shift}: when $q=p$, the density ratio satisfies $l\equiv 1$ and both
$S_f(p,p)=0$ and $\langle g,l-1\rangle_{L^2(p)}=0$, so the corrected functional vanishes even if the score model
$\tilde s_p$ is imperfect.
Second, the ideal term $S_f(p,q)$ retains the task-aware
projection interpretation in
\eqref{eq:projection-identity-main}.
Score-model error contributes the additional bias
$\langle g,l-1\rangle_{L^2(p)}$, determined by the
alignment of the density-ratio perturbation $l-1$
with the directional score error
$g=(\tilde s-s_p)^\top\nabla f$.
Orthogonality to $g$ eliminates this bias.
Third, the influence of score-model mismatch is quantitatively controlled.  
We show in App.\, \ref{app:directional} that the bias term 
satisfies 
\begin{align} \label{eq:scoreerror}
  \bigl|\langle g,\ l-1\rangle_{L^2(p)}\bigr|
  \;\le\;
  \sqrt{J(p\|\tilde p) } \;
  \|\nabla f\|_{L^4(p)}\;
  \|l-1\|_{L^4(p)},
\end{align}
where $J(p\|\tilde p)$ is the Fisher divergence minimized by denoising score matching \citep{songScoreBasedGenerativeModeling2021}.

\paragraph{Diffusion score approximation.}
We evaluate the score model at
$x_t=\alpha_t x+\sigma_t\varepsilon$,
$\varepsilon\sim\mathcal{N}(0,I)$, with timestep $t$,
while evaluating predictor derivatives at the clean input $x$.
For $t>0$, this construction does not generally preserve
the Stein identity for the clean-data distribution.
We interpret $\widehat{s}_t(x;\varepsilon)=s_\theta(x_t,t)$
as a stochastic surrogate for $\nabla\log p(x)$.
Its discrepancy, which can introduce both bias and variance,
is treated as a score-estimation error, subject to the
assumptions of our approximation analysis.

\textbf{On access to the test distribution.}
We do not assume access to an explicit density for the
test distribution $q$.  The distribution $q$ enters only implicitly through
test samples, and all expectations with respect to $q$ are estimated
empirically by sample averages over the observed test set. Analysis regarding the choice of $t$ is deferred to Appendix~\ref{sec:diffusion-time step-sensitivity}.


\paragraph{Per-sample residuals.}
The 
\gr{adjusted TASTE functional} \eqref{eq:corrected-main} can be used at different
granularity levels. Given a test input $x \sim q$, \gr{we} define the adjusted \emph{TASTE residual}
\[
  r_f(x) := \mathcal{L}_{\tilde p} f(x) - D_f,
\]
where $D_f = \E_p[\mathcal{L}_{\tilde p} f(X)]$, which in practice is estimated from training data. \gr{Here are three uses of the adjusted TASTE residuals:}
\begin{itemize}[noitemsep, leftmargin=*]
    \item \gr{Ranking:}
$r_f(x)$ is used  for OOD
ranking.
\item \gr{OOD testing:}
Using held-out training data, 
\gdr{let} $\tau_\alpha$ 
\gdr{be} the
empirical $(1-\alpha)$-quantile of $r_f(x)$ under $X\sim p$.
At test time, an input $x$ is declared out-of-distribution if
$\lvert r_f(x)\rvert>\tau_\alpha$.
This 
model-agnostic decision rule 
approximately controls the false
positive rate under the training distribution, while remaining  
sensitive to
shifts that affect the predictor’s behaviour.
\item \gr{Aggregation:} 
The \gr{adjusted TASTE residuals} 
\gr{can be} 
averaged over the test set or used in a batch statistic to produce a single test-set operator value \( \frac{1}{|\mathcal{X}_{\mathrm{test}}|}\sum_{x\in\mathcal{X}_{\mathrm{test}}} r_f(x) \) for global shift detection.
\end{itemize}

\paragraph{Per-dimension residuals.}
\gr{For finer granularity,}  writing derivatives in $\mathcal{L}_p$ component-wise as 
$
  \mathcal{L}_p f(x)
  =
  \sum_{i=1}^d
  \Big(\partial_{ii} f(x) + \partial_i\log p(x)\,\partial_i f(x)\Big) 
$ 
motivates the  per-dimension \emph{TASTE residuals}
\begin{equation}\tag{$\star$}\label{eq:componentwise-stein}
      r_{f,i}(x)
  :=
  \partial_{ii} f(x) + \partial_i\log p(x)\,\partial_i f(x),
\end{equation}
so that $r_f(x)=\sum_i r_{f,i}(x)$. Importantly, each coordinate-wise 
$r_{f,i}(x)$ is itself a valid Langevin Stein operator applied to the one-dimensional test function
\(x \mapsto \partial_i f(x)\).
Under the same regularity conditions as for the full operator,
each component satisfies
\(
\mathbb{E}_{p}[r_{f,i}(X)] = 0
\),
and admits an analogous projection formula under a shifted distribution.
As a result, all theoretical properties established for the aggregated TASTE functional— task-aware sensitivity to shift and robustness via baseline correction—extend directly to the per-dimension setting.

In imaging applications, the index $i$ corresponds to a pixel (or pixel-channel), yielding an interpretable
anomaly map $i\mapsto r_{f,i}(x)$.  In this sense, TASTE 
residuals provide a
\emph{model-agnostic and interpretable} mechanism: they apply to any pre-trained
predictor for which gradients are available, and their per-input\footnote{The "per-input" refers to a single coordinate of the gradient/Laplacian;  
each $r_{f,i}(x)$ requires the entire $x\in \R^n$.} decomposition
highlights which input coordinates contribute most strongly to the task-aware
shift signal \eqref{eq:componentwise-stein}. This setting supports
use cases such as pixel-level anomaly segmentation and localisation of
distributional shift within an image, which have been extensively studied in
industrial inspection and medical imaging 
\citep{bergmannUninformedStudentsStudentTeacher2020,
rothTotalRecallIndustrial2022,zenatiAdversariallyLearnedAnomaly2018}.

\gr{For a fixed $f_\theta$,} the per-dimension decomposition computes 
\(r_i(x)=\partial_{ii} f_\theta(x) + \hat{s}_{p,i}(x)\,\partial_i f_\theta(x)\) for each input dimension \(i\), producing a spatial/feature map \(\{r_i(x)\}_{i=1}^d\). If this map is used, the correction \(D_f\) must also be estimated as a per-dimension vector \(D_f = \big(\mathbb{E}_p[r_1(X)],\dots,\mathbb{E}_p[r_d(X)]\big)\) (in practice estimated by aggregating per-dimension raw residuals on the calibration set). This per-dimension mode enables fine-grained anomaly heatmaps but increases compute as  Hutchinson approximation, which we employ for estimating the Laplacian term, cannot be used directly.

\subsection{The TASTE pipeline}

Overall, the proposed  \gr{TASTE} pipeline is: 
(i) take any fixed predictor $f_\theta$,
(ii) take any score/density model $\tilde p$ for $p$ (e.g.\ a diffusion/score model),
(iii) compute $\mathcal{L}_{\tilde p} f_\theta(x)$ (exactly or approximately) over the test set,
(iv) subtract a train-estimated baseline $D_f$, and (v) use the resulting
per-sample or per-dimension residuals for task-aware OOD detection and
diagnostics. 
\gr{Algorithm~\ref{alg:merged-stein-batched-clean} in App.\,\ref{app:alg} gives the adjusted TASTE algorithm, with different options, in pseudocode.}

\paragraph{Probe selection for multiclass predictors.} 
Unless otherwise stated, our multiclass experiments use
$f_c(x)=\operatorname{softmax}(z(x))_c$ as the scalar probe, with
$c=\arg\max_j z_j(x)$ held fixed during differentiation. Probe selection, the use of the Laplacian, and the aggregation
of residuals are implementation choices that determine the
resulting detector. Appendix~\ref{app:taste-protocols} specifies
the scoring protocol for each experiment.
For ImageNet-1K, we additionally evaluate a first-order
classwise variant that aggregates extreme signed responses
across classes; its construction is given in
Appendix~\ref{app:taste-classwise}.

\textbf{Computational efficiency.}
Computing the Laplacian term $\Delta f_\theta(x)=\mathrm{tr}(\nabla^2 f_\theta(x))$
exactly is computationally expensive in high dimensions.  In practice we
therefore rely on stochastic trace estimators, most notably Hutchinson’s
unbiased estimator and its modern variance-reduced variants
\citep{hutchinsonStochasticEstimatorTrace1989,meyerHutchOptimalStochastic2021}.

\section{Experimental results}
\label{sec:experiments}

We present concise empirical evidence that (i) \mk{TASTE}
tracks task-relevant shifts, (ii) remains stable under harmless perturbations,
and (iii) yields effective per-sample OOD signals in mixed test sets.  Full
implementation details and additional plots are deferred to App.~\ref{app:experimental-setup}. Ablations \gr{on operator choice  
and 
sensitivity to 
quality of the learned score model 
and to the diffusion time step are in App.~\ref{app:ablations}}. 
 Research code is provided in the submitted additional zip file. 

\subsection{Controlled 2D shift: directional sensitivity}
\label{sec:2d-rotation-experiment}

In this controlled two-dimensional setting, the direction of
the distribution shift is  varied independently of its magnitude.  The training
distribution is 2D standard Gaussian $p(x)=\mathcal{N}(0,I_2)$ and the test distribution is obtained
by 
translating $p$ by a fixed magnitude $\varepsilon=10$ along a rotated version of the base direction $u=(1,1)^\top/\!\sqrt{2}$, producing
$q_\varphi=\mathcal{N}(\varepsilon R_\varphi u,I_2)$, where $R_\varphi$ denotes the rotation matrix by the angle $\varphi$. 
The prediction task is $y=x_2-x_1$, which is invariant to translations along
$(1,1)$ and maximally sensitive along $(1,-1)$.  A small ReLU network is trained
under $p$, and we evaluate prediction error, the TASTE functional
$S_f(p,q_\varphi)$, and the log-likelihood under $p$ as $\varphi$ varies.

\begin{figure}[t]
  \centering

  \begin{subfigure}[t]{0.49\linewidth}
    \centering
    \includegraphics[width=0.75\linewidth]{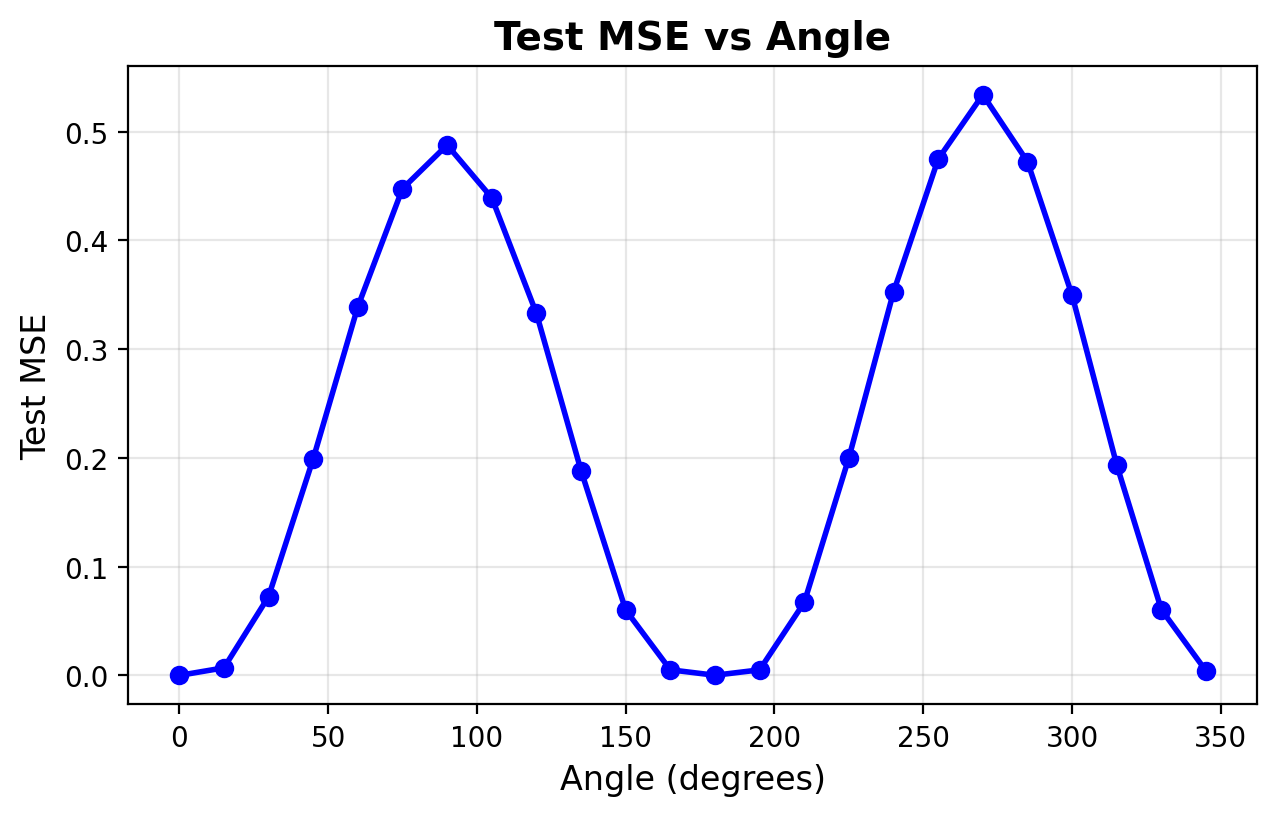}
    \label{fig:metrics_vs_angle}
  \end{subfigure}
  \hfill
  \begin{subfigure}[t]{0.49\linewidth}
    \centering
    \includegraphics[width=0.75\linewidth]{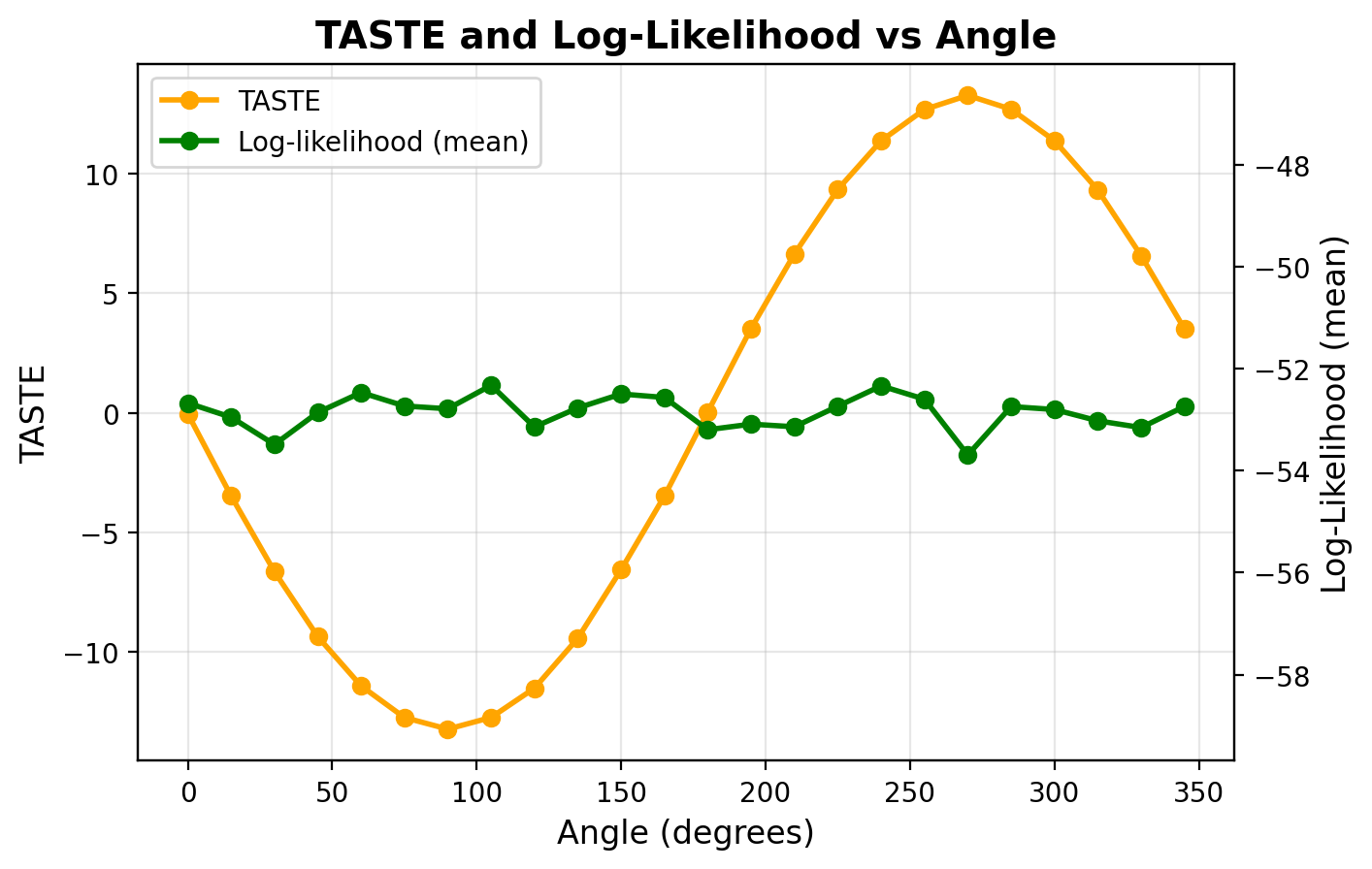}
    \label{fig:stein_vs_loglik}
  \end{subfigure}
  \caption{
    \textbf{Directional shift experiment.}
    Prediction error and the task-aware {TASTE functional} vary strongly with the
    direction of shift, peaking (in magnitude) when the shift aligns with the sensitive
    direction $(1,-1)$.  In contrast, the density-based score remains nearly
    constant for all directions.
  }
  \label{fig:rotation-experiment-lines}
\end{figure}

 Figure~\ref{fig:rotation-experiment-lines} shows how the Stein signal tracks
task degradation:
its extrema align with those of the prediction
error, while remaining invariant to shifts that leave the task unchanged.
This confirms that the TASTE functional responds selectively to task-relevant
components of distribution shift.

\subsection{MNIST: invariance to translations, sensitivity to rotations}
\label{sec:mnist-geometric}

We next evaluate the method on MNIST \citep{lecun2002gradient}, \gr{a medium scale data set,} under geometric perturbations.  Convolutional
networks are known to exhibit approximate translation invariance while remaining
sensitive to rotations; we therefore compare translations (largely benign) and
rotations (harmful) as test-time shifts  
\citep{hendrycksBenchmarkingNeuralNetwork2019,
bronsteinGeometricDeepLearning2021}.

An AlexNet-style CNN \citep{krizhevskyImageNetClassificationDeep2012} is trained on MNIST.  
So that translations do not
remove semantic content, all images are zero-padded from $28\times28$ to
$64\times64$ prior to training and evaluation. A denoising
score-matching model provides an estimate $\hat{s}_p(x)\approx\nabla\log p(x)$.
We compute the adjusted Langevin Stein residual $r_f(x)$ using
Algorithm~\ref{alg:merged-stein-batched-clean}. 
Test sets are constructed using translations of increasing $\ell^\infty$
magnitude and rotations of increasing angle.

\begin{figure}[t]
  \centering

  \begin{subfigure}[t]{0.49\linewidth}
    \centering
    \includegraphics[width=0.875\linewidth]{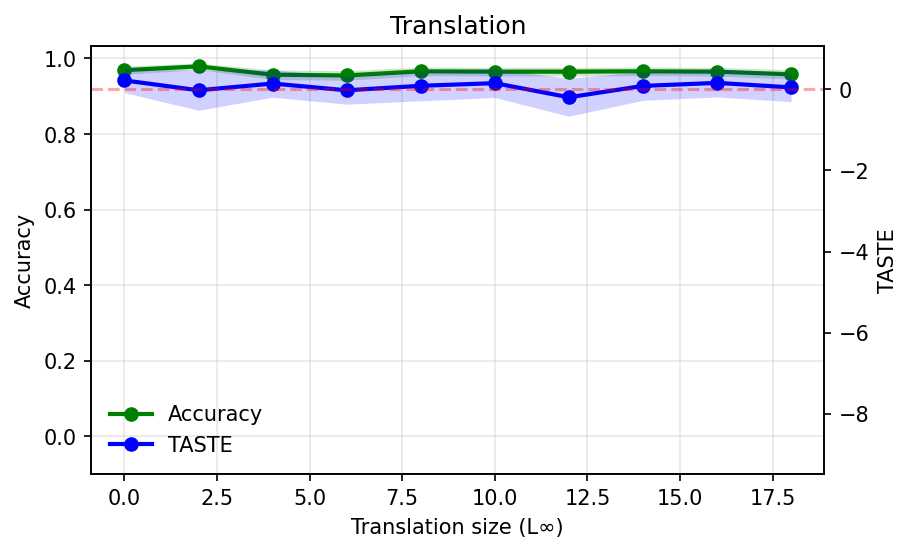}
  \end{subfigure}
  \hfill
  \begin{subfigure}[t]{0.49\linewidth}
    \centering
    \includegraphics[width=0.875\linewidth]{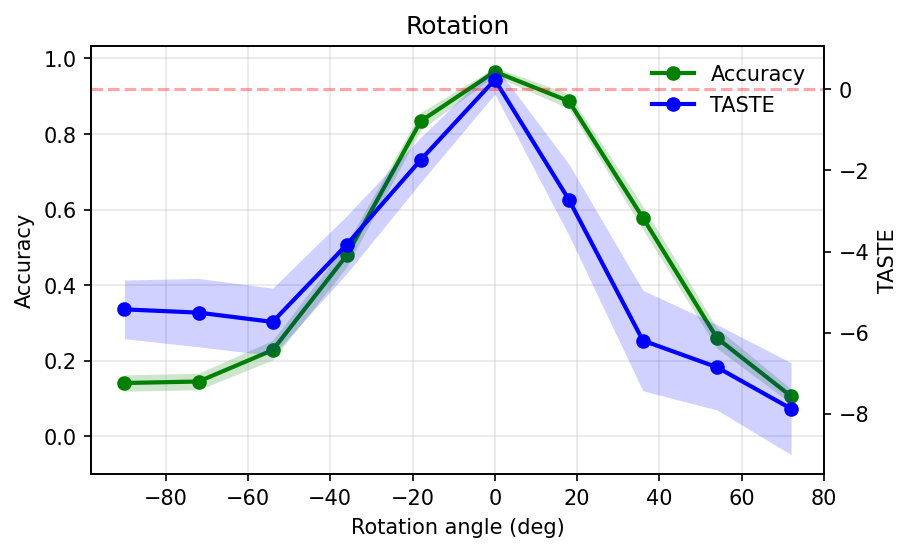}
  \end{subfigure}
  \caption{
    \textbf{MNIST under geometric perturbations.}
    Classification accuracy and adjusted TASTE score.
    Translations (left) leave both accuracy and \gr{adjusted TASTE residual} 
    unchanged, while rotations (right)
    induce a monotonic increase in the Stein signal aligned with performance
    degradation. \gr{Shaded regions show $\pm\, 2$ standard deviations.}
  }
  \label{fig:translation-rotation-cnn}
\end{figure}

Figure~\ref{fig:translation-rotation-cnn} shows that the Stein residual remains
near zero under translations but increases sharply under rotations, closely
mirroring the drop in accuracy.  Likelihood-based scores do not distinguish
between these cases, highlighting the task-awareness of the proposed signal.

Consistent with the per-dimension decomposition introduced in the methodology,
the Stein residual can also be evaluated coordinate-wise on image inputs.  On
MNIST, this yields per-pixel residual maps that localise regions most
incompatible with the training distribution as processed by the classifier, see 
App.~\ref{ap:per-pixel-heatmaps};
evaluations using a mixed MNIST / Fashion-MNIST OOD setup are found in  App.~\ref{sec:mixed-ood}.

\subsection{Benchmark evaluation \gr{on standard OOD data sets}}\label{sec:benchmark-evaluation}
\label{sec:benchmark-cifar}

We evaluate the proposed  TASTE residual against widely used OOD detection
baselines on CIFAR-10 \citep{krizhevskyLearningMultipleLayers} and ImageNet-1K \citep{russakovsky2015imagenet} covering heterogeneous shift regimes
including adversarial perturbations, synthetic corruption and perturbation
benchmarks, and \gr{other} standard OOD datasets
\citep{krizhevskyLearningMultipleLayers,yangGeneralizedOutofDistributionDetection2024}.
For CIFAR-10, the classifier is a ResNet-18
\citep{heDeepResidualLearning2015}; for ImageNet-1K, we use ResNet-50. In both cases, the score model is a diffusion-based estimator of
$\nabla\log p(x)$ \citep{hoDenoisingDiffusionProbabilistic2020}.
Implementation and configuration details are deferred to
App.~\ref{app:experimental-setup}.

\textbf{Benchmarks.}
For CIFAR-10, we group evaluation datasets into four categories:
(i) \emph{adversarial perturbations:} FGSM, PGD, and AutoAttack
\citep{goodfellowExplainingHarnessingAdversarial2015,madryDeepLearningModels2019,croceReliableEvaluationAdversarial2020};
(ii) \emph{corruption benchmarks:} CIFAR-10-C, averaged over all corruption
types;
(iii) \emph{perturbation benchmarks:} CIFAR-10-P, averaged over all perturbation
types \citep{hendrycksBenchmarkingNeuralNetwork2019}; and
(iv) \emph{classical OOD datasets:} SVHN, LSUN, iSUN, Textures, and Places365
\citep{netzerReadingDigitsNatural,yuLSUNConstructionLargescale2016,xuTurkerGazeCrowdsourcingSaliency2015,cimpoiDescribingTexturesWild2013,zhouPlacesImageDatabase2016}.
For ImageNet-1K, we evaluate adversarial perturbations and the same classical OOD
benchmark suite, upscaled to ImageNet resolution. All results are reported using
AUROC and false positive rate at 95\% true positive rate (FPR95), averaged
within each category, with "overall" reporting the average across all used datasets.

\textbf{Baselines.}
We compare against a diverse set of representative OOD detection methods spanning
confidence-based, distance-based, energy-based, post-hoc, and score-curvature
approaches. Implementation details and exact configurations are provided in
App.~\ref{app:experimental-setup}.

\textbf{Results and benchmark analysis.}
Tables~\ref{tab:cifar-benchmark} and~\ref{tab:imagenet-benchmark}
report OOD detection performance on CIFAR-10 and ImageNet-1K,
respectively. Consistent with prior work, no single method
dominates across all regimes
\citep{tajwarNoTrueStateoftheArt2021}. On CIFAR-10, TASTE achieves the highest overall AUROC and
lowest overall FPR95 under the reported aggregation.
Its strengths include the highest adversarial AUROC and
the lowest FPR95 on standard OOD benchmarks.
On CIFAR-10-C and CIFAR-10-P, performance remains competitive,
although other methods achieve higher AUROC.

ImageNet-1K presents a substantially larger classifier output
space, with 1,000 classes compared with 10 for CIFAR-10.
Hence approximating the Laplacian to a reasonable level becomes infeasible; TASTE has to be run based on the score-part only, using a  single predicted-class probability. However, this procedure may miss changes
in other class probabilities.
This motivates TASTE$_{\mathrm{ext}}$, which computes
classwise directional responses using a fast implementation and 
sums the largest and smallest 5\% by signed value
(Appendix~\ref{app:taste-classwise}). 
Compared with the predicted-class configuration,
TASTE$_{\mathrm{ext}}$ improves both AUROC and FPR95
on adversarial and standard OOD benchmarks.
It achieves the highest overall AUROC and lowest overall
FPR95 among the compared methods under the reported
aggregation. Performance nevertheless varies by regime:
ViM achieves the best standard OOD results in both metrics,
while SCOPED achieves the highest adversarial AUROC.
TASTE$_{\mathrm{ext}}$ achieves the lowest adversarial FPR95
and second-highest adversarial AUROC.
These findings support examining multiple classwise
responses on ImageNet-1K, although they do not isolate
output dimensionality as the cause of the improvement.
\paragraph{Runtime considerations.}
Table~\ref{tab:imagenet-benchmark} also reports wall-clock runtime per sample. While TASTE is slower than simple confidence-based methods such as MSP or Energy, it remains competitive with modern geometry- and diffusion-based approaches. Importantly, unlike many recent diffusion-based OOD detectors, TASTE requires only a single forward pass through the diffusion model to evaluate the score field, without iterative purification, sampling, or reconstruction procedures. This makes the method comparatively lightweight despite leveraging generative information about the data distribution.


\begin{table*}[t]
  \centering
  \caption{OOD detection performance on CIFAR-10 across heterogeneous shift regimes.}
  \label{tab:cifar-benchmark}
  \resizebox{\textwidth}{!}{%
    \begin{tabular}{l
                    cc  
                    cc  
                    cc  
                    cc
                    cc
                    }
      \toprule
      & \multicolumn{2}{c}{\textbf{Adversarial}} 
      & \multicolumn{2}{c}{\textbf{CIFAR-10-C}} 
      & \multicolumn{2}{c}{\textbf{CIFAR-10-P}} 
      & \multicolumn{2}{c}{\textbf{OOD-benchmarks}} 
      & \multicolumn{2}{c}{\textbf{Overall}} \\
      \cmidrule(lr){2-3} \cmidrule(lr){4-5} \cmidrule(lr){6-7} \cmidrule(lr){8-9}\cmidrule(lr){10-11}
      Method & AUROC$\uparrow$ & FPR95$\downarrow$ & AUROC$\uparrow$ & FPR95$\downarrow$ & AUROC$\uparrow$ & FPR95$\downarrow$ & AUROC$\uparrow$ & FPR95$\downarrow$ & AUROC$\uparrow$ & FPR95$\downarrow$\\
      \midrule
      MSP & 0.3575 & 0.9656 & 0.6124 & 0.9109 & 0.5978 & 0.9210 & 0.8032 & 0.7878 & 0.5751 {\scriptsize$\pm$ 0.0080}& 0.8521{\scriptsize$\pm$ 0.0161}\\
      ODIN & 0.4895 & 0.9436 & 0.5499 & 0.9013 & 0.5478 & 0.9222 & 0.7644 & 0.7756 & 0.5776{\scriptsize$\pm$ 0.0084} & 0.9486{\scriptsize$\pm$ 0.0070} \\
      Mahalanobis & 0.5819 & \underline{0.8388} & 0.5880 & 0.9181 & 0.5620 & 0.9435 & 0.8188 & 0.7195 & 0.6060{\scriptsize$\pm$ 0.0084} & 0.8664{\scriptsize$\pm$ 0.0116} \\
      Energy & 0.3550 & 0.9681 & \textbf{0.6256} & \underline{0.8900} & \underline{0.6110} & 0.9113 & \textbf{0.8520} & \underline{0.6683} & 0.5871{\scriptsize$\pm$ 0.0081} & \underline{0.8421}{\scriptsize$\pm$ 0.0125} \\
      kNN+ & 0.4645 & 0.9591 & 0.6151 & 0.8977 & 0.5956 & 0.9223 & 0.8288 & 0.7023 & 0.6007{\scriptsize$\pm$ 0.0083} & 0.8526{\scriptsize$\pm$ 0.0123} \\
      ReAct & \underline{0.5978} & 0.8637 & 0.5403 & 0.9810 & 0.5184 & 0.9807 & 0.7078 & 0.9996 & 0.5678{\scriptsize$\pm$ 0.0087} & 0.8993{\scriptsize$\pm$ 0.0091} \\
      ViM & 0.5495 & \textbf{0.8372} & 0.6045 & 0.9169 & 0.5716 & 0.9451 & 0.8174 & 0.7918 & \underline{0.6079}{\scriptsize$\pm$ 0.0084} & 0.8552{\scriptsize$\pm$ 0.0111} \\
      ASH & 0.3614 & 0.9691 & \underline{0.6245} & 0.8911 & \textbf{0.6121} & \underline{0.9108} & \underline{0.8498} & 0.6701 & 0.5882{\scriptsize$\pm$ 0.0081} & 0.8454{\scriptsize$\pm$ 0.0126} \\
      PRO & 0.3478 & 0.9793 & 0.6043 & 0.9116 & 0.5838 & 0.9353 & 0.7159 & 0.9203 & 0.5560{\scriptsize$\pm$ 0.0082} & 0.9334{\scriptsize$\pm$ 0.0147} \\
      TASTE (ours) & \textbf{0.6144} & 0.8650 & 0.6193 & \textbf{0.8691} & 0.5954 & \textbf{0.8611} & 0.7647 & \textbf{0.5677} & \textbf{0.6301}{\scriptsize$\pm$ 0.0082} & \textbf{0.8207}{\scriptsize$\pm$ 0.0127} \\
      \bottomrule
    \end{tabular}}
\caption*{OOD detection on CIFAR-10 across heterogeneous
shift regimes. Results are group averages; overall entries
show $\pm2$ s.d. Bold and underlined values indicate the
best and second-best results. Per-dataset results appear
in Appendix~\ref{app:experimental_results}.}
\label{tab:cifar-benchmark}
\end{table*}

\begin{table}[t]
  \centering
  \caption{OOD detection performance on ImageNet-1K.}
  \label{tab:imagenet-benchmark}
  \resizebox{\linewidth}{!}{%
    \begin{tabular}{l
                    cc
                    cc
                    cc
                    r
                    }
      \toprule
      & \multicolumn{2}{c}{\textbf{Adversarial}} 
      & \multicolumn{2}{c}{\textbf{OOD-benchmarks}}
      & \multicolumn{2}{c}{\textbf{Overall}}
      & \multirow{2}{*}{\textbf{Runtime}} \\
      \cmidrule(lr){2-3} \cmidrule(lr){4-5} \cmidrule(lr){6-7}
      Method & AUROC$\uparrow$ & FPR95$\downarrow$ 
             & AUROC$\uparrow$ & FPR95$\downarrow$
             & AUROC$\uparrow$ & FPR95$\downarrow$
             & ms/sample \\
      \midrule

      MSP & 0.2543 & 0.9551 & 0.8661 & 0.4546 & 0.4582 {\scriptsize$\pm$ 0.0019} & 0.7883 {\scriptsize$\pm$ 0.0039} & 5.95 \\
      ODIN & 0.6475 & 0.8212 & 0.8356 & 0.4893 & 0.7102 {\scriptsize$\pm$ 0.0026} & 0.7106 {\scriptsize$\pm$ 0.0050} & 19.78 \\
      Mahalanobis & 0.6785 & 0.7452 & 0.9264 & \underline{0.1749} & \underline{0.7611} {\scriptsize$\pm$ 0.0024} & \underline{0.5551} {\scriptsize$\pm$ 0.0048} & 137.69 \\
      Energy & 0.2242 & 0.9607 & 0.8971 & 0.3634 & 0.4485 {\scriptsize$\pm$ 0.0018} & 0.7616 {\scriptsize$\pm$ 0.0043} & 6.04 \\
      kNN+ & 0.6099 & 0.8089 & 0.9175 & 0.2501 & 0.7125 {\scriptsize$\pm$ 0.0025} & 0.6226 {\scriptsize$\pm$ 0.0044} & 95.73 \\
      PRO & 0.2706 & 0.9519 & 0.3581 & 0.9240 & 0.2997 {\scriptsize$\pm$ 0.0025} & 0.9426 {\scriptsize$\pm$ 0.0024} & 71.45 \\
      ReAct & 0.3088 & 0.9485 & \underline{0.9283} & 0.2570 & 0.5153 {\scriptsize$\pm$ 0.0020} & 0.7180 {\scriptsize$\pm$ 0.0036} & 7.92  \\
      ASH & 0.2242 & 0.9607 & 0.8979 & 0.3601 & 0.4488 {\scriptsize$\pm$ 0.0017} & 0.7605 {\scriptsize$\pm$ 0.0045} & 7.04 \\
      ViM & 0.3763 & 0.9198 & \textbf{0.9428} & \textbf{0.1632} & 0.5651 {\scriptsize$\pm$ 0.0019} & 0.6676 {\scriptsize$\pm$ 0.0038} & 105.30 \\
      SCOPED & \textbf{0.7872} & \underline{0.6367} & 0.4794 & 0.7309 & 0.6846 {\scriptsize$\pm$ 0.0026} & 0.6681 {\scriptsize$\pm$ 0.0080} & 157.19 \\
      TASTE (ours) & 0.7420 & 0.7122 & 0.6820 & 0.5866 & 0.7220 {\scriptsize$\pm$ 0.0024} & 0.6703 {\scriptsize$\pm$ 0.0057} & 122.32 \\
      TASTE$_{ext}$ (ours) & \underline{0.7741} & \textbf{0.5701} & 0.8485 & 0.4532 & \textbf{0.7989} {\scriptsize$\pm$ 0.0037} & \textbf{0.5311} {\scriptsize$\pm$ 0.0064} & 123.83 \\

      \bottomrule
    \end{tabular}}
\caption*{\footnotesize
Results are averaged within each category.
Adversarial attacks include FGSM, PGD, and AutoAttack.
OOD benchmarks include SVHN, LSUN, iSUN, Textures, and Places365, all upscaled to match ImageNet-1K resolution.
Runtime reports wall-clock time per sample.
For readability, we only report \gr{$\pm 2$ s.d., as approximate} confidence intervals, for the "Overall" column.
Best results are in bold; second-best are underlined.
}
\end{table}

\subsection{Per-pixel anomaly detection on MVTec AD}
\label{sec:mvtec-qualitative}

We evaluate per-pixel TASTE residuals on the MVTec Anomaly Detection dataset \citep{bergmann2019mvtec} to
assess spatial localisation on high-resolution real-world imagery.
MVTec AD 
provides pixel-level ground-truth annotations supporting both quantitative
metrics (pixel-level AUROC, average precision, AUPRO) and visual inspection.
For each test image, we compute per-pixel Stein heatmaps from the
coordinate-wise operator and compare them to ground-truth defect masks.
Across categories, the heatmaps align spatially with defect regions.
Representative examples are shown in App.~\ref{ap:per-pixel-heatmaps}.

\textbf{Quantitative localisation results.}
Averaged over all $15$ MVTec categories, TASTE achieves an image-level AUROC of
$0.6527$ and average precision of $0.817$, indicating consistent image-level
separability in a zero-shot regime.
At the pixel level, performance reaches an AUROC of $0.6977$ and an AUPRO of
$0.688$. 
Performance varies across category types.
Texture categories exhibit substantially stronger localisation
(pixel-level AUROC $0.8682$, AUPRO $0.8745$), while object categories are more
challenging (AUROC $0.6124$, AUPRO $0.5947$).

Overall, these results \gdr{demonstrate} 
that  TASTE and produce meaningful, spatially \gdr{coherent localisation}
signals 
in a 
\emph{zero-shot} setting: neither the classifier nor the score
model is trained or fine-tuned on MVTec AD data. 
While not designed to match state-of-the-art industrial detectors, the method
provides a principled zero-shot localisation mechanism whose behaviour is
transparent and consistent with the theoretical analysis.

\section{Discussion and Conclusion}

We introduced {TASTE}, a task-aware OOD detection framework that applies Stein
operators to fixed, deployed predictors. By coupling model input sensitivity
with a learned score of the training distribution, TASTE produces OOD signals
that align with task degradation and admit interpretable per-sample and
per-input decompositions. The method achieves competitive performance while providing diagnostic capabilities
that are unavailable to purely confidence-based approaches. 
Unlike methods that flag distributional deviation alone, TASTE responds most
strongly to shifts that interact with the predictor. This makes it well suited
to post-deployment monitoring, where benign covariate changes must be
distinguished from harmful task-level perturbations.

\textbf{Choice of Stein operator.}
Different Stein operators encode different notions of compatibility between a
predictor and a data distribution. The Langevin operator captures curvature and
drift alignment through second-order information, while first-order variants
emphasise gradient--score alignment and are computationally lighter. The choice
of operator therefore determines which interaction between the predictor and the
input distribution is being probed: curvature, gradient alignment, or
score-weighted smoothness, echoing broader geometric perspectives on
representation and invariance in deep learning
\citep{bronsteinGeometricDeepLearning2021}.

\textbf{Reversing the perspective.}
Rather than starting from a predefined Stein operator, one can instead specify
the desired interaction between the shift score $\nabla\log(q/p)$ and a model
quantity such as its gradient, curvature, or feature map, and then construct an
operator that targets this interaction. This ``design-by-intent'' perspective
suggests a route toward task-aware operators tailored to particular
vulnerabilities, architectures, or modalities
\citep{mijouleSteinsDensityMethod2023}.

\textbf{Limitations.}
TASTE's performance depends on score quality and probe
selection, yet our analysis provides no simple practical
test of whether a given score model is sufficiently accurate.
Centring does not remove shift-dependent bias, and our
ablations cover only limited score-model mismatches.
Population guarantees do not ensure reliable pointwise
detection.
Differential operators add overhead relative to confidence-based
baselines, potentially limiting latency-sensitive applications.
Although required for the Langevin Stein identity, the
Laplacian offers setting-dependent empirical benefits, while
stochastic estimation adds variance.
Derivative stability warrants further evaluation.
Finally, the gains from TASTE$_{\mathrm{ext}}$ do not establish
output dimensionality as their cause.



\subsection*{AI use statement}
Generative AI tools assisted with manuscript drafting and
editing, discussion of mathematical arguments,
visualisation, implementation of standard methods, and
experimental workflows. The authors retain full responsibility
for the methodology, mathematical claims, experimental results,
and final content of this work.

\subsection*{Ethics statement}
This work uses existing public benchmark datasets and
pretrained models. TASTE may support applications such as
material inspection and detection of harmful distribution
shifts. However, false positives may trigger unnecessary
interventions, while false negatives may leave harmful shifts
undetected. Deployment therefore requires application-specific
validation, calibration, and careful integration with downstream
decisions, particularly in safety-critical settings. Detection
performance should not be interpreted as a guarantee of model
safety or reliability. 

\subsection*{Reproducibility statement}
Theoretical assumptions and proofs are provided in
Appendix~\ref{app:stein-proofs}, and
Algorithm~\ref{alg:merged-stein-batched-clean} specifies
the residual computation.
Appendix~\ref{app:experimental-setup} describes the datasets,
models, and experimental procedures, including probe selection
and aggregation conventions in
Appendix~\ref{app:taste-protocols}.
Detailed benchmark results and ablations appear in
Appendices~\ref{app:experimental_results}
and~\ref{app:ablations}, respectively. We release research code to reproduce the results and explore the proposed constructions. Developing a maintained Python package remains future work.


\bibliography{references_v2}
\bibliographystyle{icml2026}

\newpage
\clearpage

\appendix

{\large{Appendix}}

\DoToC

\onecolumn


\section{Core methodology: assumptions, proofs and additional Stein identities}
\label{app:stein-proofs}
For convenience, this section of the Appendix 
is structured as follows.

\begin{itemize}[leftmargin=*]
\item Appendix~\ref{app:stein-basics}: Assumptions and proof of the Stein identity \eqref{eq:stein-identity-main} under $p$;
\item Appendix~\ref{app:projection}:
Assumptions and proof of the projection identity \eqref{eq:projection-identity-main}, formulated as Proposition \ref{prop:projection-main} and used to interpret $S_f(p,q)$;
\item Appendix~\ref{app:tilt}:
 Assumptions and proof of the controlled small-shift expansion \eqref{eq:tilt-expansion-main} for exponential tilts;
\item Appendix~\ref{app:directional}: Assumptions and proof of the score-model error decomposition \eqref{eq:directional-main};
\item Appendix~\ref{app:score-model-stability-bounds}: Assumptions and proof for the score-model error stability 
\eqref{eq:scoreerror}. 
\end{itemize}

\subsection{Stein identity for the Langevin operator}
\label{app:stein-basics}

Let $f: \R^d \rightarrow \R$ be twice differentiable.
Recall the Langevin Stein operator acting on $f$ by
\[
  \mathcal{L}_p f(x)=\Delta f(x)+(\nabla\log p(x))^\top \nabla f(x)
  \;=\;
  \frac{1}{p(x)}\,\nabla\cdot\big(p(x)\nabla f(x)\big).
\]

\begin{definition}[Stein class for $\mathcal{L}_p$]
\label{def:stein-class-app}
Let $p$ be a continuously differentiable density on $\R^d$. A function
$f:\R^d\to\R$ belongs to the Stein class of $p$ (for $\mathcal{L}_p$), denoted
$f\in\mathcal{F}(p)$, if:
\begin{enumerate}[label=\textnormal{(S\arabic*)}, leftmargin=*]
\item \label{S1}
$f$ is twice continuously differentiable and $\Delta f$, $\nabla f$ are locally
integrable with respect to Lebesgue measure.
\item \label{S2}
The vector field $p(x)\nabla f(x)$ is integrable and its flux over spheres
vanishes:
\[
\lim_{R\to\infty}\int_{\partial B_R} p(x)\nabla f(x)\cdot n(x)\,dS(x)=0,
\]
where $B_R \subset \R^d$ is the Euclidean ball of radius $R$ in $\R^d$, and $n(x)$ is the outward unit normal.
\item \label{S3}
$\mathcal{L}_p f$ is integrable under $p$.
\end{enumerate}
\end{definition}

\medskip
\begin{proposition}[Stein identity]
\label{prop:stein-identity-app}
If $p$ is continuously differentiable and $f\in\mathcal{F}(p)$, then
\[
  \E_p[\mathcal{L}_p f(X)] = 0.
\]
\end{proposition}

\begin{proof}
By \ref{S3}, $\E_p[\mathcal{L}_p f(X)]$ exists. Using the divergence form,
\[
  \E_p[\mathcal{L}_p f(X)]
  = \int_{\R^d} p(x)\mathcal{L}_p f(x)\,dx
  = \int_{\R^d}\nabla\cdot\big(p(x)\nabla f(x)\big)\,dx
  = \lim_{R \rightarrow \infty} \int_{B_R} \nabla\cdot\big(p(x)\nabla f(x)\big)\,dx,
\]
where $B_R$ is the ball of radius $R$, as above.
By the divergence theorem,
\begin{equation}
    \label{eq:div} \int_{B_R}\nabla\cdot\big(p\nabla f\big)\,dx
  = \int_{\partial B_R} p(x)\nabla f(x)\cdot n(x)\,dS(x).
\end{equation}
Letting $R\to\infty$ and using \ref{S2} yields
$\int_{\R^d}\nabla\cdot(p\nabla f)\,dx=0$, hence $\E_p[\mathcal{L}_p f]=0$.
\end{proof}

\subsection{Proof of the projection identity 
(\ref{eq:projection-identity-main})}
\label{app:projection}

We prove the \gr{projection} identity
\[
S_f(p,q)
=
-\,\E_q\!\Big[
\nabla f(X)^\top\nabla\log\frac{q(X)}{p(X)}
\Big]
\]
under sufficient regularity conditions.

\begin{assumption}[Regularity for the  projection identity]
\label{ass:projection-app}
Let $p,q$ be continuously differentiable densities on $\R^d$ with $q$ absolutely continuous w.r.t. $p$ and
let $l(x)=q(x)/p(x)$. Assume:
\begin{enumerate}[label=\textnormal{(P\arabic*)}, leftmargin=*]
\item \label{P1}
$f\in\mathcal{F}(p)$;
\item \label{P2} $\mathcal{L}_p f \in L^1(q)$;
\item \label{P3}
$l$ is 
differentiable and $p\nabla ( f\,l)$ is integrable.
\item \label{P4}
The boundary flux vanishes for the vector field $p\,l\,\nabla f$:
\[
\lim_{R\to\infty}\int_{\partial B_R} p(x)\,l(x)\,\nabla f(x)\cdot n(x)\,dS(x)=0.
\]
\item \label{P5}
The product $\nabla f^\top \nabla l$ is integrable under Lebesgue measure.
\end{enumerate}
\end{assumption}
\gr{Next we give a full statement for \eqref{eq:projection-identity-main}, including its assumptions; for reference we formulate it as a proposition.}

\begin{proposition}
    \label{prop:projection-main}[Projection identity]
Let $p, q$ and $f$ satisfy \ref{P1}-- \ref{P5}. Then 
\[
  S_f(p,q)
  =
  -\,\E_q\!\Bigg[
    \nabla f(X)^\top \nabla\log\frac{q(X)}{p(X)}
  \Bigg].
\] 
\end{proposition}

\begin{proof}
Let $p,q$ and $f$ be such that \ref{P1} -- \ref{P4} are satisfied. From \ref{P1}, $\mathcal{L}_p f$ exists, and from \ref{P2}, $S_f(p,q)$ exists. Substituting $q=lp$ in the divergence form \eqref{eq:langevin-stein-divergence-main} gives 
\[
S_f(p,q)
= \int_{\R^d} q(x)\,\mathcal{L}_p f(x)\,dx
= \int_{\R^d} l(x)\,\nabla\cdot\big(p(x)\nabla f(x)\big)\,dx.
\]
Integrating by parts, which is justified by \ref{P3} and \ref{P5}, using the product rule in the divergence form gives  
\[
\int_{\R^d} l\,\nabla\cdot(p\nabla f)\,dx
=
\int_{\R^d}\nabla\cdot(l\,p\nabla f)\,dx
-
\int_{\R^d} p\,\nabla f^\top \nabla l\,dx.
\]
By \ref{P4} and \gr{\eqref{eq:div}}, 
the first term on the right-hand side vanishes; hence
\[
S_f(p,q) = -\int_{\R^d} p(x)\,\nabla f(x)^\top \nabla l(x)\,dx.
\]
Now we  use $\nabla l = l\,\nabla\log l$ to obtain
\begin{align*}
S_f(p,q)
= &-\int_{\R^d} p(x)\,l(x)\,\nabla f(x)^\top \nabla\log l(x)\,dx \\
= & -\int_{\R^d} q(x)\,\nabla f(x)^\top \nabla\log\frac{q(x)}{p(x)}\,dx, 
\end{align*}
which is the desired identity.
\end{proof}

\subsection{Exponential-tilt expansions and remainder control}
\label{app:tilt}

We derive the small-$\varepsilon$ expansions used in
\eqref{eq:tilt-expansion-main}. First, we formally state the result.

\begin{proposition}
Let $p$ be a continuously differentiable probability density on $\mathbb{R}^d$, and let $q_\varepsilon(x)$
be its exponential tilt by a measurable function $h$
\[
q_\varepsilon(x)
= \frac{1}{Z_\varepsilon}\, p(x)\, e^{\varepsilon h(x)}.
\]
Assume the following:
\begin{enumerate}[label=\textnormal{(\roman*)}, leftmargin=*]
    \item There exists $\delta>0$ such that
    $\E_p\!\big[e^{\gamma |h(X)|}\big]<\infty$ for all $|\gamma|<\delta$.
    \item The Langevin Stein operator satisfies
    $\mathcal{L}_p f \in L^2(p)$.
\end{enumerate}
Then, for all positive $\varepsilon$ such that $\varepsilon<\delta/4$, the following expansions
hold:
\begin{align}
S_f(p,q_\varepsilon)
&= \varepsilon\, \mathrm{Cov}_p\!\bigl(\mathcal{L}_p f(X),\, h(X)\bigr)
  + O(\varepsilon^2),
\label{eq:stein-prop-expansion}
\\[4pt]
\Var_{q_\varepsilon}[\mathcal{L}_p f(X)]
&= \Var_p[\mathcal{L}_p f(X)]
 + \varepsilon\, \mathrm{Cov}_p\!\bigl((\mathcal{L}_p f(X))^2,\, h(X)\bigr)
 + O(\varepsilon^2).
\label{eq:stein-prop-variance}
\end{align}
\end{proposition}

\begin{proof}
Our goal is to expand $\E_{q_\varepsilon}[g(X)]$ and
$\Var_{q_\varepsilon}[g(X)]$ for small~$\varepsilon$. We proceed in several steps.

\textbf{Expansion of \texorpdfstring{$\E_{q_\varepsilon}[g(X)]$}{Eq[g(X)]}.} Let $g(X)=\mathcal{L}_p f(X)$, which is square-integrable under $p$ by assumption.
By the definition of the exponential tilt,
\begin{equation}
    \label{proof:first}
\E_{q_\varepsilon}[g(X)]
= \frac{\E_p[g(X)e^{\varepsilon h(X)}]}{\E_p[e^{\varepsilon h(X)}]}.
\end{equation}

Under the assumptions stated in the proposition, we shall prove below that 
\begin{equation}\label{eq:expansion_to_prove} 
\E_p[g(X)e^{\varepsilon h(X)}] = \E_p[g(X)] + \varepsilon\E_p[g(X)h(X)] + O(\varepsilon^2), 
\end{equation}
and similarly,
\begin{equation}\label{eq:simple-expansion}
\E_p[e^{-\varepsilon h(X)}]
= 1 - \varepsilon\,\E_p[h(X)] + O(\varepsilon^2).
\end{equation}
For now we assume that these expansions hold, in order 
to preserve the flow of the argument. The detailed derivation is deferred to the end of this subsection. 
Multiplying the numerator and denominator expansions yields, with \eqref{proof:first}, that
\[
\E_{q_\varepsilon}[g(X)]
= \E_p[g(X)]
+ \varepsilon\Bigl(\E_p[g(X)h(X)] - \E_p[g(X)]\,\E_p[h(X)]\Bigr)
+ O(\varepsilon^2),
\]
which is exactly
\[
\E_{q_\varepsilon}[g(X)]
= \E_p[g(X)]
+ \varepsilon\,\mathrm{Cov}_p(g(X),h(X))
+ O(\varepsilon^2).
\]

Applying the same argument with $g(X)^2$ in place of $g(X)$ yields
\[
\E_{q_\varepsilon}[g(X)^2]
= \E_p[g(X)^2]
+ \varepsilon\,\mathrm{Cov}_p(g(X)^2,h(X))
+ O(\varepsilon^2).
\]


\textbf{Specialisation to the Stein quantity \texorpdfstring{$g=\mathcal L_p f$}{g=Lpf}.}\label{sec:expansion-derivation}
Plugging in $g=\mathcal L_p f$ we obtain the deterministic expansions
\[
\E_{q_\varepsilon}[\mathcal L_p f(X)]
= \varepsilon\,\mathrm{Cov}_p(\mathcal L_p f(X), h(X))
  + O(\varepsilon^2),
\]
and
\[
\E_{q_\varepsilon}[(\mathcal L_p f(X))^2]
= \E_p[(\mathcal L_p f(X))^2]
  + \varepsilon\,\mathrm{Cov}_p((\mathcal L_p f(X))^2, h(X))
  + O(\varepsilon^2).
\]
Since the mean is $O(\varepsilon)$, its square is $O(\varepsilon^2)$, hence
\[
\Var_{q_\varepsilon}[\mathcal L_p f(X)]
=
\Var_p[\mathcal L_p f(X)]
+ \varepsilon\,\mathrm{Cov}_p((\mathcal L_p f(X))^2, h(X))
+ O(\varepsilon^2).
\]

\textbf{Proof of the first-order expansion of $\mathbb{E}[g e^{\varepsilon h}]$.}

It now remains to prove the expressions \eqref{eq:expansion_to_prove} and \eqref{eq:simple-expansion} used above. We begin with the Taylor expansion of the exponential function $f(y) = e^y$ around $y=0$. For any real scalar $y \in \mathbb{R}$, we have:
\[
e^y = 1 + y + R(y),
\]
where the remainder term $R(y)$ is bounded by the Lagrange error bound:
\begin{equation} \label{eq:lagrange}
|R(y)| \leq \frac{y^2}{2} e^{|y|}.
\end{equation}
For any fixed vector $x$, let $y = \varepsilon h(x)$. Substituting this into the expression above and multiplying by $g(x)$, we obtain:
\begin{align*} 
g(x) e^{\varepsilon h(x)} = & g(x)\Big(1 + \varepsilon h(x) + R(\varepsilon h(x))\Big) \\ = & g(x) + \varepsilon g(x)h(x) + g(x) R(\varepsilon h(x)). 
\end{align*}
Taking the expectations:
\begin{equation} \label{eq:expansion}
\mathbb{E}_p[g(x) e^{\varepsilon h(x)}] = \mathbb{E}_p[g(x)] + \varepsilon \mathbb{E}_p[g(x)h(x)] + \mathbb{E}_p[g(x) R(\varepsilon h(x))].
\end{equation}
We next show that the remainder term $E_{\text{err}} = \mathbb{E}_p[g(x) R(\varepsilon h(x))]$ satisfies $|E_{\text{err}}| \leq C \varepsilon^2$ for some constant $C$. Using the bound from \eqref{eq:lagrange}:
\[
|E_{\text{err}}| \leq \mathbb{E}_p\left[|g(x)| \cdot |R(\varepsilon h(x))|\right] \leq \mathbb{E}_p\left[ |g(x)| \frac{\varepsilon^2 h(x)^2}{2} e^{\varepsilon |h(x)|} \right].
\]
Factoring out constants:
\[
|E_{\text{err}}| \leq \frac{\varepsilon^2}{2} \mathbb{E}_p\left[ |g(x)| h(x)^2 e^{\varepsilon |h(x)|} \right].
\]
We now apply the Cauchy-Schwarz inequality to bound the expectation:
\[
\mathbb{E}_p\left[ |g(x)| \left( h(x)^2 e^{\varepsilon |h(x)|} \right) \right] \leq \lVert g \rVert_{L^2(p)} \cdot \lVert h^2 e^{\varepsilon |h|} \rVert_{L^2(p)}.
\]
We note that $\lVert g \rVert_{L^2(p)}$ is finite by assumption.
It remains to show that the second term is bounded. The norm is given by:
\[
\lVert h^2 e^{\varepsilon |h|} \rVert_{L^2(p)} = \left( \mathbb{E}_p\left[ h(x)^4 e^{2\varepsilon |h(x)|} \right] \right)^{1/2}.
\]
By the exponential-integrability assumption, the moment generating function of
the scalar random variable $|h(X)|$ exists in a neighbourhood of zero, which in
particular implies the finiteness of all polynomial moments of $h$ under $p$.
Moreover, exponential integrability allows control of mixed
polynomial–exponential moments.  Fix $\varepsilon>0$ sufficiently small, for
instance $\varepsilon<\delta/4$.  Applying the Cauchy--Schwarz inequality, we
obtain
\[
\mathbb{E}_p\!\big[h(X)^4 e^{2\varepsilon|h(X)|}\big]
\le
\big(\mathbb{E}_p[h(X)^8]\big)^{1/2}
\big(\mathbb{E}_p[e^{4\varepsilon |h(X)|}]\big)^{1/2}.
\]
Since $\varepsilon<\delta/4$, the exponential-integrability assumption on $h$ implies
$\mathbb{E}_p[e^{4\varepsilon |h(X)|}]<\infty$, and exponential integrability
also ensures $\mathbb{E}_p[h(X)^8]<\infty$.  Consequently, the mixed moment
$\mathbb{E}_p[h(X)^4 e^{2\varepsilon|h(X)|}]$ is finite for all
$\varepsilon<\delta/4$, justifying the uniform remainder bound used in \eqref{eq:expansion_to_prove}. Expansion \eqref{eq:simple-expansion} follows 
similarly by using the expansion $e^{-\epsilon h(x)} = 1 - \epsilon h(x) + \tilde{R}(\epsilon h(x)).$

For completeness, we state the explicit bound. Let $K = \lVert g \rVert_{L^2(p)} \cdot \lVert h^2 e^{\varepsilon |h|} \rVert_{L^2(p)} < \infty$. Then
\[
|E_{\text{err}}| \leq \frac{K}{2} \varepsilon^2.
\]

Thus, $E_{\text{err}} = O(\varepsilon^2)$, completing the proof.
\end{proof}

\subsection{Directional score-model error decomposition}
\label{app:directional}

Here we provide the full derivation of \eqref{eq:directional-main} and the
tilt-specialisation \gr{\eqref{eq:tilt-expansion-main}} used in the main text. 
Let $\tilde p$ be a learned density model with score $\tilde s(x)=\nabla\log\tilde p(x)$.
Define
\begin{align*}
  \mathcal{L}_{\tilde p} f(x)
  & :=
  \Delta f(x) + \tilde s(x)^\top \nabla f(x), \\
  g(x)&:=(\tilde s(x)-s_p(x))^\top \nabla f(x)\end{align*}
  and the likelihood ratio 
  \[ 
  l(x):=\frac{q(x)}{p(x)}.
\] 

\begin{proposition}[Directional decomposition]
\label{prop:directional-app}
Assume that $p$ is a continuously differentiable probability density, that $q$ is absolutely continuous w.r.t. $p$, that $f\in\mathcal{F}(p) \cap \mathcal{F}(q)$, and that $\mathcal{L}_{\tilde p}f$ and $g$ are integrable
under both $p$ and $q$. Then
\[
  \E_q[\mathcal{L}_{\tilde p} f]
  =
  \E_p[\mathcal{L}_{\tilde p} f]
  +
  S_f(p,q)
  +
  \E_p[g(X)(l(X)-1)].
\]
\end{proposition}
\begin{proof}
We begin with the decomposition
\[
  \mathcal{L}_{\tilde p} f
  = \mathcal{L}_p f
    + (\tilde s - \nabla\log p)^\top \nabla f,
\]
which is valid for all $f\in\mathcal{F}(p) \cap \mathcal{F}(q)$.
Taking expectations under $q$ and $p$ and subtracting gives
\begin{equation}
  \label{eq:directional}
  \E_q[\mathcal{L}_{\tilde p} f]
  - \E_p[\mathcal{L}_{\tilde p} f]
  =
  \E_q[\mathcal{L}_p f]
  + \Bigl(\E_q - \E_p\Bigr)
       \bigl[(\tilde s - \nabla\log p)^\top\nabla f \bigr]. 
\end{equation}
Here we used the Stein identity,  Proposition \ref{prop:stein-identity-app}. 

The first term on the right-hand side of \eqref{eq:directional} is $S_f(p,q)$ by definition.
For the second term, write $q(x)=l(x)p(x)$ with $l=q/p$, and $g(x)=(\tilde s(x)-\nabla\log p(x))^\top\nabla f(x)$,  giving
\[
  \Bigl(\E_q - \E_p\Bigr)[g(X)]
  = \E_p[g(X)(l(X)-1)].
\]
Thus we obtain  \eqref{eq:directional-main}.

For the exponential-tilt case $q_\varepsilon(x)\propto p(x)e^{\varepsilon h(x)}$,
the expansion follows directly from the
change-of-measure expansion proved in Appendix~\ref{app:tilt}, combined with the fact that
$l_\varepsilon(x)=1+\varepsilon\,(h(x)-\E_p[h]) + O(\varepsilon^2)$, which can be established using the same technique as in Appendix~\ref{app:tilt}.
This yields
\[
  \E_p[g(X)(l_\varepsilon(X)-1)]
  = \varepsilon\,\Cov_p(g,h) + O(\varepsilon^2).
\]
\end{proof}

\subsection{Score-model error decomposition and bound}

\label{app:score-model-stability-bounds}
\begin{proposition}[Fisher-controlled stability to score-model error]
\label{lem:fisher-stability}
Let $p$ be the continuously differentiable training density on $\mathbb{R}^d$, let $\tilde p$ be a learned
density with score $\tilde s = \nabla \log \tilde p$, and let $q$ be a test
density such that $q \ll p$ with density ratio $l(x) = q(x)/p(x)$.
Let $s_p = \nabla \log p$ and define
\[
  g(x)
  := \bigl(\tilde s(x) - s_p(x)\bigr)^\top \nabla f(x).
\]
Assume $\tilde s - s_p \in L^2(p)$,  $\nabla f \in L^4(p)$, and $l \in L^4(p)$.
Then the score-model error term satisfies
\[
  \bigl|\langle g,\ l-1\rangle_{L^2(p)}\bigr|
  \;\le\;
  \sqrt{J(p\|\tilde p) }\;
  \|\nabla f\|_{L^4(p)}\;
  \|l-1\|_{L^4(p)},
\]
where
\[
  J(p\|\tilde p)
  := \E_p\bigl[\|\tilde s(X) - s_p(X)\|^2\bigr]
\]
is the Fisher divergence of $\tilde p$ from $p$.
\end{proposition}

\begin{proof}
By definition,
\[
  \langle g,\ l-1\rangle_{L^2(p)}
  =
  \E_p\!\left[(\tilde s(X)-s_p(X))^\top \nabla f(X)\,(l(X)-1)\right].
\]
Applying Cauchy--Schwarz in $L^2(p)$ by grouping the score-model error separately gives 
\[
\begin{aligned}
\bigl|\E_p[g(X)(l(X)-1)]\bigr|
&\le
\Bigl(\E_p\|\tilde s(X)-s_p(X)\|^2\Bigr)^{1/2}
\Bigl(\E_p\bigl[\|\nabla f(X)\|^2(l(X)-1)^2\bigr]\Bigr)^{1/2}.
\end{aligned}
\]
The first factor equals $\sqrt{J(p\|\tilde p)}$ by definition.
For the second factor, we apply Hölder's inequality with conjugate exponents $2$ and $2$:
\[
  \E_p\bigl[\|\nabla f(X)\|^2(l(X)-1)^2\bigr]
  \le
  \bigl(\E_p\|\nabla f(X)\|^{4}\bigr)^{1/2}
  \bigl(\E_p|l(X)-1|^{4}\bigr)^{1/2}.
\]
Taking square roots of both sides gives
\begin{align*}  
  \Bigl(\E_p\bigl[\|\nabla f(X)\|^2(l(X)-1)^2\bigr]\Bigr)^{1/2}
  & \le
  \bigl(\E_p\|\nabla f(X)\|^{4}\bigr)^{1/4}
  \bigl(\E_p|l(X)-1|^{4}\bigr)^{1/4}\\
&
  =
  \|\nabla f\|_{L^4(p)}\,\|l-1\|_{L^4(p)}.
\end{align*}
Combining the bounds yields
\[
  \bigl|\E_p[g(X)(l(X)-1)]\bigr|
  \le
  \sqrt{J(p\|\tilde p)}\;
  \|\nabla f\|_{L^4(p)}\;
  \|l-1\|_{L^4(p)},
\]
as claimed.
\end{proof}

\noindent\textbf{Remark.}
The same argument applies with any Hölder pair $p,q>1$ satisfying
$1/p+1/q=1$, yielding bounds involving higher-order $L^{2p}(p)$ and
$L^{2q}(p)$ norms of $\nabla f$ and the density ratio $l-1$, respectively.

\section{TASTE: algorithm} \label{app:alg}

\gr{Algorithm \ref{alg:merged-stein-batched-clean} gives pseudocode for obtaining TASTE residuals. It includes the option to estimate the Laplacian via Hutchinson's method with $K$ samples, where $K$ is chosen by the user. The option whether or not to compute $D$ reflects whether or not the density $p$ under training is estimated. The algorithm calculates batches of adjusted Stein residuals.} 
\begin{algorithm}[h!]
\caption{Batched Adjusted TASTE Residuals}
\label{alg:merged-stein-batched-clean}
\begin{algorithmic}[1]
\STATE \textbf{Input:} test set $\mathcal{X}_{\mathrm{test}}$, (optional) calibration set $\mathcal{X}_{\mathrm{train}}$, predictor $f_\theta$, score model $\hat{s}_p$, batch sizes $B_{\mathrm{train}},B_{\mathrm{test}}$, options: \texttt{compute\_D} (bool), \texttt{per\_input} (bool),  Hutchinson samples $K$
\STATE \textbf{Helper 1:} \texttt{STEIN\_OP\_BATCH}($X$)
\STATE \quad Compute gradients $G \leftarrow \nabla f_\theta(X)$ for batch $X$
\STATE \quad Estimate Laplacians $L \leftarrow \Delta f_\theta(X)$ for batch $X$ (exact or Hutchinson with $K$ samples)
\STATE \quad Compute batch values $U \leftarrow L + \hat{s}_p(X)^\top G$ \quad (vector of length $|X|$)
\STATE \quad \textbf{Return} $U$

\STATE \textbf{Helper 2:} \texttt{STEIN\_OP\_BATCH\_ELEMENTWISE}($X$)
\STATE \quad Compute gradients $G \leftarrow \nabla f_\theta(X)$ for batch $X$
\STATE \quad Compute elementwise second-order term $L \leftarrow \Delta_{\mathrm{elem}} f_\theta(X)$ for batch $X$
\STATE \quad Compute elementwise contributions
\STATE \quad \quad $U \leftarrow L + \hat{s}_p(X) \odot G$
\STATE \quad \textbf{Return} $U$ \quad (matrix of shape $|X|\times d$)

\IF{\texttt{per\_input} is true}
  \STATE $\texttt{STEIN\_OP} \leftarrow \texttt{STEIN\_OP\_BATCH\_ELEMENTWISE}$
\ELSE
  \STATE $\texttt{STEIN\_OP} \leftarrow \texttt{STEIN\_OP\_BATCH}$
\ENDIF

\IF{\texttt{compute\_D} is true}
  \STATE $S \leftarrow 0,\; n \leftarrow 0$
  \FOR{each training batch $X_{\mathrm{train},b}$ of size $B_{\mathrm{train}}$}
    \STATE $U_b \leftarrow \texttt{STEIN\_OP}(X_{\mathrm{train},b})$
    \STATE $S \leftarrow S + \sum U_b,\; n \leftarrow n + |U_b|$
  \ENDFOR
  \STATE $D_f \leftarrow S / n$
\ELSE
  \STATE $D_f \leftarrow 0$
\ENDIF
\STATE Initialize list of scores $\text{Scores}\leftarrow\varnothing$
\FOR{each test batch $X_{\mathrm{test},b}$ of size $B_{\mathrm{test}}$}
  \STATE $U_b \leftarrow \texttt{STEIN\_OP}(X_{\mathrm{test},b})$
  \STATE $R_b \leftarrow U_b - D_f$ \quad (vector of adjusted \gr{TASTE} residuals)
  \STATE Append $\{R_b\}$ to $\text{Scores}$ 
\ENDFOR
\STATE \textbf{Output:} \(\text{Scores}\)
\end{algorithmic}
\end{algorithm}

\section{Experimental setup and implementation details}
\label{app:experimental-setup}

Here we provide implementation details for all experiments reported in
Section~\ref{sec:experiments}, including model architectures, training
procedures, score estimation, and evaluation protocols. All experiments are conducted in a fully \emph{unsupervised} setting:
neither the classifier nor the score model is trained or fine-tuned on any OOD data. We use Adam \citep{kingmaAdamMethodStochastic2017} for all experiments.

\subsection{Experimental probes and scoring protocols}
\label{app:taste-protocols}

We distinguish three operations: selecting a scalar probe,
constructing its per-example residual, and aggregating the
resulting residuals. These choices are specified separately
because they serve different purposes and are not
interchangeable.

Let $r(x)$ denote the signed per-example residual produced
by the chosen implementation. In the Gaussian and MNIST
shift experiments, we report the dataset-level statistic
\[
    \widehat{S}
    = \frac{1}{n}\sum_{i=1}^{n} r(x_i),
\]
computed separately for each shift type and magnitude.
The averaging is applied to signed residuals, without an
absolute-value transformation.

For CIFAR-10 and the original ImageNet-1K detector,
we instead use the signed residual $r(x)$ directly as a
per-example detection score. We compute AUROC and FPR95
from these scores, without averaging them over the
evaluation dataset.

For MVTec AD, we retain the coordinate-wise residuals
rather than reducing them to a single per-image value,
and use these residuals in the anomaly-map construction.
The ImageNet-1K classwise variant uses the separate
aggregation rule in Appendix~\ref{app:taste-classwise}.

\begin{table}[t]
\centering
\small
\setlength{\tabcolsep}{3pt}
\renewcommand{\arraystretch}{1.12}
\caption{Scoring and aggregation protocols.
Gaussian experiments use the scalar model output;
TASTE$_{\mathrm{ext}}$ uses all classwise softmax probabilities;
the remaining experiments use the predicted-class softmax
probability, holding the selected class fixed during differentiation.
All scores and sums shown are signed.
The \texttt{per\_input} flag refers to
Algorithm~\ref{alg:merged-stein-batched-clean};
``--'' denotes the separate JVP construction in
Appendix~\ref{app:taste-classwise}.}
\label{tab:taste-protocols}

\begin{tabularx}{\linewidth}{
    @{}
    >{\raggedright\arraybackslash}p{0.18\linewidth}
    >{\raggedright\arraybackslash}p{0.10\linewidth}
    c
    >{\raggedright\arraybackslash}p{0.12\linewidth}
    >{\raggedright\arraybackslash}X
    >{\raggedright\arraybackslash}p{0.18\linewidth}
    @{}
}
\toprule
\textbf{Experiment}
& \textbf{Level}
& \texttt{per\_input}
& \textbf{Score used}
& \textbf{Further aggregation}
& \textbf{Reported} \\
\midrule

Gaussian shifts
& Dataset
& \texttt{False}
& $r_f(x)$
& $\frac{1}{n}\sum_{i=1}^{n}r_f(x_i)$
  per shifted dataset
& Shift-response plots \\
\addlinespace

MNIST rotations and translations
& Dataset
& \texttt{False}
& $r_f(x)$
& $\frac{1}{n}\sum_{i=1}^{n}r_f(x_i)$
  per transformation and magnitude
& Shift-response plots \\
\addlinespace

CIFAR-10 detection
& Datapoint
& \texttt{False}
& $r_f(x)$
& None
& AUROC, FPR95 \\
\addlinespace

ImageNet-1K: predicted class
& Datapoint
& \texttt{False}
& $r_f(x)$
& None
& AUROC, FPR95 \\
\addlinespace

ImageNet-1K: TASTE$_{\mathrm{ext}}$
& Datapoint
& --
& Classwise $d_c(x)$
& $\sum_{c\in\mathcal{E}(x)}d_c(x)$:
  largest and smallest $5\%$ per input
& AUROC, FPR95 \\
\addlinespace

MVTec AD localisation
& Coordinate
& \texttt{True}
& $r_{f,i}(x)$
& Sum RGB-channel contributions per pixel
& Pixel AUROC, AUPRO; maps thresholded at $q_{0.99}$ \\

\bottomrule
\end{tabularx}
\end{table}

\paragraph{Baseline correction and ranking metrics.}
A common scalar baseline correction applied to the final
per-example score leaves its ordering unchanged and
therefore does not affect AUROC or FPR95.
This observation concerns ranking-based evaluation;
the numerical value of the correction remains relevant
when interpreting a dataset-level residual or applying
a fixed numerical threshold.

\subsection{First-order classwise aggregation for ImageNet-1K}
\label{app:taste-classwise}

For ImageNet-1K, approximating the Laplacian term in TASTE to a reasonable standard becomes computationally infeasible. Hence TASTE has to be run based on the score-part only. However, the predicted-class probe in TASTE examines one class probability
at each input. To capture responses across the larger
ImageNet-1K output space, TASTE$_{\mathrm{ext}}$ instead
computes a first-order score for every class and then
aggregates a selected subset.

Let $z(x)\in\mathbb{R}^{C}$ denote the classifier logits
and write
\[
f_c(x)=\sigma_c(z(x)),
\qquad
\sigma_c(z)=\frac{\exp(z_c)}{\sum_{j=1}^{C}\exp(z_j)},
\]
so that $f_c$ is the class-$c$ softmax probe used
throughout the paper; let $\hat{s}_p(x)\in\mathbb{R}^{d}$ denote the score
estimate.

The first-order TASTE score for class
$c$ is therefore
\[
d_c(x)=\nabla f_c(x)^\top\hat{s}_p(x).
\]

To compute these scores efficiently for a set of classes, we first differentiate
the logits along the same direction.
Let $J_z(x)\in\mathbb{R}^{C\times d}$ be the input
Jacobian of the logits, with
$[J_z(x)]_{ji}=\partial z_j(x)/\partial x_i$.
A single Jacobian--vector product (JVP) gives
\[
u(x)=J_z(x)\hat{s}_p(x),
\qquad
u_j(x)=\nabla z_j(x)^\top\hat{s}_p(x).
\]
Thus, $u_j(x)$ describes the directional change in
logit $j$, and all $C$ such changes are obtained together.

We next convert the logit changes into probability changes.
The softmax derivative is
\[
\frac{\partial \sigma_c(z)}{\partial z_j}
=
\sigma_c(z)\bigl(\delta_{cj}-\sigma_j(z)\bigr),
\]
where $\delta_{cj}=1$ when $c=j$ and is zero otherwise.
Applying the chain rule at $z=z(x)$ yields
\begin{align*}
d_c(x)
&=
\sum_{j=1}^{C}
\frac{\partial \sigma_c}{\partial z_j}(z(x))
\,\nabla z_j(x)^\top\hat{s}_p(x)\\
&=
\sum_{j=1}^{C}
f_c(x)\bigl(\delta_{cj}-f_j(x)\bigr)u_j(x)\\
&=
f_c(x)\left(
u_c(x)-\sum_{j=1}^{C}f_j(x)u_j(x)
\right).
\end{align*}
The final expression requires only the probabilities
$f_j(x)$ and the JVP output $u(x)$.
Consequently, all classwise scores can be computed
without constructing a full Jacobian or performing
a separate backward pass for each class.

Write $d(x)=(d_1(x),\ldots,d_C(x))^\top$.
Since the class probabilities sum to one,
\[
\sum_{c=1}^{C}d_c(x)
=
\left(\nabla\sum_{c=1}^{C}f_c(x)\right)^\top
\hat{s}_p(x)
=0.
\]
Summing all classwise scores would therefore produce
an identically zero statistic.

Instead, for ImageNet-1K, where $C=1000$, we sort the
entries of $d(x)$ and retain the largest
50 and smallest 50 entries.
These correspond to the upper and lower $5\%$ of classes.
Let
\[
\mathcal{E}(x)
=
\operatorname{Top}_{50}(d(x))
\cup
\operatorname{Bottom}_{50}(d(x)),
\]
where $\operatorname{Top}_{50}$ and
$\operatorname{Bottom}_{50}$ return the corresponding
class indices.
The final per-example OOD TASTE$_{\mathrm{ext}}$ statistic is the signed sum
\[
T_{\mathrm{ext}}(x)
=
\sum_{c\in\mathcal{E}(x)}d_c(x).
\]

\subsection{Controlled 2D directional shift experiment}
\label{app:2d-setup}

\textbf{Data generation.}
Training samples are drawn from the standard 2-dimensional Gaussian
$p(x)=\mathcal{N}(0,I_2)$. Test distributions are obtained by translating $p$ by
a fixed magnitude $\varepsilon=10$ along directions
$\varepsilon R_\varphi u$, where $u=(1,1)^\top/\sqrt{2}$ and
$\varphi\in[0,2\pi)$.
Each test set contains $1000$ samples.

\textbf{Model.}
The model is a fully connected ReLU network with a single hidden layer of size $64$. The model is trained on samples from $p$ using
learning rate $0.001$, for $100$ epochs.

\textbf{TASTE calculation.}
Since $p$ is known analytically, the score
$s_p(x)=-\;x$ is used directly. Adjusted Stein residuals $r_f(x)$ are computed using
Algorithm~\ref{alg:merged-stein-batched-clean} and averaged across each test set to obtain an estimate of  $S_f(p,q_\varphi)$.

\subsection{MNIST geometric perturbation experiments}
\label{app:mnist-setup}

\textbf{Dataset and preprocessing.}
MNIST images are zero-padded from $28\times28$ to $64\times64$ prior to training
and evaluation.
This padding ensures that translations do not remove semantic content, and it allows
invariance effects to be isolated from boundary artefacts.

\textbf{Perturbations.}
Test sets are generated using:
\begin{itemize}[leftmargin=*]
\item Translations with $\ell^\infty$ magnitude
$t\in\{0,2,4,\dots\}$. For each datapoint, the translation direction is chosen uniformly from a $\ell^\infty$ ball of desired radius.
\item Rotations with angle
$\alpha\in\{-90,-72,\dots,72\}$ degrees.
\end{itemize}

\textbf{Classifier.}
We use an AlexNet-style convolutional
classifier adapted to single-channel inputs, with five $3\times3$
convolutional layers, ReLU activations, and channel widths
$64$, $192$, $384$, $256$, and $256$.
Spatial downsampling is performed using average pooling, yielding feature maps of sizes $32\times32$ and $16\times16$, followed by adaptive average pooling to a $1\times1$ resolution. The classifier head is a three-layer fully connected component with sizes
($512\!\to\!128\!\to\!10$).

\textbf{Score model.}
We use a lightweight convolutional score network to
approximate the input-space score field $\hat{s}_p(x)\approx\nabla_x\log p(x)$.
The model consists of four $3\times3$ convolutional layers with ReLU activations
and channel widths $32$, $64$, $64$, and $1$, preserving the $64\times64$
spatial resolution throughout. The model is trained using Denoising Score Matching 
employing standard practices.

\textbf{Stein residuals.}
Adjusted Stein residuals $r_f(x)$ are computed using
Algorithm~\ref{alg:merged-stein-batched-clean}, with the correction term
$D_f$ estimated on held-out MNIST training samples. The Laplacian is  estimated using Hutchinson's estimator with the number of samples $K$ set to $5$.

\subsection{CIFAR-10 benchmark evaluation}
\label{app:cifar-setup}

\textbf{Classifier.}
The in-distribution classifier is a ResNet-18. We take model weights pre-trained on ImageNet and fine-tune them on CIFAR-10, with learning rate $0.0001$ for $5$ epochs.

\textbf{Score model.}
The score function $\hat s_p(x)$ is estimated using a pretrained diffusion model. Specifically, we obtain the weights from the publicly released \texttt{google/ddpm-cifar10-32} checkpoint. The model is trained to predict the additive noise $\varepsilon$ in the diffusion process, from which the score is obtained via the standard conversion $\hat s_p(x_t) = -\varepsilon_\theta(x_t,t)/\sigma_t$. 

While the diffusion model is trained over 1000 noise levels, we empirically find that evaluating the score at a relatively low time step ($t=50$ out of $1000$) yields a useful balance between proximity to the clean data score and numerical stability. Lower noise levels (smaller $t$) tend to produce noisier or unstable score estimates, whereas higher noise levels carry less information about the data distribution. Prior work on score-based models suggests that intermediate noise levels often provide the most robust signal for downstream tasks \citep{songScoreBasedGenerativeModeling2021,hoDenoisingDiffusionProbabilistic2020}, even though they are not typically interpreted as clean-data scores per se. For empirical analysis see Appendix~\ref{sec:diffusion-time step-sensitivity}.


\subsection{ImageNet-1K experiments}

\textbf{Classifier.}
For ImageNet-1K experiments we use a ResNet-50 backbone pretrained on ImageNet-1K available via torchvision package (\texttt{ResNet50\_Weights.IMAGENET1K\_V1}: https://download.pytorch.org/models/resnet50-0676ba61.pth).

\textbf{Score model.}
The score function $\hat s_p(x)$ is estimated using a pretrained unconditional diffusion model trained on $256\times256$ ImageNet images. Specifically, we use the publicly released OpenAI diffusion checkpoint from the guided diffusion framework: https://github.com/openai/guided-diffusion. As in the CIFAR-10 experiments, the diffusion model predicts the additive noise $\varepsilon$ in the forward diffusion process, and the score estimate is obtained via the standard relation
\[
\hat s_p(x_t)
=
-\varepsilon_\theta(x_t,t)/\sigma_t.
\]

All ImageNet experiments use the same diffusion timestep configuration as the CIFAR-10 setup.

\paragraph{Operator and probe choice.} Both ImageNet-1K variants omit the Laplacian term for computational reasons, a choice alsosupported by the evidence in Table~\ref{tab:c10-lap-ablation}.
The TASTE row uses the predicted-class softmax probe,
whereas TASTE$_{\mathrm{ext}}$ uses the classwise
aggregation in Appendix~\ref{app:taste-classwise}.

\subsection{Compute resources and runtime.}
\label{app:compute}
All experiments were conducted on a single NVIDIA H100 GPU. The full evaluation suite on ImageNet-1K  corresponds to approximately 240k evaluations per detector: 40k samples for ID train+calibration, 10k ID test samples, roughly 100k adversarial samples (across 10 attacks), and approximately 100k samples across classic OOD benchmarks. 

Depending on the detector, this results in evaluation costs ranging from approximately 0.4 GPU-hours (MSP/Energy) to over 10 GPU-hours (SCOPED) per full sweep. TASTE requires approximately 8 GPU-hours per full evaluation sweep. Summed across all reported detectors, the total evaluation cost is approximately 50 GPU-hours.

These estimates refer to evaluation only; the classifier and score-model training are not accounted for, as we used pretrained checkpoints. The full research project additionally required more compute for preliminary experiments, smaller-scale experiment, and ablations, but their total computational cost is substantially smaller than that of the reported ImageNet-1K experimental sweep. 

\subsection{Baselines and hyperparameters}
\label{app:baselines}

We evaluate against a broad range of established and recent OOD detection methods,
covering confidence-based, distance-based, energy-based, representation-based,
and post-hoc approaches.

Confidence-based baselines include Maximum Softmax Probability (MSP)
\citep{hendrycksBaselineDetectingMisclassified2018} and ODIN
\citep{liangEnhancingReliabilityOutofdistribution2020}.
Distance- and representation-based methods include Mahalanobis distance
\citep{leeSimpleUnifiedFramework2018} and kNN-based detectors
\citep{sunOutofDistributionDetectionDeep2022}.
Energy-based scoring is implemented following
\citep{liuEnergybasedOutofdistributionDetection2021}.

We additionally include several recent post-hoc and hybrid methods that improve OOD
separability through feature shaping, logit correction, perturbation sensitivity,
or score-based geometry, including ReAct \citep{sun2021react}, ViM \citep{wang2022vim}, ASH
\citep{djurisic2023ash}, PRO \citep{chen2025pro}, 
and SCOPED
\citep{barkley2025scoped}.

All baseline methods are implemented using publicly available codebases or
standard reference implementations, primarily via the \texttt{pytorch-ood}
library \citep{kirchheim2022pytorch}. Hyperparameters are selected according to
the respective original papers or default settings provided in the library.


\subsection{MVTec AD localisation experiments}
\label{app:mvtec-setup}

\textbf{Dataset.}
We use the MVTec Anomaly Detection dataset with the standard train/test split.
Only \texttt{train/good} images are used for calibration.

\textbf{Models.}
\begin{itemize}[leftmargin=*]
\item Classifier: ImageNet-pretrained ResNet-50, with weights downloaded through torchvision. 
\item Score model: ImageNet-pretrained diffusion model - obtained from the publicly available  Hugging Face  checkpoint \emph{google/ddpm-ema-celebahq-256}. Similarly to \ref{app:cifar-setup}, the score is obtained via the standard conversion $\hat s_p(x_t) = -\varepsilon_\theta(x_t,t)/\sigma_t$ with time step $t=50$ (out of $1000$).
\end{itemize}
No model is fine-tuned on MVTec AD. 

\textbf{Per-pixel Stein heatmaps.}
Per-pixel residuals are computed via the coordinate-wise Stein operator. Note that this gives a per-input score, and so there are three per pixel - one for each input channel. We take their sum to obtain the per-pixel score.
Heatmaps are resized to $256 \times 256$ for inference and upsampled to
ground-truth resolution for evaluation.

\textbf{Evaluation.}
We report pixel-level AUROC, average precision, and AUPRO.
Thresholds for visualisation are obtained via upper-tail calibration on
\texttt{train/good} images at quantile $\alpha=0.01$.

Additional qualitative results are shown in Appendix~\ref{ap:per-pixel-heatmaps}.

\subsection{On the Laplacian term for ReLU networks.}
For architectures with a piecewise-affine backbone (e.g.\ ResNets with ReLU and
pooling) followed by a smooth softmax output, classical second derivatives of
the backbone vanish almost everywhere.
Let $z(x)\in\mathbb{R}^K$ denote the logits and
$f_k(x)=\sigma_k(z(x))$ the softmax probabilities.
Since $\nabla_x^2 z_i(x)=0$ a.e., all non-zero curvature of $f_k$ arises from the
softmax via the chain rule, yielding
$$
\Delta f_k(x)
=
\sum_{i,j}
\frac{\partial^2 \sigma_k}{\partial z_i \partial z_j}\,
\langle \nabla_x z_i(x), \nabla_x z_j(x)\rangle.
$$ 

Thus, for such networks the Laplacian can be computed using only first-order
input gradients of the logits and the closed-form softmax Hessian, without 
second-order backpropagation through ReLU or pooling layers.
If the model has no smooth output nonlinearity (e.g.\ regression), the Laplacian
vanishes a.e.\ and may be omitted.
(For large $K$, the sum can \gr{even} be restricted to the top-$k$ logits for efficiency). 

This construction is used in the MVTecAD localisation experiments \ref{app:mvtec-setup} to calculate the per-input second derivative explicitly. 

\subsection{Comment on licences.} 
\label{app:licences}
All datasets, pretrained models, and benchmark checkpoints used in this work are publicly available and used in accordance with their respective licences and terms of use. CIFAR-10 and MNIST are used under their standard academic usage conditions, while ImageNet-1K is used under the ImageNet non-commercial research access policy. Publicly released pretrained classifiers and diffusion checkpoints are used under their corresponding open-source or research licences.

\FloatBarrier
\section{Granular Experimental Results}\label{app:experimental_results}

\textbf{Benchmark datasets and shift regimes.}
We evaluate task-aware OOD detection across a broad set of datasets designed to
capture qualitatively different forms of distribution shift, ranging from
semantic dataset shift to low-level corruptions, structured geometric
perturbations, and adversarial examples, see also \cite{tajwarNoTrueStateoftheArt2021}.

\smallskip
\noindent\textbf{Classical OOD benchmarks.}
Following standard practice, we use SVHN, LSUN, iSUN, Describable Textures, and
Places365 as semantic OOD datasets with respect to CIFAR-10.
These datasets differ substantially from CIFAR-10 in terms of image content,
statistics, and semantics, and are commonly used to benchmark the ability of OOD
methods to detect inputs drawn from entirely different data-generating
processes.
SVHN consists of real-world digit images with markedly different backgrounds
and color statistics; LSUN and iSUN contain large-scale scene images; Textures
focuses on fine-grained texture patterns; and Places365 covers a wide range of
scene categories.
Together, these datasets probe classical semantic OOD detection, where the test
distribution differs from the training distribution at a high semantic level.

\smallskip
\noindent\textbf{CIFAR-10-C (common corruptions).}
To evaluate robustness under realistic covariate shift, we use CIFAR-10-C,
which applies 19 common image corruptions—such as Gaussian noise, blur, fog,
snow, contrast changes, and compression artifacts—each at five severity levels.
These corruptions alter low-level image statistics while largely preserving
high-level semantic content.
CIFAR-10-C is widely used to assess distribution shift arising from sensor noise,
environmental conditions, or data acquisition artifacts, and provides a
systematic test of how OOD methods respond to gradual, non-adversarial changes
in the input distribution.

\smallskip
\noindent\textbf{CIFAR-10-P (structured perturbations).}
We additionally evaluate on CIFAR-10-P, which applies temporally consistent and
structured geometric perturbations such as translations, rotations, scaling,
shearing, and motion blur.
Unlike CIFAR-10-C, these perturbations often leave pixel-level statistics
relatively unchanged while inducing changes that directly interact with the
classifier’s invariances and equivariances.
CIFAR-10-P therefore provides a particularly relevant testbed for
\emph{task-aware} OOD detection, as some perturbations (e.g.\ translations) are
largely benign for convolutional networks, whereas others (e.g.\ rotations)
can substantially degrade task performance despite minimal changes in overall
image statistics.

\smallskip
\noindent\textbf{Adversarial attacks.}
To probe behaviour under worst-case, task-adversarial distribution shifts, we
include adversarial examples generated using FGSM, PGD, and AutoAttack under
both $\ell^\infty$ and $\ell^2$ constraints, with multiple perturbation budgets.
These attacks are specifically constructed to induce classifier failure while
keeping perturbations visually subtle and often imperceptible.
Adversarial examples therefore represent a regime in which confidence-based and
likelihood-based OOD methods are known to struggle, as inputs may remain close
to the training distribution in pixel space while causing severe task-level
errors.
Including adversarial attacks allows us to assess whether task-aware Stein
residuals respond to shifts that directly affect the predictor, even when
traditional OOD signals fail.

\smallskip
\noindent\textbf{Summary.}
Across these datasets, we cover four complementary shift regimes:
\begin{enumerate}
    \item 
 semantic dataset shift (classical OOD benchmarks),
\item low-level covariate shift (CIFAR-10-C),
\item  structured, invariance-breaking perturbations (CIFAR-10-P),
and \item  worst-case, task-adversarial perturbations (adversarial attacks).
\end{enumerate}
This comprehensive evaluation enables a detailed assessment of how different
OOD detection methods respond to shifts that are benign, harmful, or adversarial
with respect to the underlying prediction task.

\smallskip
\noindent\textbf{Granular results.}
Each table in this section reports a per-dataset breakdown for a single OOD
detection method (see the baselines in Section~\ref{sec:benchmark-cifar} for details) and corresponds to one column of the aggregated results
presented in the main text. AUROC is the area under the receiver operating characteristic curve; FPR95 is the false positive rate at 95\% true positive rate.

These fine-grained results expose substantial variability in OOD detection
performance across datasets, corruption types, and shift regimes.
Such variability highlights the sensitivity of existing methods to the specific
nature of distribution shift and provides additional empirical support for the
central observation of \citet{tajwarNoTrueStateoftheArt2021}: namely, that no
single OOD detection method consistently dominates across all benchmarks, and
that comparative performance is strongly dataset-dependent. Tables~\ref{tab:granular1}~-~\ref{tab:granular2} show per-dataset results on CIFAR-10 for a subset of baseline methods. As in the results presented in the main text, TASTE uses predicted class probability as the probe.

\begin{table}[h]
\centering
\caption{OOD detection results: adversarial}
\resizebox{\textwidth}{!}{%
\begin{tabular}{llrrrrrr}
\toprule
OOD Dataset & metric & MSP & Energy & ODIN & Mahalanobis & kNN+  & TASTE  \\
\midrule
adversarial\_autoattack\_linf\_2\_255 & AUROC & 0.5019 & 0.4896 & 0.5393 & \textbf{0.5445} & 0.5276 &  0.4786 \\
 & FPR95 & 0.9611 & 0.9643 & \textbf{0.93} & 0.9371 & 0.9513 &  0.9571 \\
adversarial\_autoattack\_linf\_4\_255 & AUROC & 0.2841 & 0.289 & 0.4851 & 0.505 & 0.3839  & \textbf{0.6543} \\
 & FPR95 & 0.9681 & 0.9711 & 0.9532 & 0.9213 & 0.9711  & \textbf{0.893} \\
adversarial\_autoattack\_linf\_8\_255 & AUROC & 0.1557 & 0.1675 & 0.3319 & 0.7207 & 0.4341  & \textbf{0.8609} \\
 & FPR95 & 0.9688 & 0.9713 & 0.9718 & \textbf{0.513} & 0.9688  & 0.602 \\
adversarial\_fgsm\_linf\_4\_255 & AUROC & 0.6254 & 0.6039 & 0.5673 & 0.6172 & \textbf{0.6266} &  0.4866 \\
 & FPR95 & 0.923 & 0.9147 & 0.9306 & 0.9071 & \textbf{0.8999}  & 0.9508 \\
adversarial\_fgsm\_linf\_8\_255 & AUROC & 0.6075 & 0.6026 & 0.5151 & 0.6341 & \textbf{0.6388} & 0.4949 \\
 & FPR95 & 0.922 & 0.8993 & 0.9715 & 0.8969 & \textbf{0.8847} & 0.948 \\
adversarial\_pgd\_l2\_0.5\_steps=50 & AUROC & 0.2656 & 0.2741 & 0.5036 & 0.4714 & 0.3748 & \textbf{0.5727} \\
 & FPR95 & 0.998 & 0.9985 & 0.9418 & 0.9507 & 0.9918  & \textbf{0.9348} \\
adversarial\_pgd\_l2\_1.0\_steps=50 & AUROC & 0.05975 & 0.07737 & 0.3817 & 0.634 & 0.311  & \textbf{0.8612} \\
 & FPR95 & 1 & 1 & 0.9631 & 0.6251 & 0.9981  & \textbf{0.5297} \\
adversarial\_pgd\_linf\_2\_255\_steps=50 & AUROC & 0.5837 & 0.5521 & 0.5402 & \textbf{0.5937} & 0.5812  & 0.433 \\
 & FPR95 & 0.9646 & 0.9689 & \textbf{0.9205} & 0.9273 & 0.9484 &  0.9637 \\
adversarial\_pgd\_linf\_4\_255\_steps=50 & AUROC & 0.3609 & 0.3548 & \textbf{0.5535} & 0.5159 & 0.4419  & 0.5017 \\
 & FPR95 & 0.9908 & 0.9929 & \textbf{0.9151} & 0.9442 & 0.981  & 0.9486 \\
adversarial\_pgd\_linf\_8\_255\_steps=50 & AUROC & 0.1304 & 0.1389 & 0.4769 & 0.5827 & 0.3248 & \textbf{0.7578} \\
 & FPR95 & 0.9991 & 0.9997 & 0.9381 & \textbf{0.7652} & 0.996  & 0.7966 \\

\bottomrule
\end{tabular}%
}
\label{tab:granular1}
\end{table}

\begin{table}[H]
\centering
\caption{OOD detection results: CIFAR-10-C}
\label{tab:ood-cifar10c}\resizebox{\textwidth}{!}{%
\begin{tabular}{llrrrrrrr}
\toprule
OOD Dataset & metric & MSP & Energy & ODIN & Mahalanobis & kNN+ &  TASTE  \\
\midrule

cifar10c\_brightness & AUROC & 0.5321 & 0.5389 & 0.4928 & 0.5356 & \textbf{0.543} &  0.5373 \\
 & FPR95 & 0.9454 & 0.9354 & 0.9562 & 0.9452 & 0.939 &  \textbf{0.9335} \\
cifar10c\_contrast & AUROC & 0.6945 & 0.7206 & \textbf{0.802} & 0.5841 & 0.6653  & 0.6846 \\
 & FPR95 & 0.86 & 0.781 & \textbf{0.5962} & 0.929 & 0.8666 &  0.7849 \\
cifar10c\_defocus\_blur & AUROC & 0.5801 & 0.5763 & \textbf{0.5864} & 0.5667 & 0.5847 &  0.5842 \\
 & FPR95 & 0.9308 & 0.9212 & \textbf{0.8546} & 0.9318 & 0.9186  & 0.8969 \\
cifar10c\_elastic\_transform & AUROC & 0.6027 & \textbf{0.6091} & 0.5777 & 0.5549 & 0.5956  & 0.6088 \\
 & FPR95 & 0.9272 & 0.9236 & 0.8942 & 0.9504 & 0.9278  & \textbf{0.8505} \\
cifar10c\_fog & AUROC & 0.6217 & 0.6364 & \textbf{0.6875} & 0.5347 & 0.6108 &  0.6465 \\
 & FPR95 & 0.9018 & 0.8914 & \textbf{0.7756} & 0.9596 & 0.9116 &  0.8526 \\
cifar10c\_frost & AUROC & 0.6086 & \textbf{0.6278} & 0.528 & 0.5703 & 0.6049 &  0.627 \\
 & FPR95 & 0.9166 & 0.8964 & 0.957 & 0.9458 & 0.9148 & \textbf{0.8693} \\
cifar10c\_gaussian\_blur & AUROC & 0.617 & 0.6142 & \textbf{0.6462} & 0.6018 & 0.6192 &  0.6206 \\
 & FPR95 & 0.9222 & 0.9128 & \textbf{0.8006} & 0.9184 & 0.91  & 0.8583 \\
cifar10c\_gaussian\_noise & AUROC & 0.6341 & \textbf{0.6509} & 0.4403 & 0.6333 & 0.646 &  0.6332 \\
 & FPR95 & 0.8998 & 0.881 & 0.9884 & 0.8932 & \textbf{0.869}  & 0.8778 \\
cifar10c\_glass\_blur & AUROC & 0.6751 & \textbf{0.7104} & 0.5464 & 0.6415 & 0.6791 &  0.6725 \\
 & FPR95 & 0.882 & 0.8446 & 0.9454 & 0.8968 & 0.8704  & \textbf{0.7801} \\
cifar10c\_impulse\_noise & AUROC & 0.702 & \textbf{0.7305} & 0.4776 & 0.7233 & 0.7255 &  0.6812 \\
 & FPR95 & 0.865 & 0.8148 & 0.985 & \textbf{0.7908} & 0.7958  & 0.7985 \\
cifar10c\_jpeg\_compression & AUROC & 0.5505 & \textbf{0.5577} & 0.5125 & 0.5168 & 0.5397 &  0.5521 \\
 & FPR95 & 0.936 & 0.927 & 0.9338 & 0.9516 & 0.9344 &  \textbf{0.9042} \\
cifar10c\_motion\_blur & AUROC & 0.6302 & 0.6409 & 0.6239 & 0.6099 & \textbf{0.6453} &  0.6365 \\
 & FPR95 & 0.9102 & 0.9062 & 0.8426 & 0.9094 & 0.8946  & \textbf{0.8343} \\
cifar10c\_pixelate & AUROC & 0.5407 & 0.5479 & 0.5094 & 0.5214 & 0.5376  & \textbf{0.5486} \\
 & FPR95 & 0.9368 & \textbf{0.9178} & 0.9372 & 0.9408 & 0.9268 & 0.9193 \\
cifar10c\_saturate & AUROC & 0.5668 & 0.5892 & 0.5264 & 0.5833 & \textbf{0.5923}  & 0.5599 \\
 & FPR95 & 0.9298 & \textbf{0.9018} & 0.9292 & 0.9038 & 0.9024 & 0.9253 \\
cifar10c\_shot\_noise & AUROC & 0.5986 & \textbf{0.6101} & 0.437 & 0.5952 & 0.6074 &  0.6015 \\
 & FPR95 & 0.9104 & 0.8944 & 0.9848 & 0.9096 & \textbf{0.8914} &  0.9004 \\
cifar10c\_snow & AUROC & 0.6104 & 0.6159 & 0.4865 & 0.5842 & 0.6032 & \textbf{0.6169} \\
 & FPR95 & 0.9206 & 0.9046 & 0.972 & 0.9346 & 0.9134 & \textbf{0.8872} \\
cifar10c\_spatter & AUROC & 0.621 & \textbf{0.6448} & 0.4804 & 0.6007 & 0.626 &  0.6176 \\
 & FPR95 & 0.905 & \textbf{0.8742} & 0.9726 & 0.9044 & 0.887 &  0.892 \\
cifar10c\_speckle\_noise & AUROC & 0.6064 & \textbf{0.6209} & 0.444 & 0.6026 & 0.6138 &  0.6119 \\
 & FPR95 & 0.9118 & 0.8902 & 0.9842 & 0.912 & \textbf{0.8878}  & 0.9009 \\
cifar10c\_zoom\_blur & AUROC & 0.6434 & 0.645 & 0.6435 & 0.6118 & \textbf{0.6466} &  0.6444 \\
 & FPR95 & 0.895 & 0.8912 & \textbf{0.8146} & 0.9182 & 0.8952  & 0.8275 \\

\bottomrule
\end{tabular}%
}
\end{table}

\begin{table}[!t]
\centering
\caption{OOD detection results: CIFAR-10-P}
\label{tab:ood-cifar10p}
\resizebox{\textwidth}{!}{%
\begin{tabular}{llrrrrrr}
\toprule
OOD Dataset & metric & MSP & Energy & ODIN & Mahalanobis & kNN+ &  TASTE \\
\midrule

cifar10p\_brightness & AUROC & 0.5497 & 0.5549 & 0.499 & 0.5488 & \textbf{0.5583} &  0.5557 \\
 & FPR95 & 0.9395 & 0.9357 & 0.9585 & 0.9415 & 0.934 & \textbf{0.9205} \\
cifar10p\_gaussian\_noise & AUROC & 0.5069 & \textbf{0.5107} & 0.4618 & 0.4965 & 0.5076 &  \textbf{0.5107} \\
 & FPR95 & 0.9473 & \textbf{0.9454} & 0.9629 & 0.9559 & 0.9505 &  0.9456 \\
cifar10p\_motion\_blur & AUROC & 0.6294 & 0.6368 & 0.6134 & 0.6183 & \textbf{0.6446} &  0.6342 \\
 & FPR95 & 0.9197 & 0.9203 & 0.8495 & 0.9152 & 0.9095  & \textbf{0.8123} \\
cifar10p\_rotate & AUROC & 0.7012 & \textbf{0.7351} & 0.5982 & 0.6109 & 0.6842 &  0.6822 \\
 & FPR95 & 0.8787 & 0.8463 & 0.9183 & 0.9477 & 0.8917 &  \textbf{0.731} \\
cifar10p\_scale & AUROC & 0.6425 & 0.636 & 0.6375 & 0.6235 & \textbf{0.6482} &  0.6448 \\
 & FPR95 & 0.9102 & 0.9122 & 0.8345 & 0.9203 & 0.9056 &  \textbf{0.8149} \\
cifar10p\_shear & AUROC & 0.6641 & \textbf{0.6996} & 0.5723 & 0.5878 & 0.6514 &  0.6496 \\
 & FPR95 & 0.8966 & 0.8619 & 0.921 & 0.938 & 0.8969  & \textbf{0.8146} \\
cifar10p\_shot\_noise & AUROC & 0.5136 & \textbf{0.5184} & 0.4522 & 0.4989 & 0.5137 &  0.5178 \\
 & FPR95 & 0.9444 & \textbf{0.9422} & 0.9708 & 0.9559 & 0.946 &  0.9459 \\
cifar10p\_snow & AUROC & 0.5989 & \textbf{0.6132} & 0.475 & 0.5697 & 0.5946 &  0.6041 \\
 & FPR95 & 0.9179 & 0.905 & 0.9735 & 0.9357 & 0.9116 & \textbf{0.8996} \\
cifar10p\_tilt & AUROC & 0.5564 & \textbf{0.5699} & 0.5331 & 0.5059 & 0.5428 &  0.5572 \\
 & FPR95 & 0.9378 & 0.9359 & 0.9293 & 0.9629 & 0.947 &  \textbf{0.8954} \\
cifar10p\_translate & AUROC & 0.6523 & \textbf{0.6884} & 0.6304 & 0.5762 & 0.6388 &  0.6318 \\
 & FPR95 & 0.9019 & 0.8788 & 0.9218 & 0.959 & 0.9148  & \textbf{0.7862} \\
cifar10p\_zoom\_blur & AUROC & 0.5605 & 0.5583 & 0.5526 & 0.5457 & 0.567 & \textbf{0.5676} \\
 & FPR95 & 0.9374 & 0.9411 & 0.9044 & 0.946 & 0.9379 &  \textbf{0.8937} \\

\bottomrule
\end{tabular}%
}
\end{table}

\begin{table}[!t]
\centering
\caption{OOD detection results: OOD benchmarks}
\label{tab:ood-benchmarks}
\resizebox{\textwidth}{!}{%
\begin{tabular}{lllrrrrr}
\toprule
OOD Dataset & metric & MSP & Energy & ODIN & Mahalanobis & kNN+ &  TASTE  \\
\midrule

iSUN & AUROC & 0.7926 & \textbf{0.8561} & 0.682 & 0.8379 & 0.8203  & 0.7408 \\
 & FPR95 & 0.7848 & 0.6455 & 0.9396 & 0.6829 & 0.6995  & \textbf{0.5901} \\
LSUN & AUROC & 0.8208 & \textbf{0.8999} & 0.7047 & 0.88 & 0.8677  & 0.7666 \\
 & FPR95 & 0.7391 & \textbf{0.5161} & 0.9421 & 0.5912 & 0.5945  & 0.5364 \\
Places365 & AUROC & 0.8005 & \textbf{0.854} & 0.7206 & 0.7943 & 0.8335  & 0.7746 \\
 & FPR95 & 0.7878 & 0.6461 & 0.86 & 0.7714 & 0.6894 & \textbf{0.5609} \\
SVHN & AUROC & 0.8044 & 0.816 & \textbf{0.9216} & 0.7705 & 0.8122 &  0.749 \\
 & FPR95 & 0.8258 & 0.8118 & \textbf{0.4123} & 0.8402 & 0.7944  & 0.5432 \\
Textures & AUROC & 0.7976 & \textbf{0.8339} & 0.7931 & 0.8113 & 0.8101  & 0.7445 \\
 & FPR95 & 0.8018 & 0.7222 & 0.7239 & 0.7119 & 0.7358 &  \textbf{0.592} \\

\bottomrule
\end{tabular}%
}
\label{tab:granular2}
\end{table}

\FloatBarrier   
\newpage        

\section{Per-pixel anomaly heatmaps}\label{ap:per-pixel-heatmaps}

\begin{figure}[H]   
    \centering
    \vspace{-0.6em} 
    \includegraphics[width=0.55\textwidth]{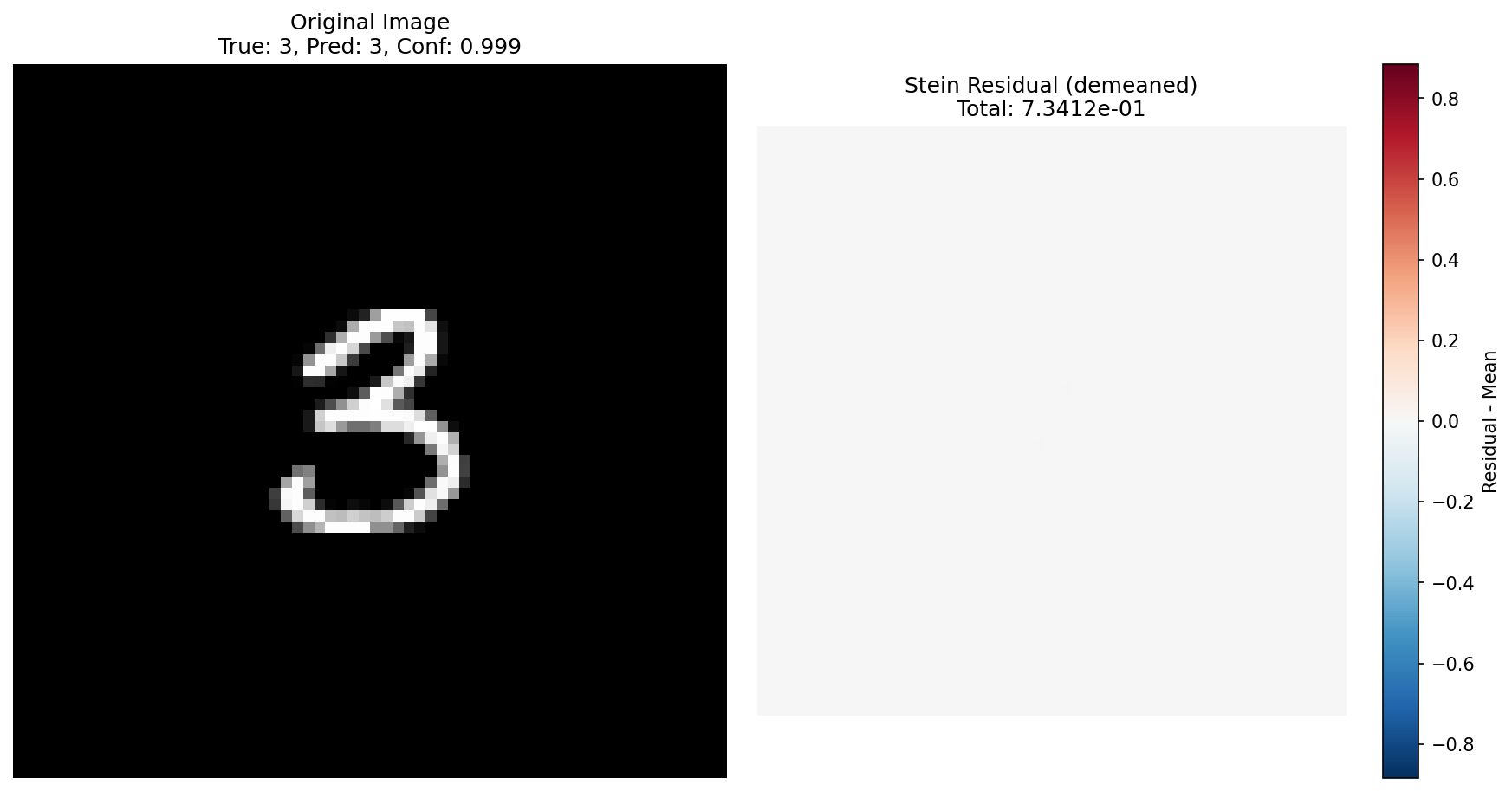}\\[4pt]
    \includegraphics[width=0.55\textwidth]{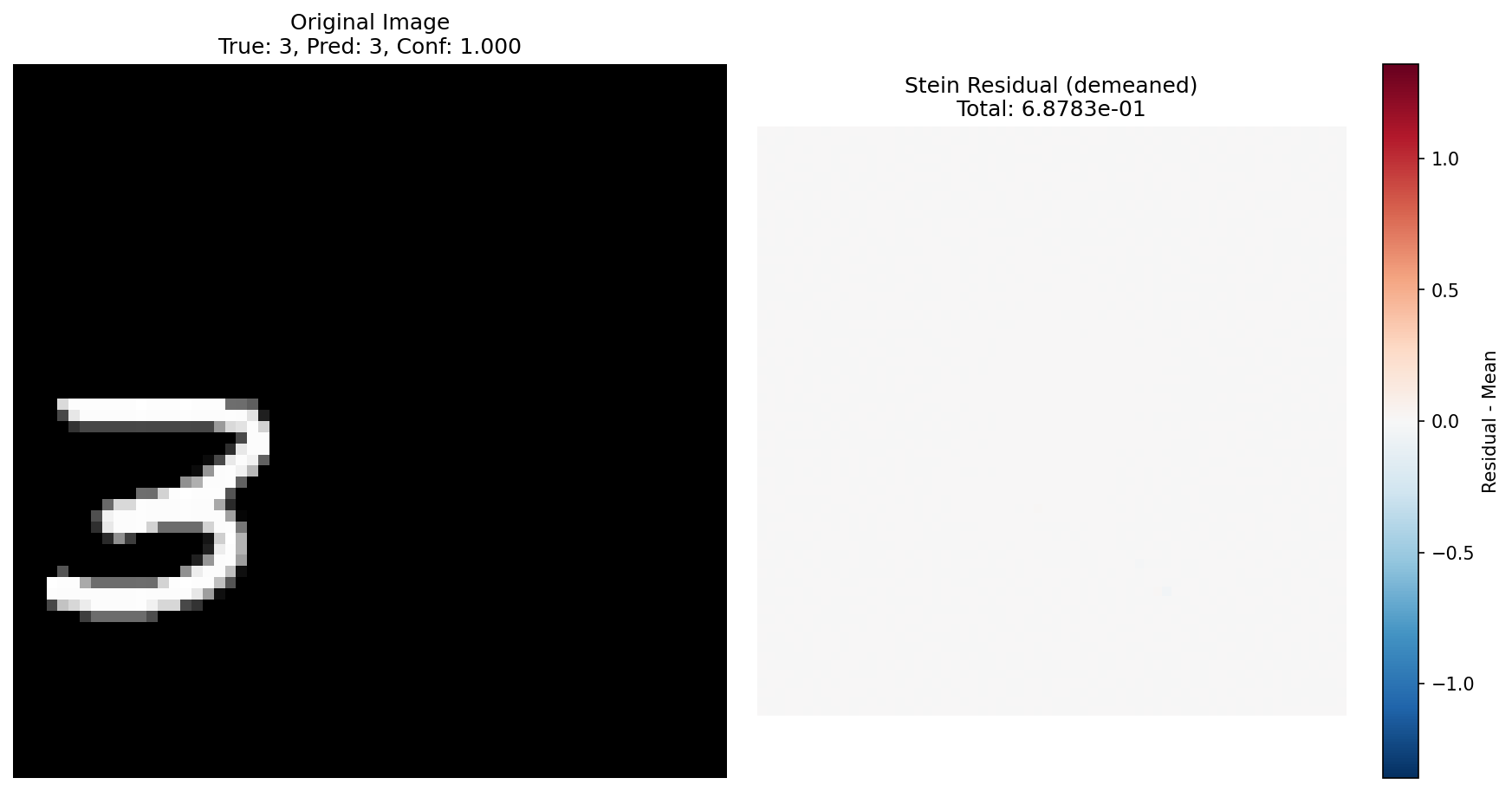}\\[4pt]
    \includegraphics[width=0.55\textwidth]{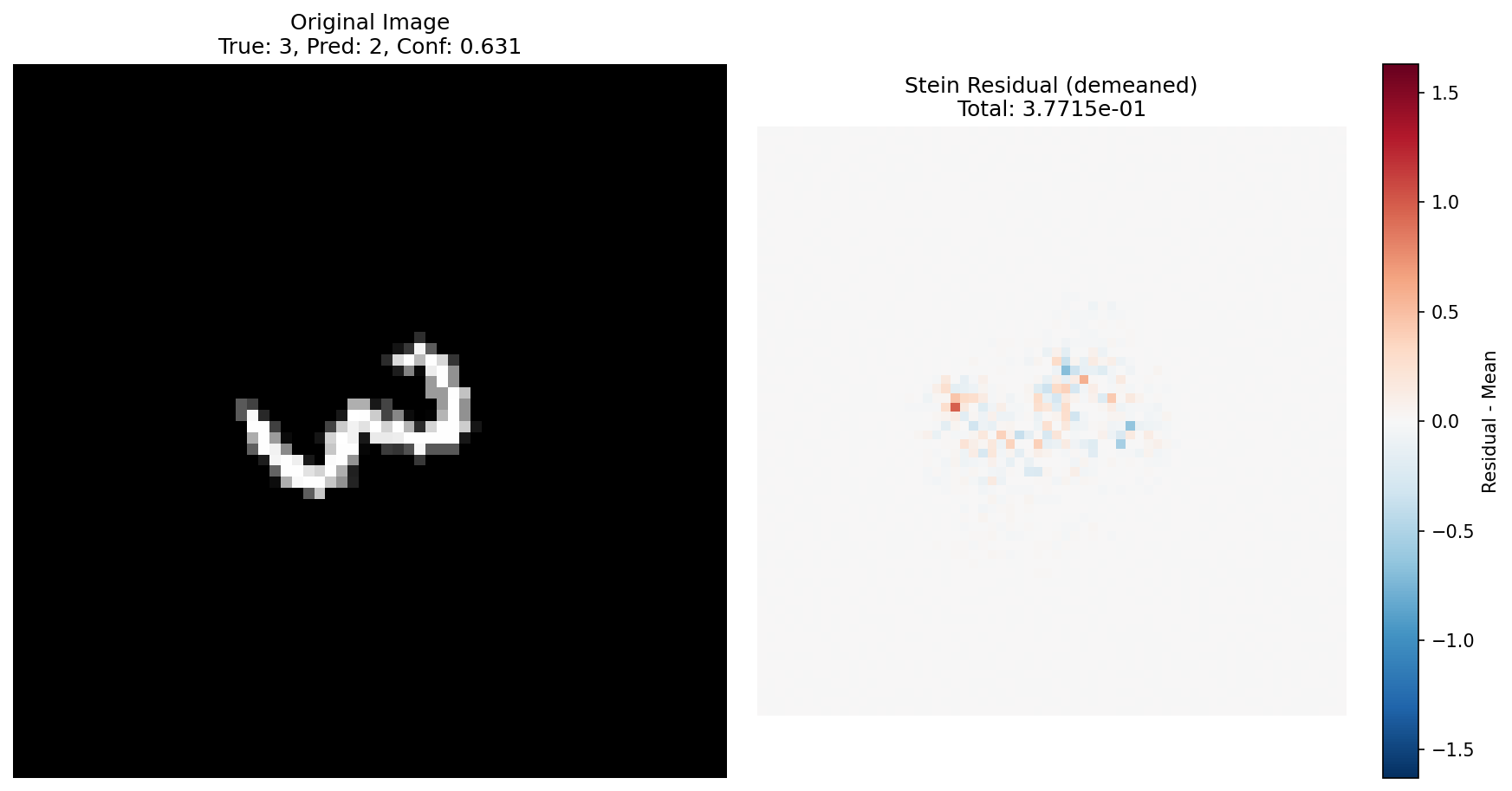}\\[4pt]
    \includegraphics[width=0.55\textwidth]{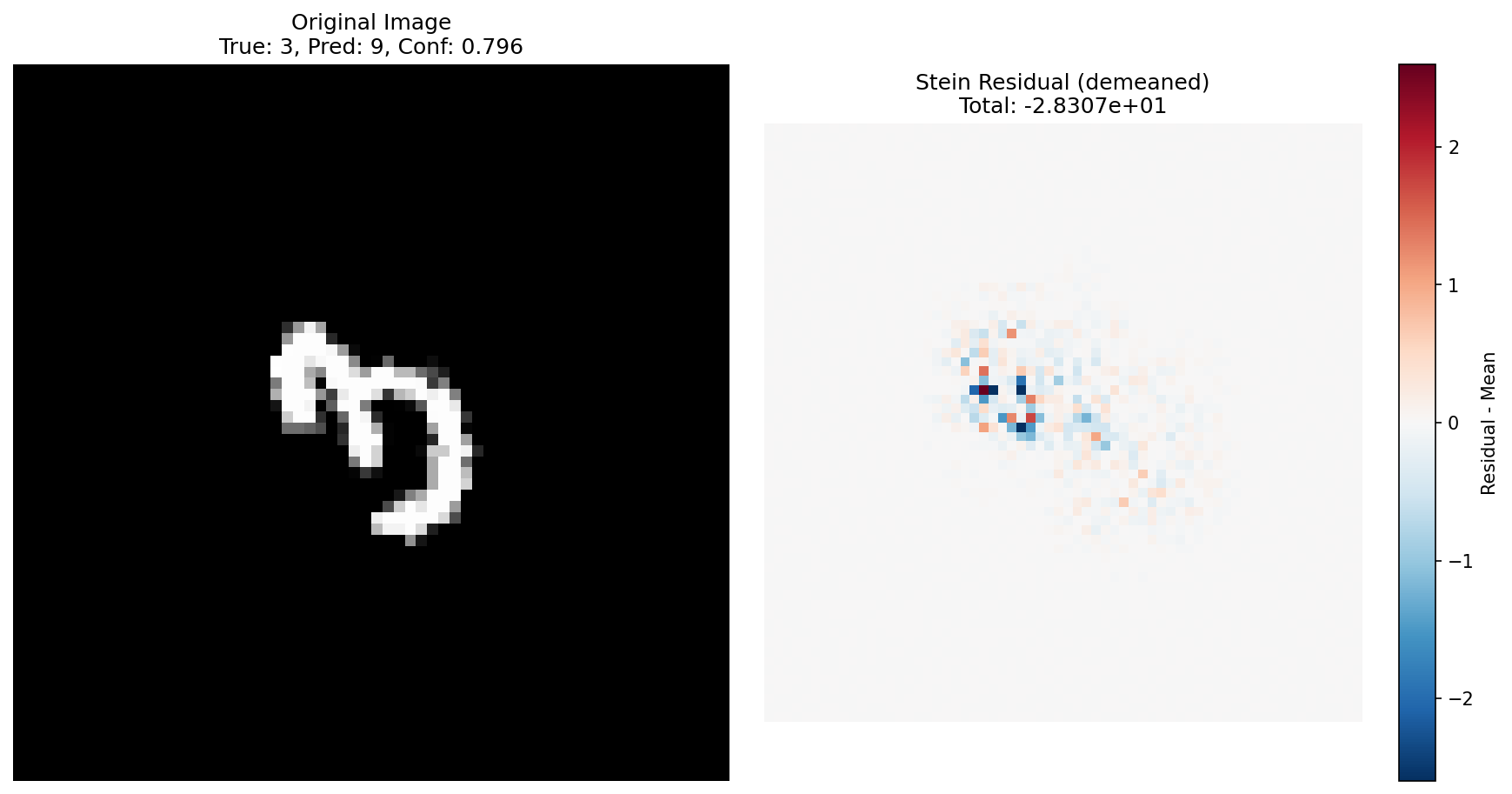}
    \caption{\textbf{Demonstration of per-pixel anomaly heatmaps for the MNIST-based task.}
    Input digits are shown on the left, while the heatmaps are on the right. 
    Note that both original and translated digits do not generate significant anomaly signal.}
    \label{fig:four-stack}
\end{figure}

\begin{figure}[H]
  \centering

  \begin{subfigure}{0.498\linewidth}
    \centering
    \includegraphics[width=\linewidth,height=0.30\textheight,keepaspectratio]{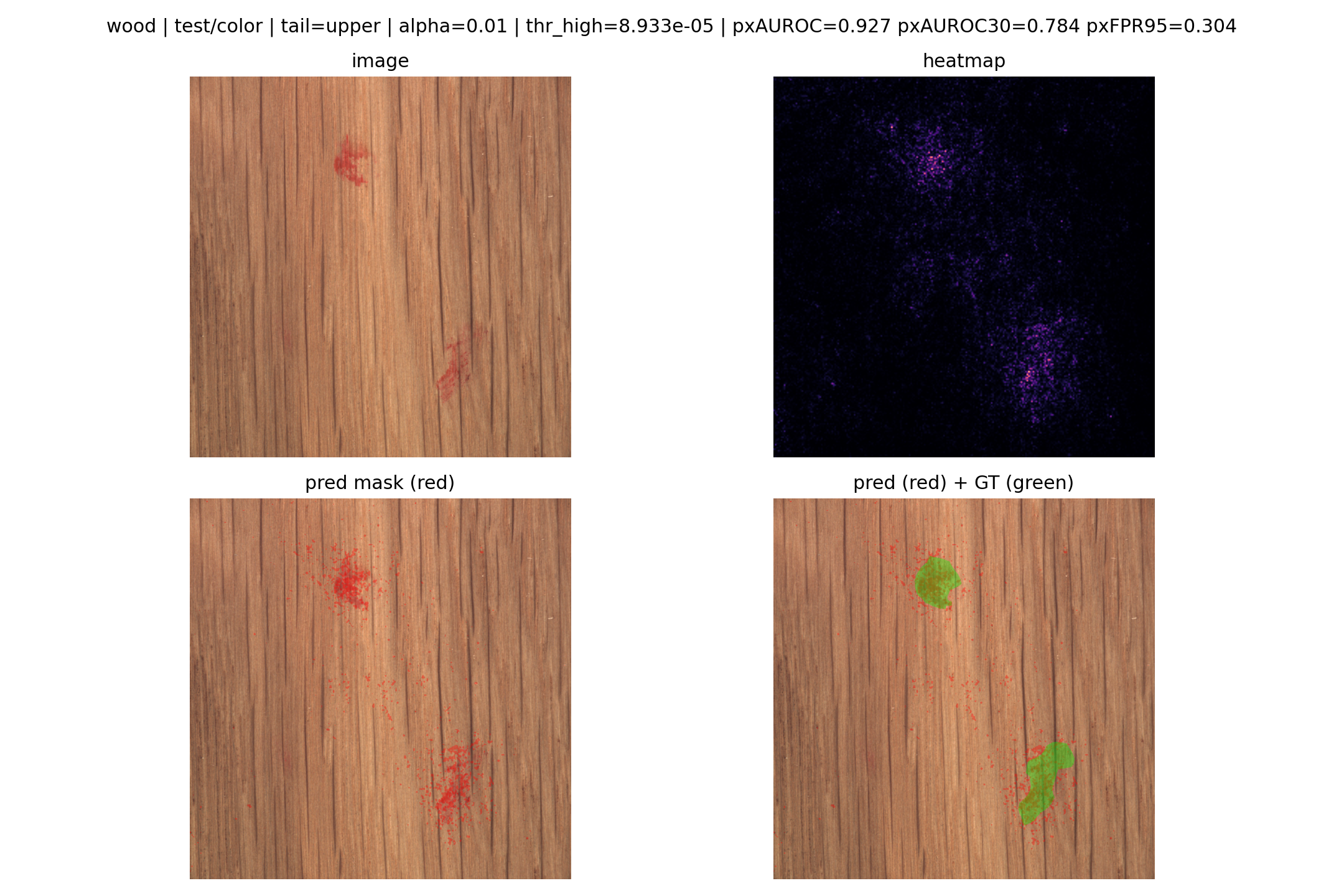}
  \end{subfigure}\hspace{0.004\linewidth}%
  \begin{subfigure}{0.498\linewidth}
    \centering
    \includegraphics[width=\linewidth,height=0.30\textheight,keepaspectratio]{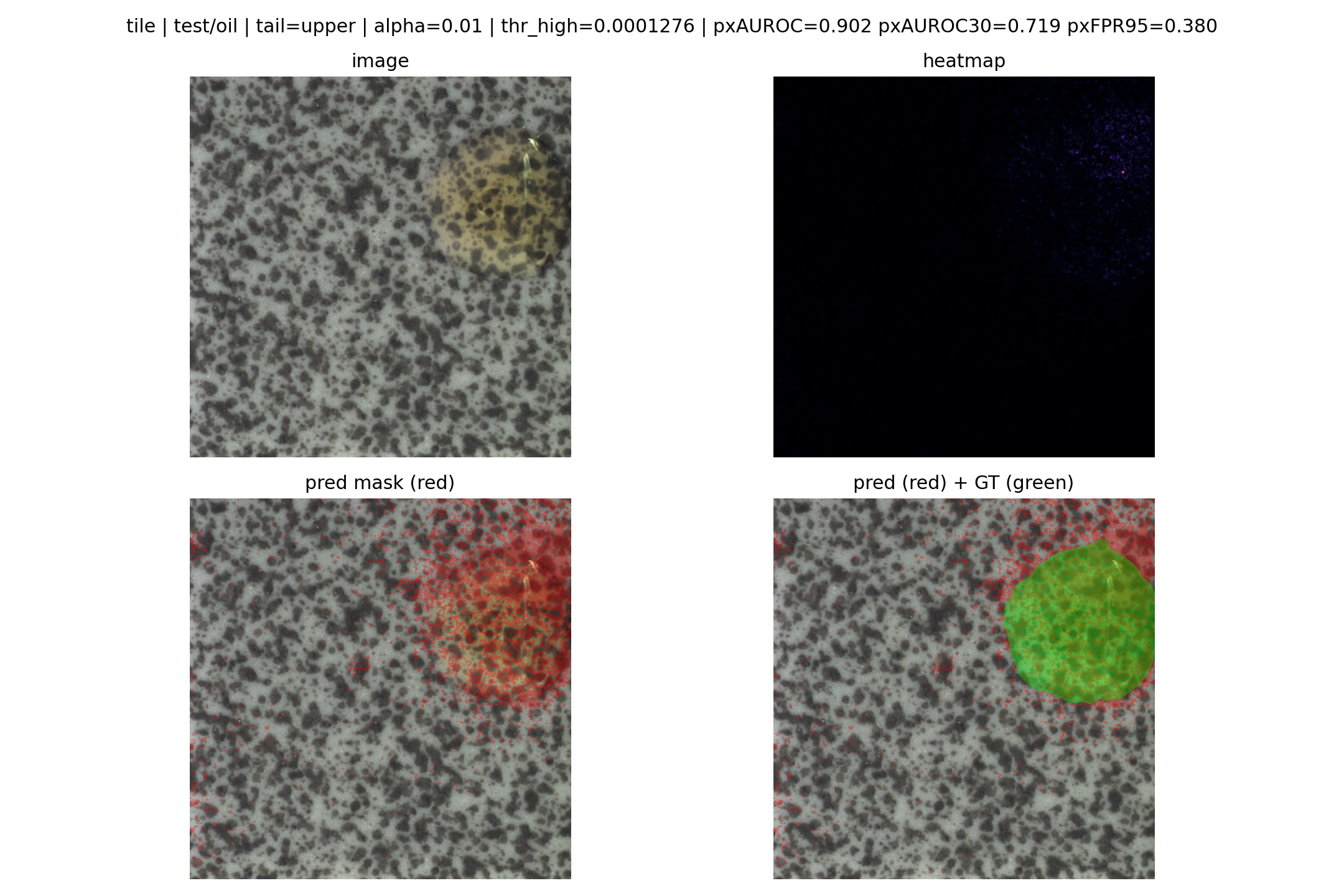}
  \end{subfigure}

  \vspace{0.12em}

  \begin{subfigure}{0.498\linewidth}
    \centering
    \includegraphics[width=\linewidth,height=0.30\textheight,keepaspectratio]{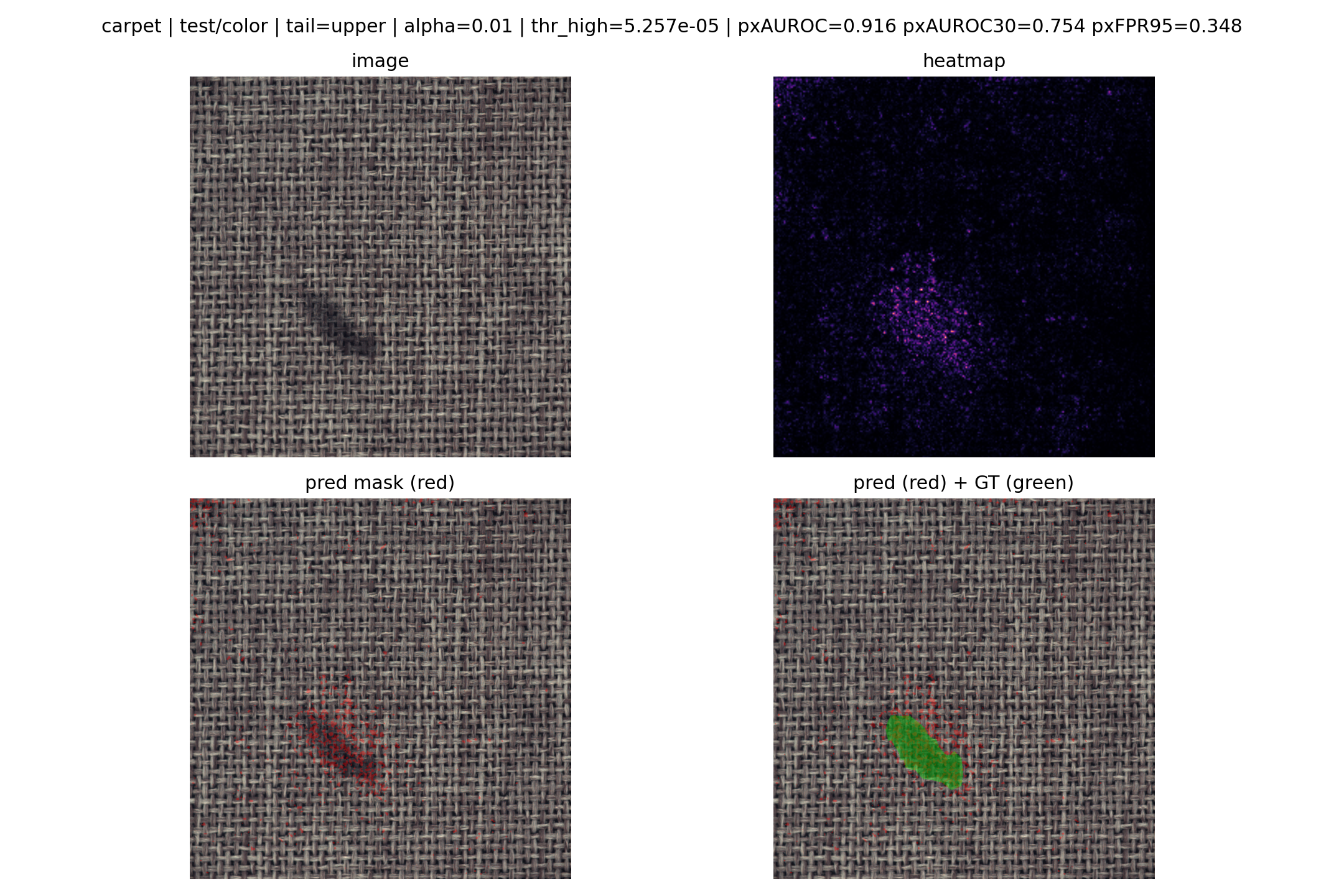}
  \end{subfigure}\hspace{0.004\linewidth}%
  \begin{subfigure}{0.498\linewidth}
    \centering
    \includegraphics[width=\linewidth,height=0.30\textheight,keepaspectratio]{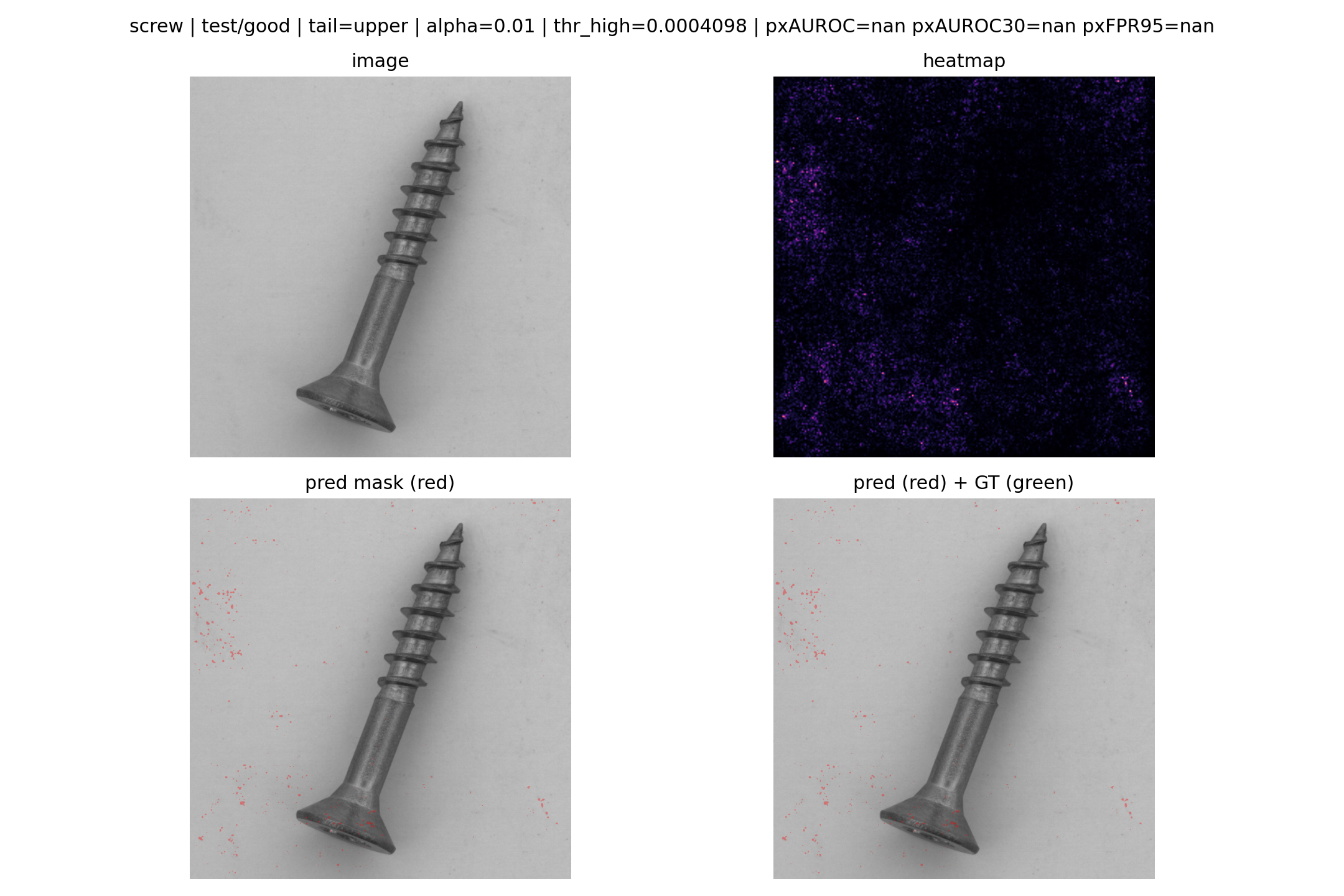}
  \end{subfigure}

  \vspace{0.12em}

  \begin{subfigure}{0.498\linewidth}
    \centering
    \includegraphics[width=\linewidth,height=0.30\textheight,keepaspectratio]{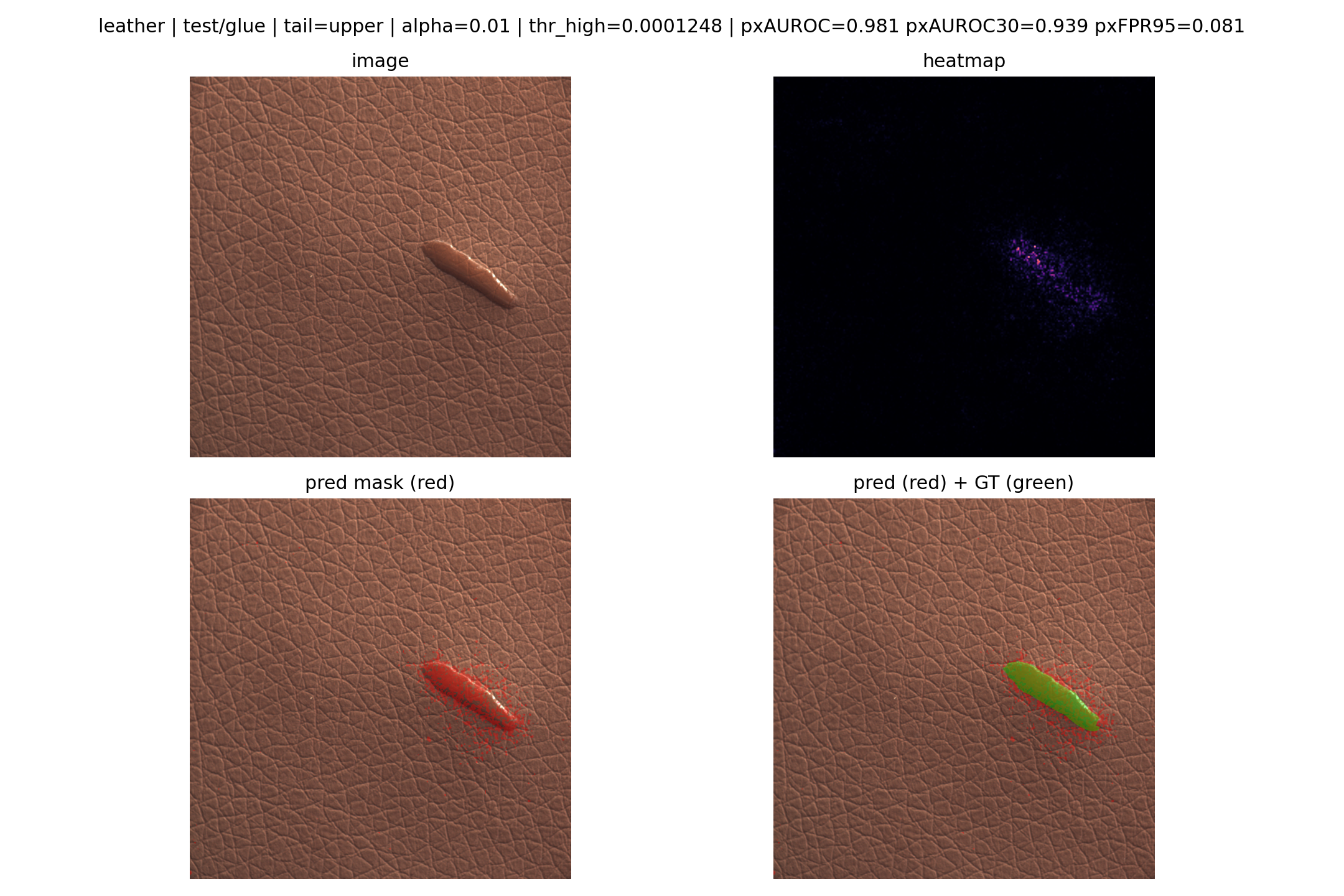}
  \end{subfigure}\hspace{0.004\linewidth}%
  \begin{subfigure}{0.498\linewidth}
    \centering
    \includegraphics[width=\linewidth,height=0.30\textheight,keepaspectratio]{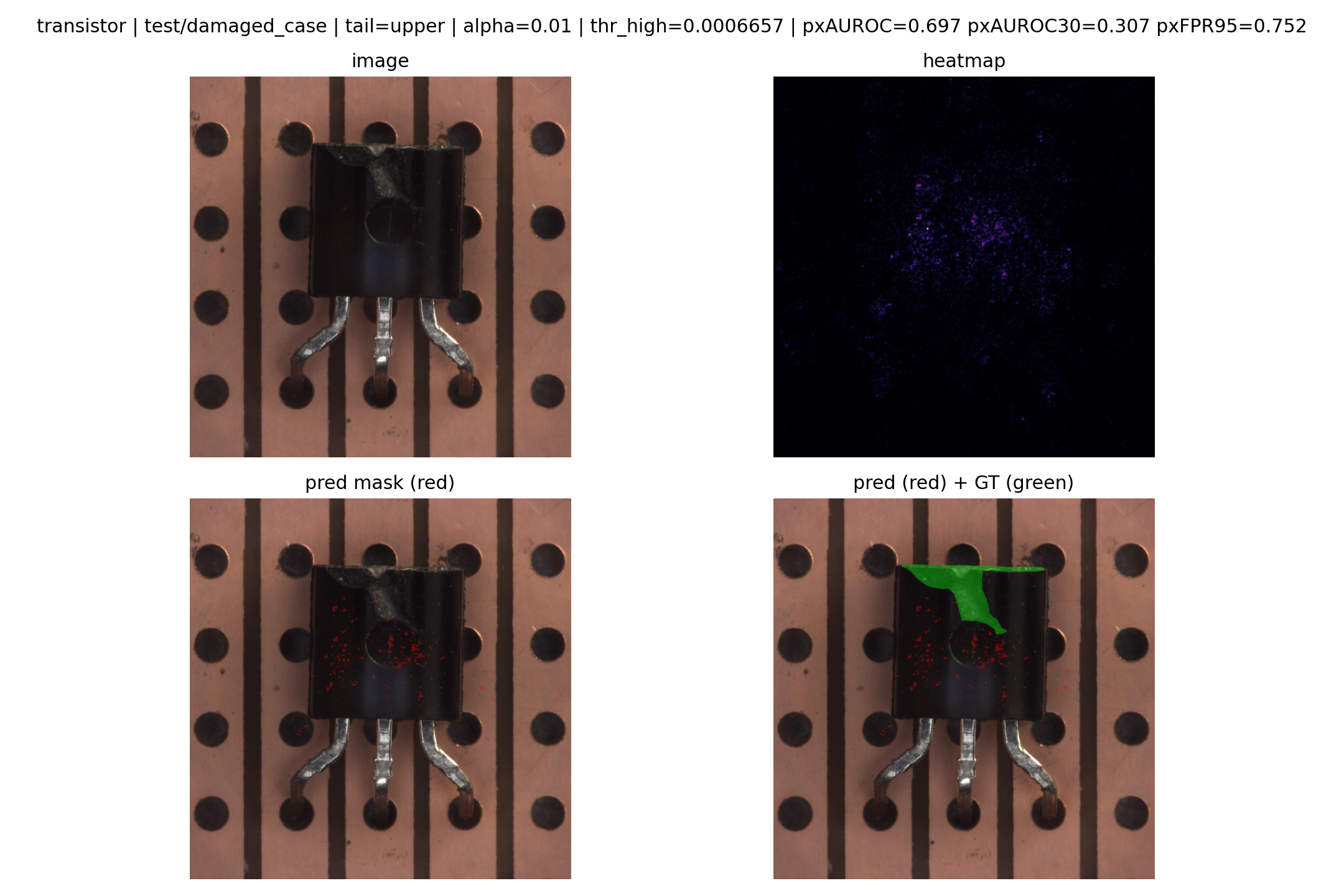}
  \end{subfigure}

  \caption{
    \textbf{Additional qualitative results.}
    Each panel is a randomly chosen example for one of the MVTec AD categories. Within each panel there are four subplots: the original image (with corruption), detector's heatmap, prediction overlay at $\alpha = 0.01$, prediction and ground truth (GT) overlay.
  }
  \label{fig:appendix-3x2}
\end{figure}

\begin{figure}[H]
  \centering

  \begin{subfigure}{0.498\linewidth}
    \centering
    \includegraphics[width=\linewidth,height=0.30\textheight,keepaspectratio]{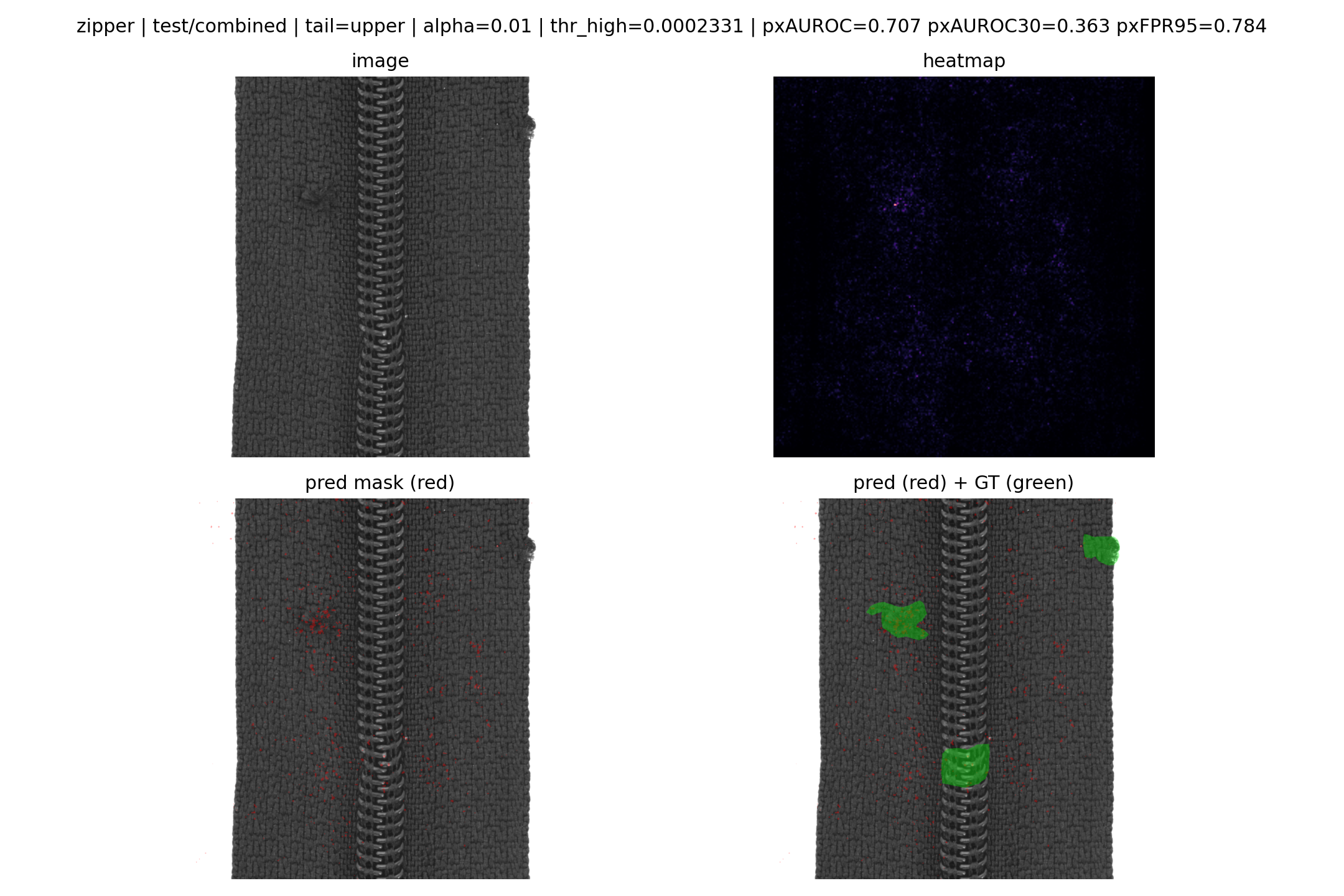}
  \end{subfigure}\hspace{0.004\linewidth}%
  \begin{subfigure}{0.498\linewidth}
    \centering
    \includegraphics[width=\linewidth,height=0.30\textheight,keepaspectratio]{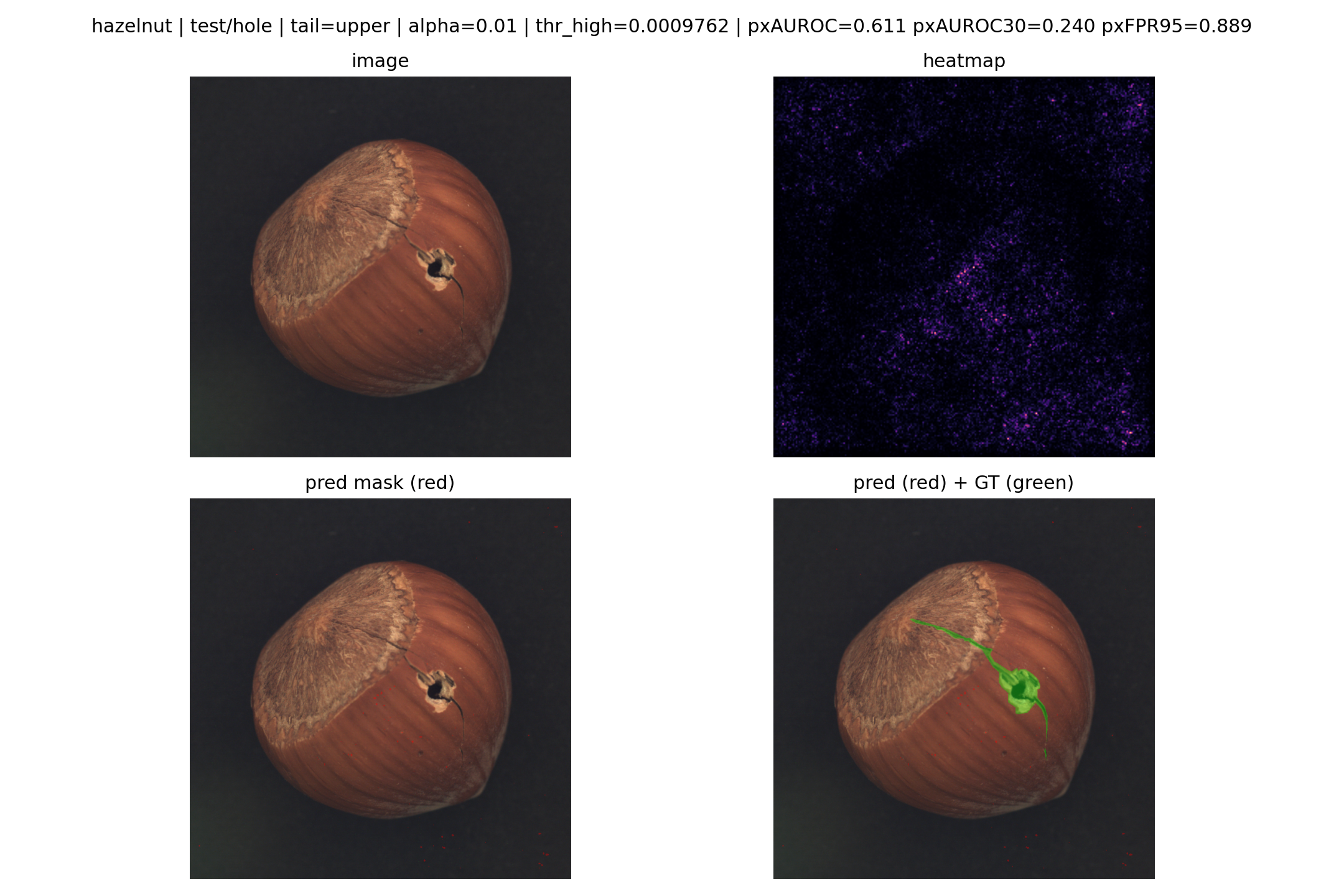}
  \end{subfigure}

  \vspace{0.12em}

  \begin{subfigure}{0.498\linewidth}
    \centering
    \includegraphics[width=\linewidth,height=0.30\textheight,keepaspectratio]{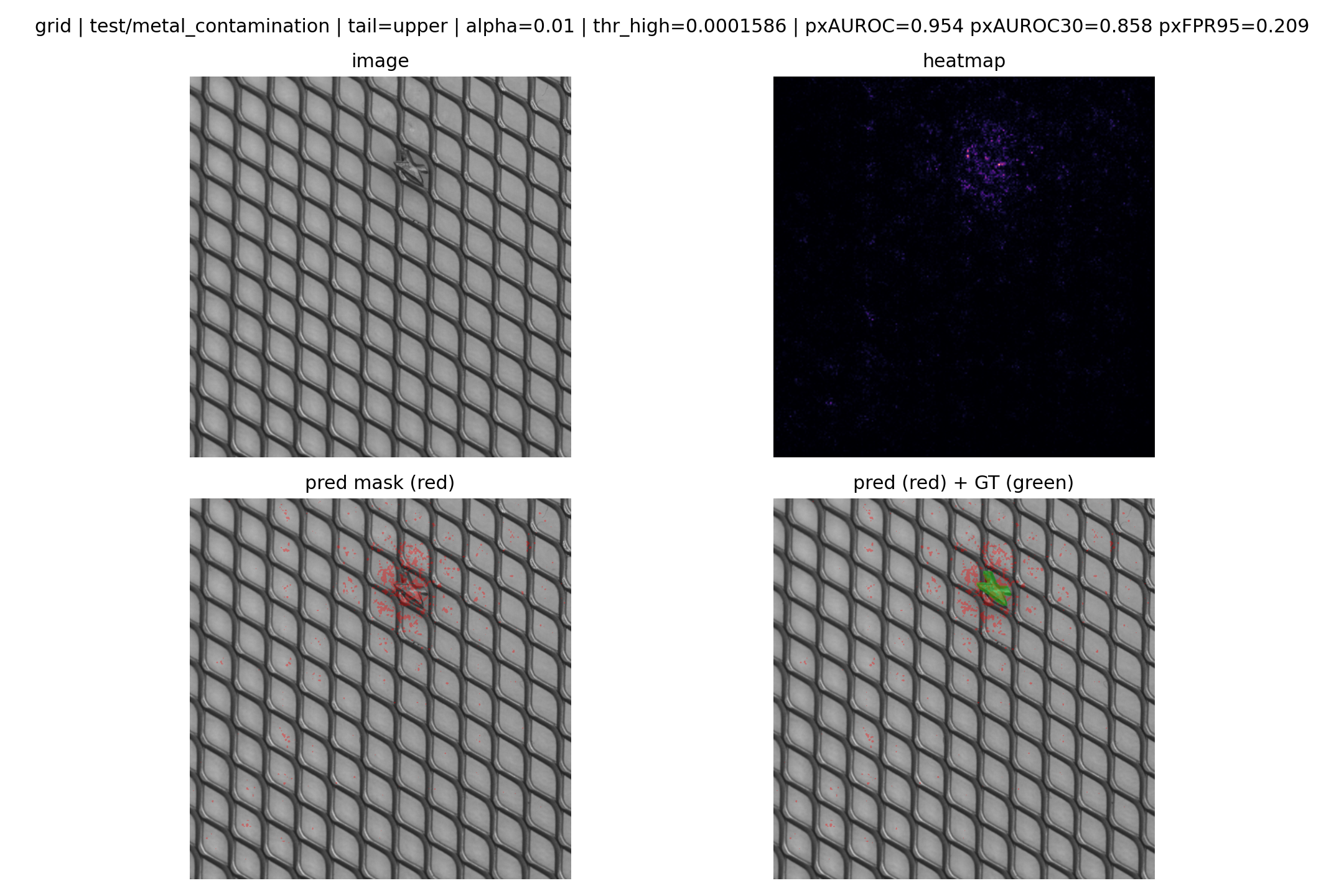}
  \end{subfigure}\hspace{0.004\linewidth}%
  \begin{subfigure}{0.498\linewidth}
    \centering
    \includegraphics[width=\linewidth,height=0.30\textheight,keepaspectratio]{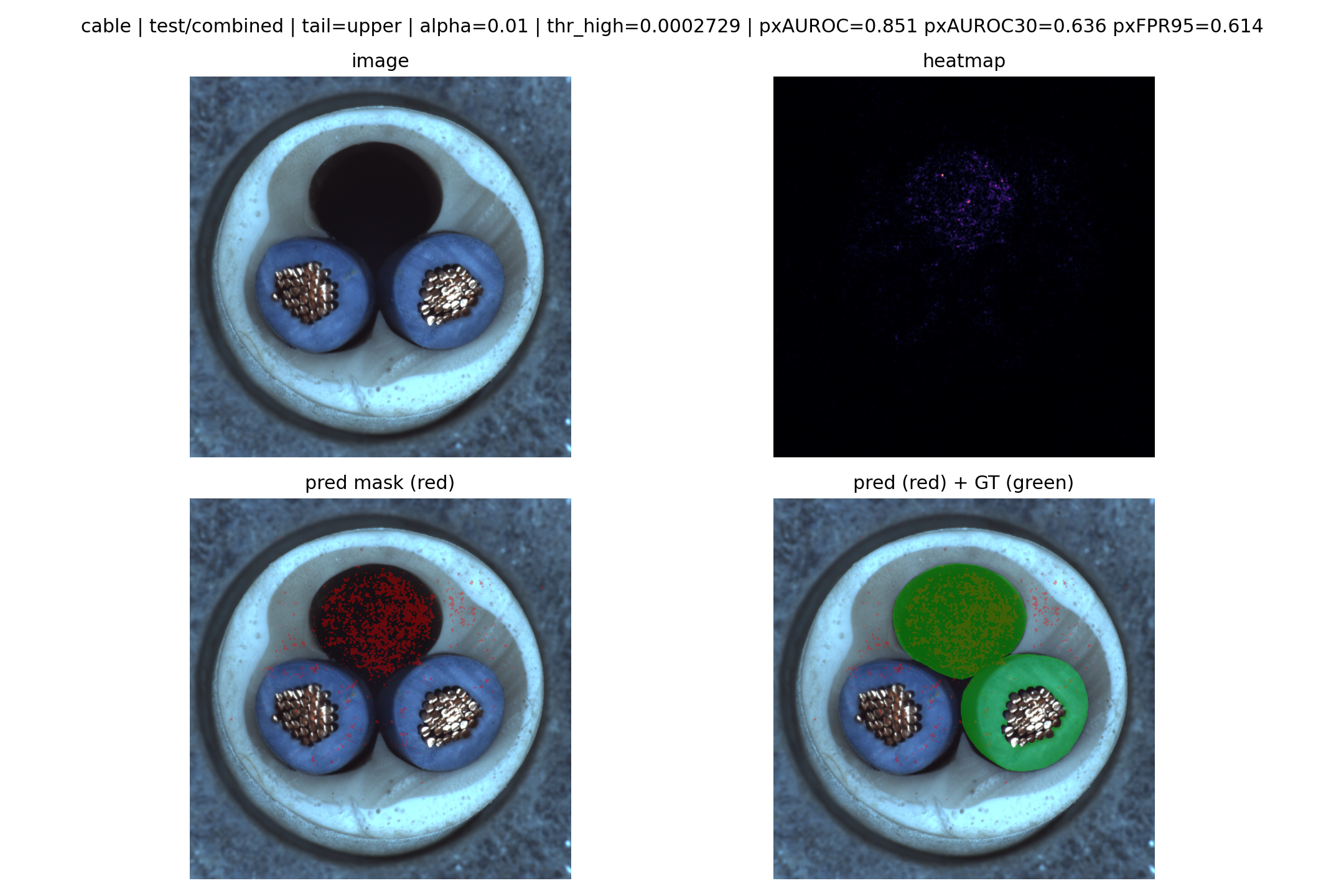}
  \end{subfigure}

  \vspace{0.12em}

  \begin{subfigure}{0.498\linewidth}
    \centering
    \includegraphics[width=\linewidth,height=0.30\textheight,keepaspectratio]{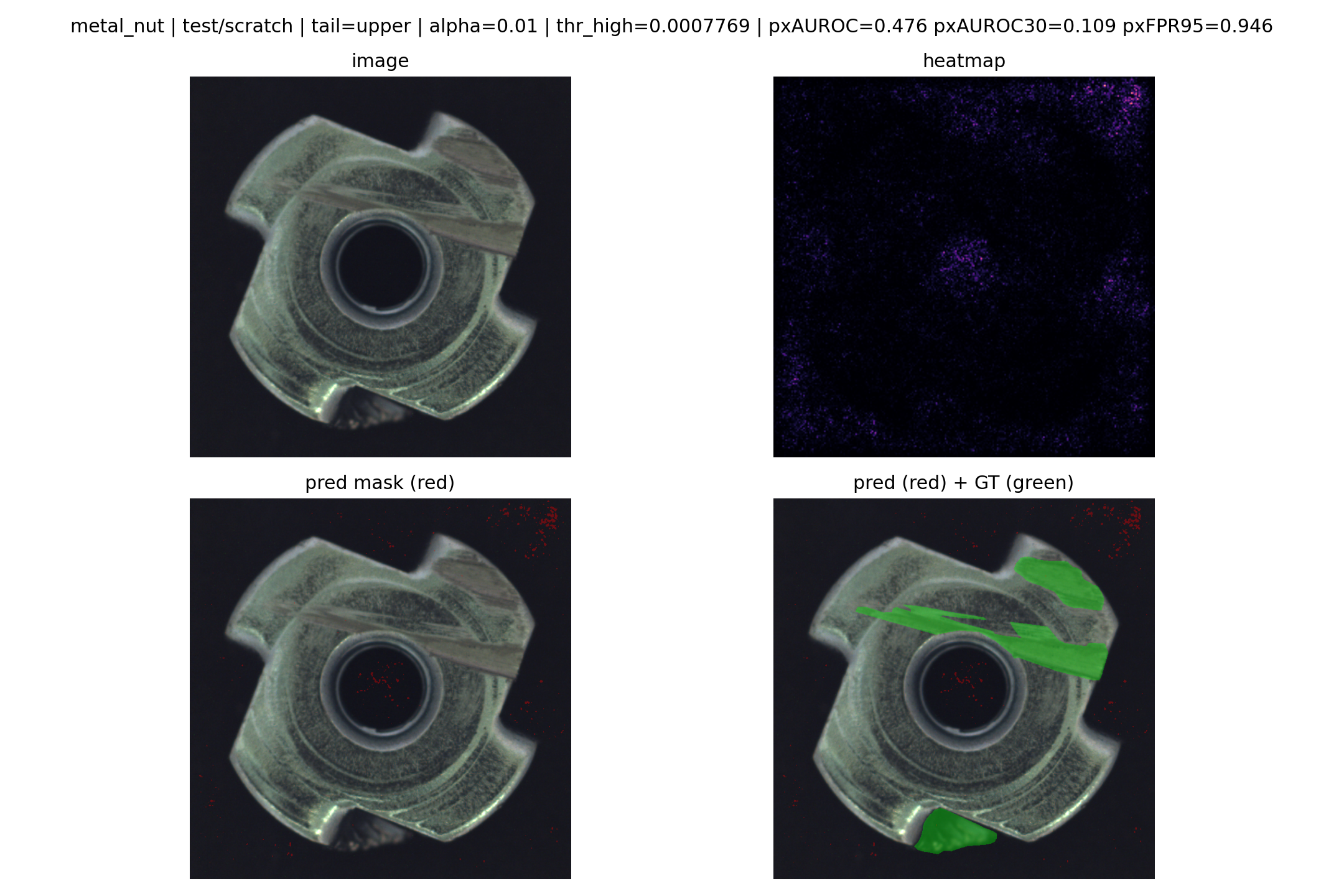}
  \end{subfigure}\hspace{0.004\linewidth}%
  \begin{subfigure}{0.498\linewidth}
    \centering
    \includegraphics[width=\linewidth,height=0.30\textheight,keepaspectratio]{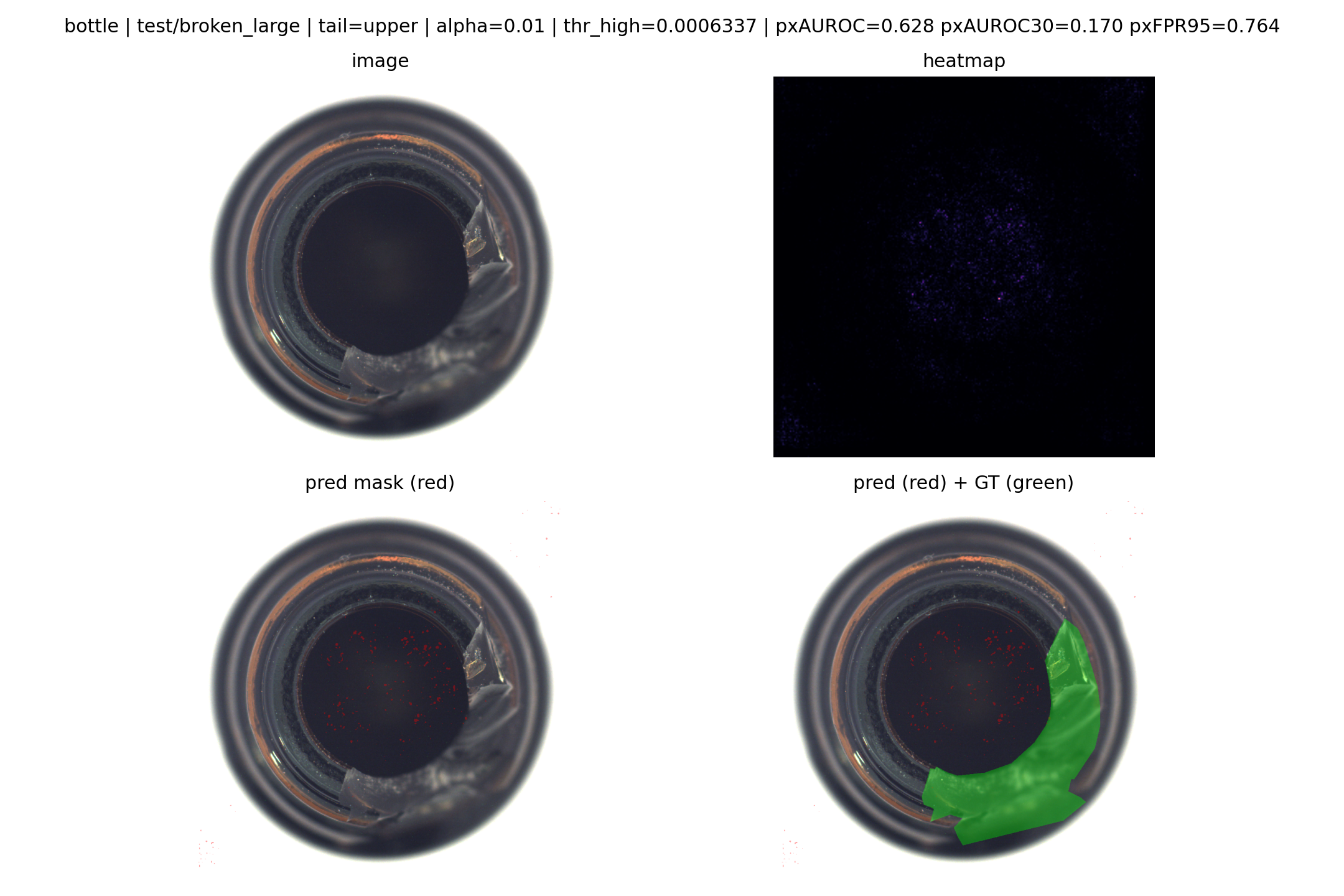}
  \end{subfigure}

  \caption{
    \textbf{Additional qualitative results.}
    Each panel is a randomly chosen example for one of the MVTec AD categories. Within each panel there are four subplots: the original image (with corruption), detector's heatmap, prediction overlay at $\alpha = 0.01$, prediction and ground truth (GT) overlay.}
  \label{fig:appendix-3x2B}
\end{figure}

\textbf{Observations.}
Figures \ref{fig:appendix-3x2} and \ref{fig:appendix-3x2B}  show that  
across MVTec AD categories, per-pixel Stein residuals can produce spatially coherent
and visually meaningful localisation signals, particularly on texture classes,
where defects manifest as systematic deviations in local appearance.
Quantitative metrics and qualitative inspection indicate substantially weaker
performance on object categories, where anomalies are often small, highly
specific, and semantically defined (e.g.\ missing components or structural
deformations).
We hypothesise that such object-level defects are less accessible to a
zero-shot pipeline based on a generic ImageNet-pretrained classifier and score
model, which have not been exposed to object-specific normality constraints.
Importantly, no component of our pipeline—neither the classifier nor the score
model—has been trained or adapted on MVTec AD data, making this evaluation
fully \emph{zero-shot}.

\section{Ablation study}
\label{app:ablations}

Next we present an ablation study of individual operator components and a sensitivity analysis of the diffusion time step in the score model. 
\gr{Most ablation experiments are carried out} on the full CIFAR-10 OOD
benchmark suite used in the main results, as it probes detection performance
across heterogeneous shift regimes.

\subsection{Sensitivity to score quality.}
TASTE relies on an estimate of the score function $\nabla \log p(x)$, raising the question of robustness to imperfect score models.
To evaluate this, we degrade the quality of the score estimate by applying a low resolution diffusion model to downscaled input and then upscaling to the original size.

Table~\ref{tab:score-mismatch-effect} shows that this mismatch leads to a consistent drop in AUROC and an increase in FPR95 across datasets.
However, the degradation is moderate: the low-resolution variant remains competitive while providing an approximately $4\times$ speedup. Resolution mismatch reduces average performance in this
experiment, although individual metrics need not degrade. This suggests that TASTE does not require highly precise score estimates to remain effective, and can trade accuracy for efficiency in practical deployments.

\subsection{Ablation: Operator decomposition}
\label{app:ablation-operator}

We study the effect of progressively simplifying the Langevin Stein operator
by removing individual components.
Specifically, we compare:
\begin{itemize}[leftmargin=*]
\item \textbf{Full operator (TASTE):}
\[
\mathcal{L}_p f(x) = \Delta f(x) + s_p(x)^\top \nabla f(x),
\]
\item \textbf{No-Laplacian:}
\[
\mathcal{L}_p^{\text{no-lap}} f(x) = s_p(x)^\top \nabla f(x),
\]
\item \textbf{Score-only:}
using only the score $L^2$ norm, without interaction with
$\nabla f$.
\end{itemize}



Table~\ref{tab:c10-lap-ablation} reports performance across adversarial attacks,
corruption benchmarks (CIFAR-10-C), structured perturbations (CIFAR-10-P), and
 standard OOD datasets. Across all benchmark categories, the full operator and the no-Laplacian variant exhibit nearly identical performance, while the score-only variant performs
 substantially worse.
 In particular, the AUROC and FPR95 of \textsc{TASTE} and
 \textsc{TASTE-no-Laplacian} differ only marginally, whereas additionally removing the
 $\nabla f$ interaction leads to a pronounced drop in performance. 
 This behaviour is explained by 
 \gdr{the observation that in this experiment} 
the full score is dominated by
 \[
 |s(x)^\top \nabla f_c(x)|
 \;\approx\;
 \|s(x)\|\,\|\nabla f_c(x)\|\,|\cos\theta_c|,
 \]
 where $\theta_c$ is the angle between the score direction and the class-specific
gradient. Empirically, we observe that while the gradient and the Laplacian are of comparable magnitude, the output of the score model is often much larger in scale, causing the dot-product term of the operator to dominate, even if the correlation between the score and the gradient is small.
\paragraph{Role of the Laplacian term.}
Under the stated regularity assumptions, the Laplacian
provides the curvature correction required for the
Langevin Stein identity. Its theoretical role does not,
however, imply improved detection performance in every
setting.

For a fixed probe $f$, let
$\widetilde S_f^{\mathrm{full}}$ and
$\widetilde S_f^{(1)}$ denote the population-centred
functionals with and without the Laplacian, using the
same score estimate. Then
\[
\widetilde S_f^{\mathrm{full}}(p,q)
-\widetilde S_f^{(1)}(p,q)
=(\mathbb{E}_q-\mathbb{E}_p)[\Delta f].
\]
Thus, omitting the Laplacian changes the dataset-level
response according to the shift in expected probe
curvature. This identity characterises the dataset-level
response; the effect on per-example detection is
assessed empirically.

Our CIFAR-10 ablation
(Table~\ref{tab:c10-lap-ablation}) shows little detection
benefit from the Laplacian under the evaluated
implementation, supporting the cheaper first-order
variant in this setting.
The MNIST rotation experiment
(Figure~\ref{fig:mnist-ablation}) demonstrates that the
term can matter, but does not establish a general
advantage on larger benchmarks.
Its practical contribution depends on the probe, shift,
and accuracy of curvature estimation.


\begin{table}[t]
  \centering
  \caption{Effect of score model degradation on OOD detection performance on ImageNet-1K. Pre-softmax logit of the predicted class is used as the probe.}
  \label{tab:score-mismatch-effect}
  \resizebox{\linewidth}{!}{%
  \begin{tabular}{l
                  cc  
                  cc  
                  cc  
                  cc  
                  cc  
                  cc  
                  }
    \toprule
    & \multicolumn{2}{c}{\textbf{iSUN}}
    & \multicolumn{2}{c}{\textbf{LSUN}}
    & \multicolumn{2}{c}{\textbf{Places365}}
    & \multicolumn{2}{c}{\textbf{SVHN}}
    & \multicolumn{2}{c}{\textbf{Textures}}
    & \multicolumn{2}{c}{\textbf{Average}} \\
    
    \cmidrule(lr){2-3} \cmidrule(lr){4-5}
    \cmidrule(lr){6-7} \cmidrule(lr){8-9}
    \cmidrule(lr){10-11} \cmidrule(lr){12-13}
    
    Method 
    & AUROC$\uparrow$ & FPR95$\downarrow$
    & AUROC$\uparrow$ & FPR95$\downarrow$
    & AUROC$\uparrow$ & FPR95$\downarrow$
    & AUROC$\uparrow$ & FPR95$\downarrow$
    & AUROC$\uparrow$ & FPR95$\downarrow$
    & AUROC$\uparrow$ & FPR95$\downarrow$ \\
    
    \midrule
    TASTE 
    & 0.7603 & 0.9720
    & 0.7820 & 0.9380
    & 0.5927 & 0.9410
    & 0.9477 & 0.1200
    & 0.6396 & 0.8890
    & \textbf{0.7445} & \textbf{0.7720} \\

    TASTE (low res)
    & 0.7120 & 0.9460
    & 0.7044 & 0.9480
    & 0.5330 & 0.9510
    & 0.9206 & 0.2050
    & 0.5806 & 0.9490
    & \textbf{0.6901} & \textbf{0.8190} \\

    \bottomrule
  \end{tabular}%
  }
\vspace{2pt}
\end{table}

\begin{table*}[t]
  \centering
  \caption{Operator construction ablation study - performance on CIFAR-10 across heterogeneous shift regimes. 
Results are averaged within each category. The dataset composition is identical to that used in the CIFAR-10 evaluation suite presented in Section~\ref{sec:benchmark-evaluation}.
}
  \label{tab:c10-lap-ablation}
  \resizebox{\textwidth}{!}{%
    \begin{tabular}{l
                    cc  
                    cc  
                    cc  
                    cc
                    cc
                    }
      \toprule
      & \multicolumn{2}{c}{\textbf{Adversarial}} 
      & \multicolumn{2}{c}{\textbf{CIFAR-10-C}} 
      & \multicolumn{2}{c}{\textbf{CIFAR-10-P}} 
      & \multicolumn{2}{c}{\textbf{OOD-benchmarks}} 
      & \multicolumn{2}{c}{\textbf{Overall}} \\
      \cmidrule(lr){2-3} \cmidrule(lr){4-5} \cmidrule(lr){6-7} \cmidrule(lr){8-9}\cmidrule(lr){10-11}
      Method & AUROC$\uparrow$ & FPR95$\downarrow$ & AUROC$\uparrow$ & FPR95$\downarrow$ & AUROC$\uparrow$ & FPR95$\downarrow$ & AUROC$\uparrow$ & FPR95$\downarrow$ & AUROC$\uparrow$ & FPR95$\downarrow$\\
      \midrule
      TASTE  & 0.6144 & 0.865 & 0.6193 & 0.8691 & 0.5954 & 0.8611 & 0.7647 & 0.5677 & \textbf{0.6301} & \textbf{0.8207} \\
      TASTE-no-Laplacian & 0.6264 & 0.8711 & 0.6098 & 0.8708 & 0.5902 & 0.8657 & 0.7597 & 0.5751 & 0.6253 & 0.8367 \\
      Score-only & 0.5337 & 0.9437 & 0.6364 & 0.787 & 0.5789 & 0.8756 & 0.4452 & 0.8623 & 0.5783 & 0.8518 \\
      \bottomrule
    \end{tabular}%
  }
  \vspace{1pt}
\end{table*}

\begin{figure}[H]
  \centering
\includegraphics[width=\linewidth]{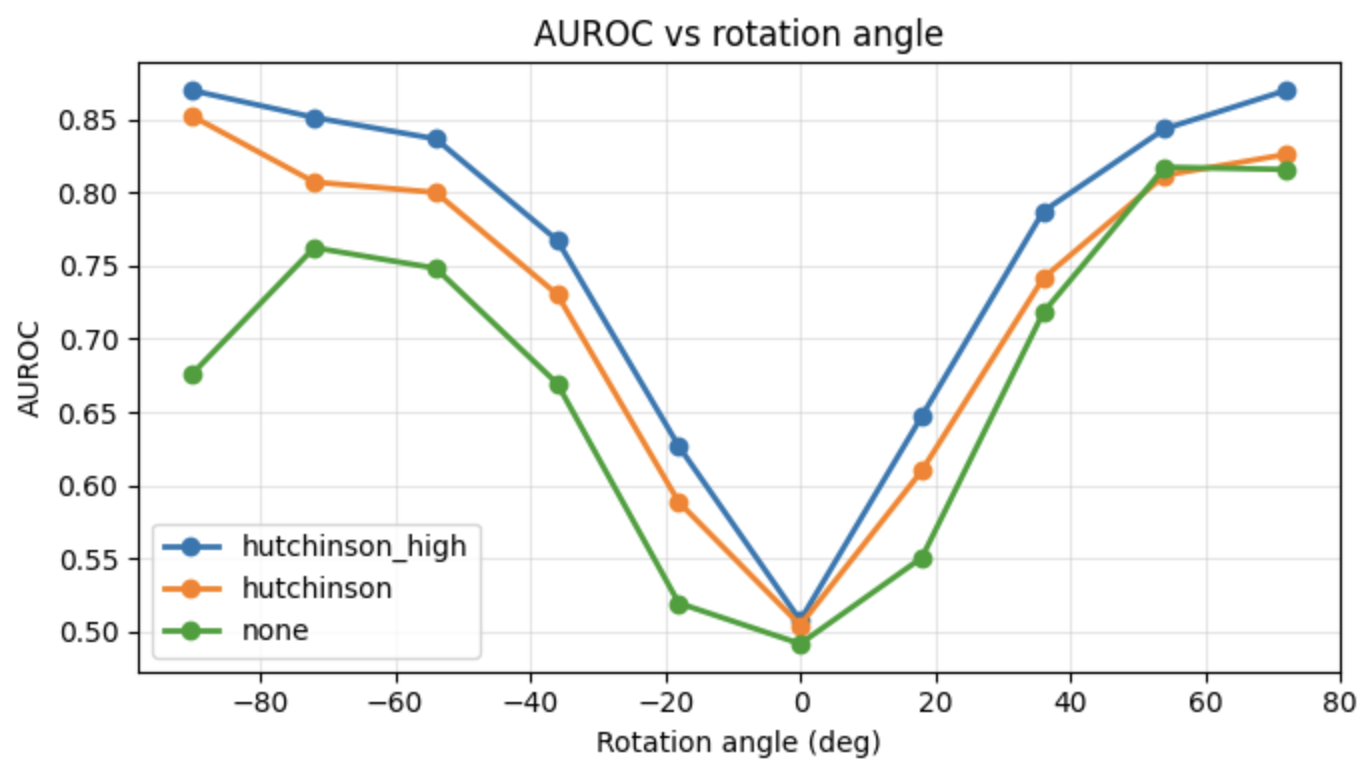}
  \caption{\textbf{The effect of Laplacian estimation on performance of TASTE.} Evaluation on MNIST under rotation-induced shifts using an AlexNet-style CNN shows monotone improvement as the quality of Laplacian estimation increases.}
    \label{fig:mnist-ablation}
\end{figure}

\subsection{Sensitivity to diffusion time step}
\label{sec:diffusion-time step-sensitivity}

This sensitivity analysis is carried out on a per-pixel level, using the MVTec AD dataset from the main part of the paper. 
The per-pixel Stein residual depends on the quality and scale of the score
estimate $\hat{s}_p(x)\approx\nabla_x\log p_t(x)$ produced by the diffusion
model at a chosen time step $t$.
Since the diffusion score interpolates between fine-grained data geometry
(small $t$) and heavily smoothed distributions (large $t$), the time step
implicitly controls the granularity of the resulting anomaly signal.
We therefore study how per-pixel OOD performance varies as a function of $t$.

\textbf{Setup.}
Using the MVTec AD dataset, we evaluate pixel-level anomaly detection metrics
across a range of diffusion time steps while keeping all other components fixed:
the classifier, Stein operator, residual aggregation, and evaluation protocol.
For each time step, we compute per-pixel Stein heatmaps and report
pixel-level AUROC, AUPRO, and average precision (AP), averaged across categories.
\begin{figure}[H]
  \centering

  \begin{subfigure}[t]{0.48\linewidth}
    \centering
    \includegraphics[
      width=\linewidth,
      height=0.22\textheight,
      keepaspectratio
    ]{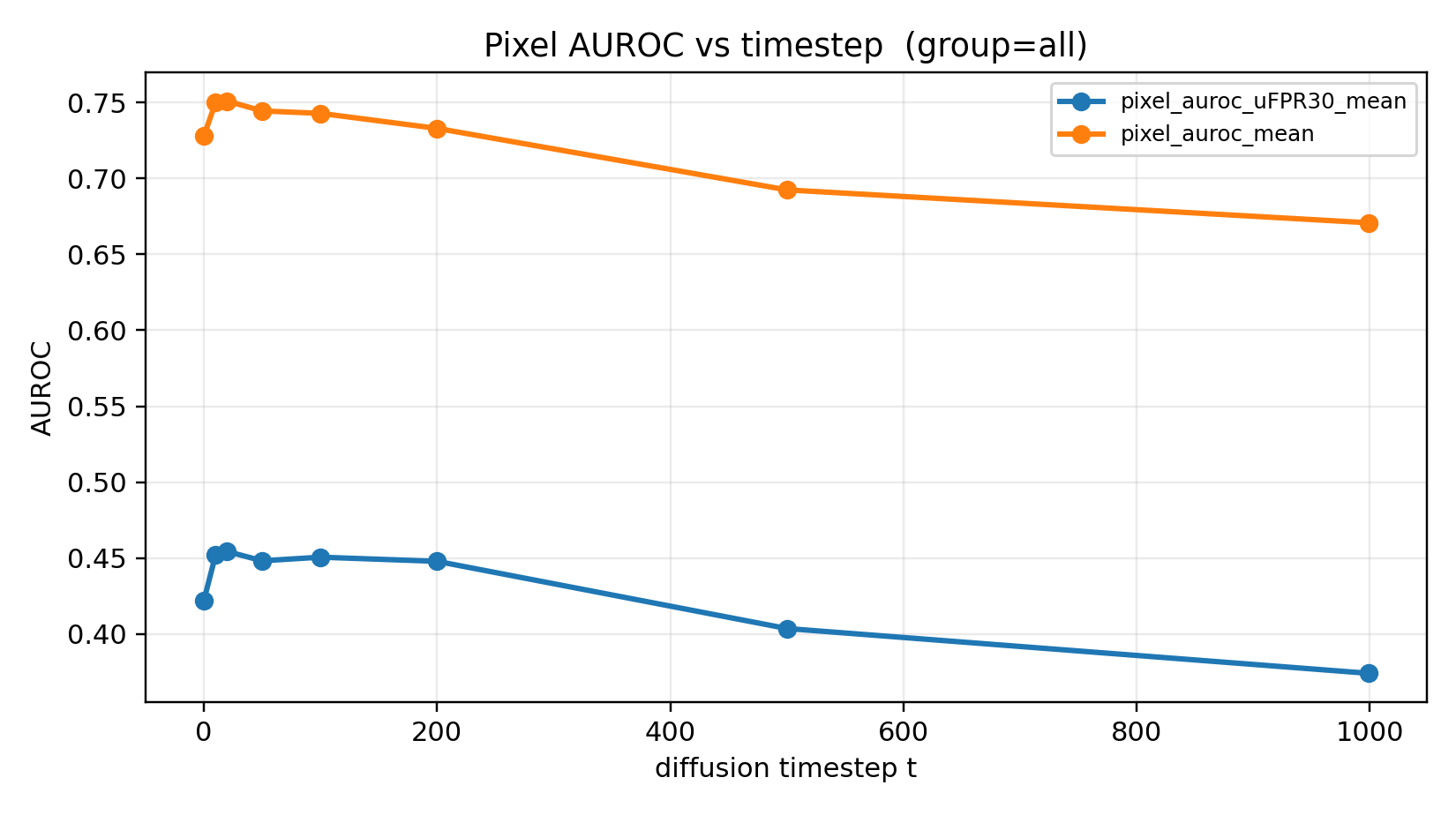}
    \caption{Pixel AUROC}
    \label{fig:pixel-auroc-time step}
  \end{subfigure}
  \hfill
  \begin{subfigure}[t]{0.48\linewidth}
    \centering
    \includegraphics[
      width=\linewidth,
      height=0.22\textheight,
      keepaspectratio
    ]{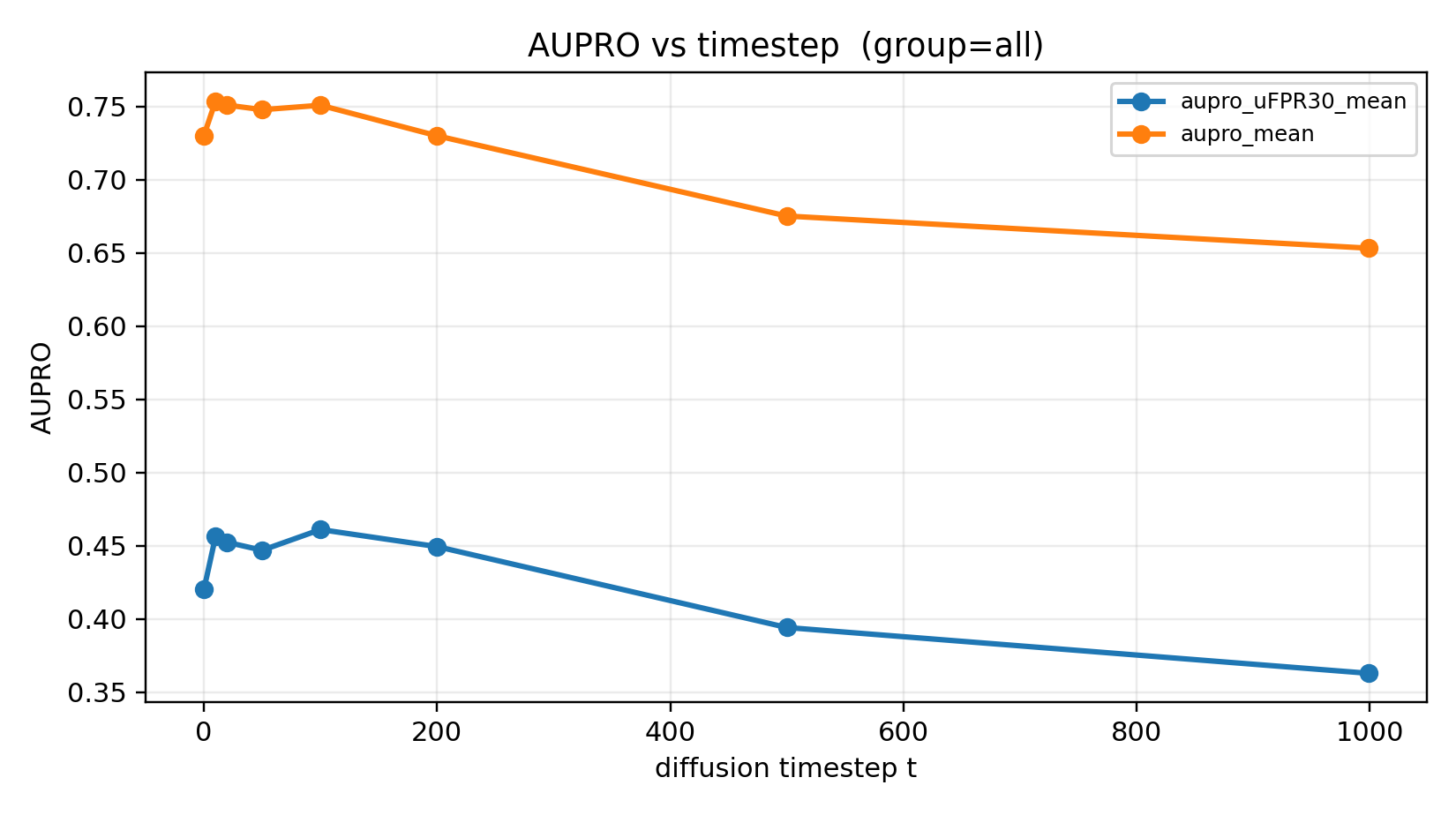}
    \caption{AUPRO}
    \label{fig:pixel-aupro-time step}
  \end{subfigure}

  \vspace{0.6em}

  \begin{subfigure}[t]{\linewidth}
    \centering
    \includegraphics[
      width=\linewidth,
      height=0.22\textheight,
      keepaspectratio
    ]{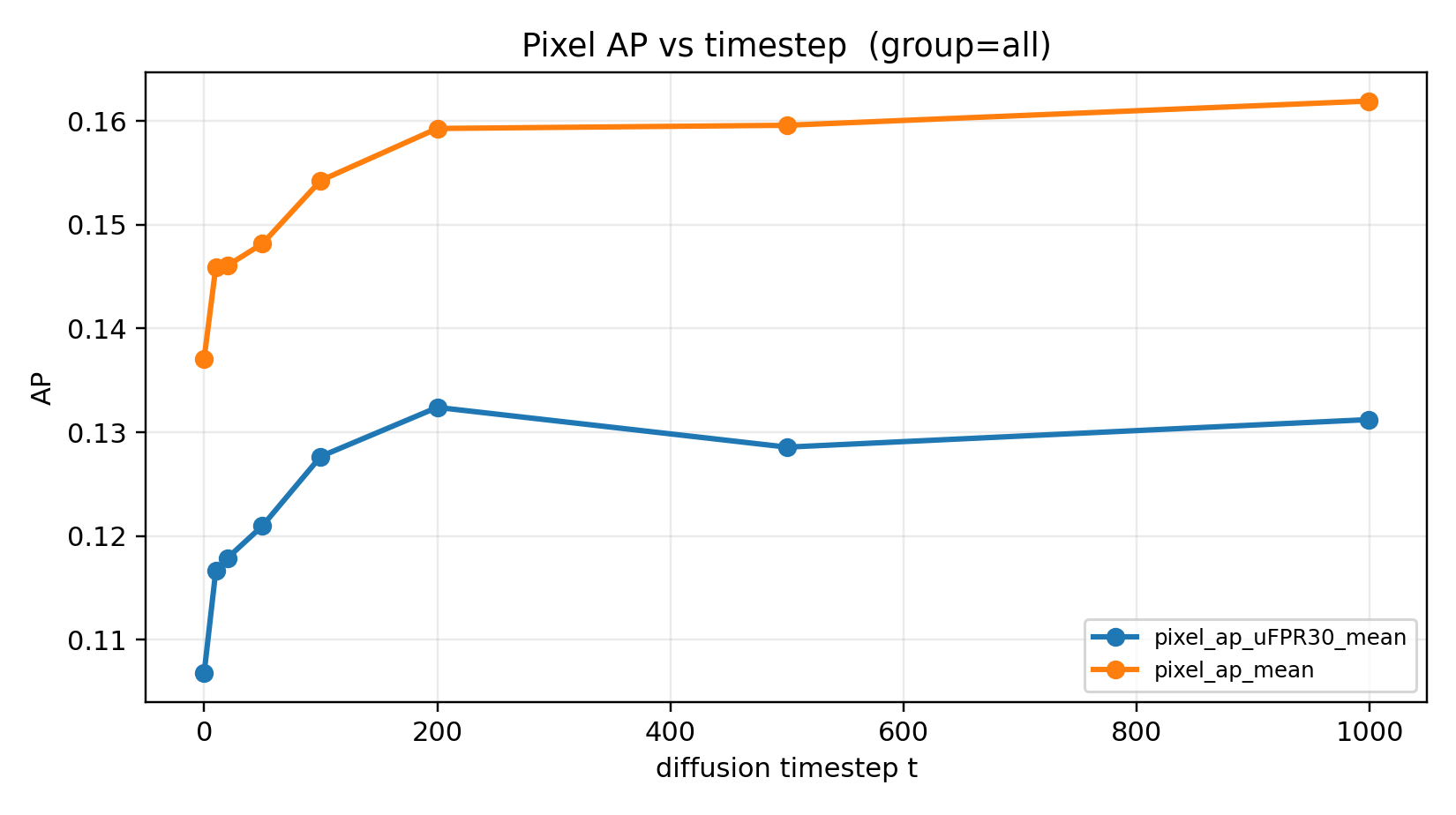}
    \caption{Pixel AP}
    \label{fig:pixel-ap-time step}
  \end{subfigure}

  \caption{
    \textbf{Effect of diffusion time step on per-pixel anomaly detection.}
    As the diffusion time step increases, ranking-based metrics (AUROC (Panel a), AUPRO (Panel b) )
    decrease, while average precision (Panel c) improves, indicating a shift from global
    separability to more confident, sparse anomaly responses. \gr{Orange:} 
    standard area under curve statistics, \gr{blue:} 
    additionally restricting the  False Positive Rate to below 30\% and thus promoting early separation.}
  \label{fig:mvtec-time step-metrics}
\end{figure}

\textbf{Observed behaviour.}
Figure~\ref{fig:mvtec-time step-metrics} shows a consistent and interpretable
trade-off.
As the diffusion time step increases, both pixel-level AUROC and AUPRO decrease,
indicating reduced global separability and weaker region-level coherence of the
anomaly scores.
In contrast, average precision increases monotonically with time step.
This divergence reflects a shift in the score distribution:
higher time steps smoothen the score field and suppress moderate responses, while
amplifying extreme deviations from the data manifold.
As a result, fewer pixels receive large scores, but those that do are more likely
to correspond to true defect regions.

\textbf{Interpretation.}
These results highlight that different metrics probe distinct aspects of
per-pixel anomaly detection.
AUROC and AUPRO assess global ranking quality and spatial consistency,
respectively, and therefore favour lower diffusion noise where fine-grained
structure is preserved.
Average precision, by contrast, emphasises the reliability of the
highest-scoring pixels and is sensitive to the purity of the score tail.
From this perspective, increasing the diffusion time step trades global
separability for more confident, sparse anomaly responses.
This behaviour is consistent with the geometry of diffusion score fields and
suggests that the time step can be selected based on the desired operating regime:
lower time steps for comprehensive localisation and higher time steps for
high-confidence defect highlighting.

\textbf{Practical Implications.}
This sensitivity analysis demonstrates that the proposed Stein
residual does not behave arbitrarily with respect to the score model, but
exhibits predictable and structured changes as the underlying notion of data
geometry is varied.
In all cases, the method remains fully zero-shot, requiring no retraining or
fine-tuning on the target dataset.

\section{Per-sample Stein residuals in mixed MNIST / Fashion-MNIST OOD}
\label{sec:mixed-ood}

We now assess the discriminative power of \emph{per-sample} Stein residuals in a
heterogeneous OOD scenario based on mixed in- and out-of-distribution test
sets, a setting that is known to be challenging for OOD detection and has been
used to expose failure modes of likelihood-based and confidence-based
approaches \citep{nalisnickDeepGenerativeModels2019,
rabanserFailingLoudlyEmpirical2019}.

\textbf{Setup.}
The classifier $f_\theta$ and score model $\hat{s}_p$ are trained exactly before.  Test sets are constructed by mixing MNIST
(in-distribution) and Fashion-MNIST (OOD) according to a corruption level
$\beta\in\{0.80,0.82,\dots,1.00\}$, where the test set contains $\beta\cdot 10{,}000$ MNIST images and $(1-\beta)\cdot 10000$ Fashion-MNIST images.

\textbf{Aggregate behaviour.}
For each $\alpha$ we compute the classifier accuracy and the mean Stein
signal obtained by aggregating the TASTE residuals.  As shown in Figure~\ref{fig:metrics-vs-alpha},
both quantities vary monotonically with the corruption level: accuracy declines
as Fashion-MNIST content increases, while the Stein score increases in magnitude.

\begin{figure}[t]
  \centering
  \includegraphics[width=\linewidth]{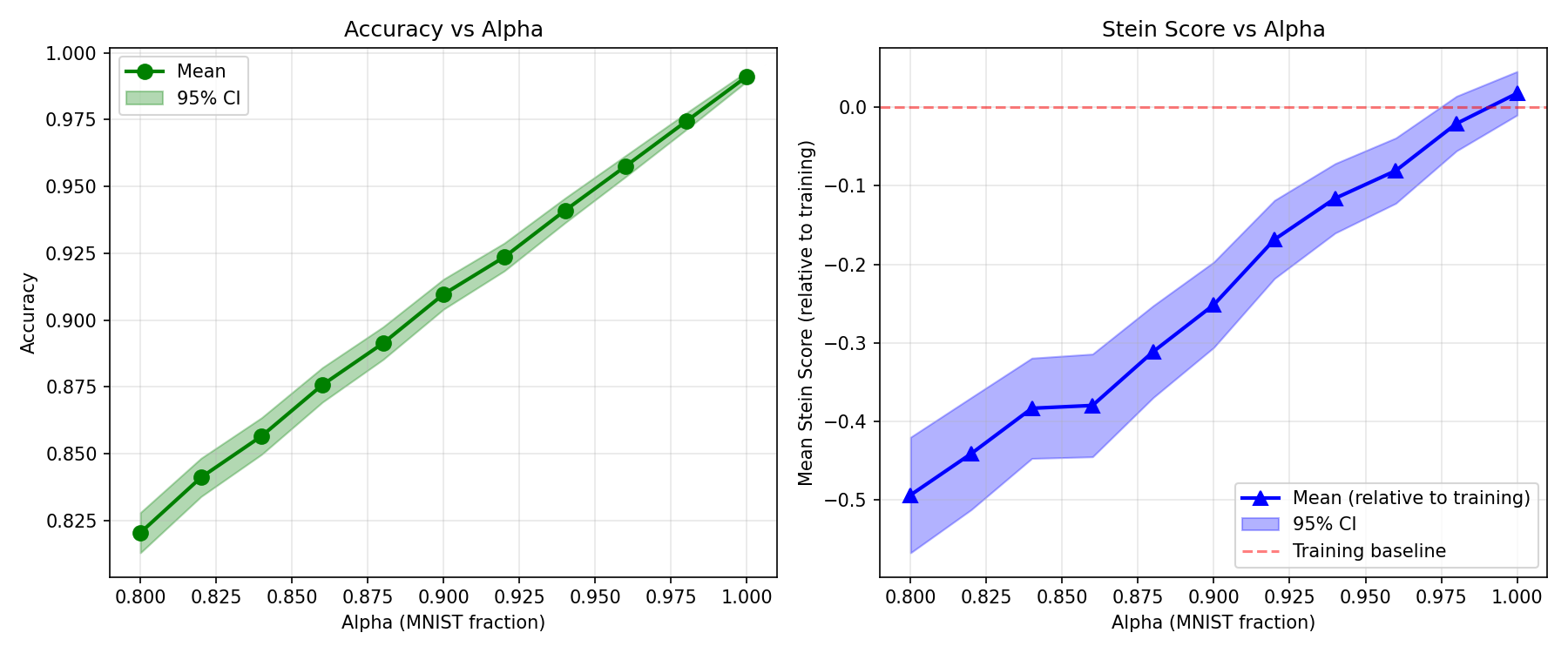}
  \caption{
    \textbf{Model performance and Stein score  as a function of test set corruption level.} The methodology picks up on the presence of out-of-distribution input data for even very small corruption levels. Shaded regions: $\pm 2$ s.d., as approximate 95\% confidence intervals.
  }
  \label{fig:metrics-vs-alpha}
\end{figure}

\begin{figure}[t]
  \centering
  \includegraphics[width=\linewidth]{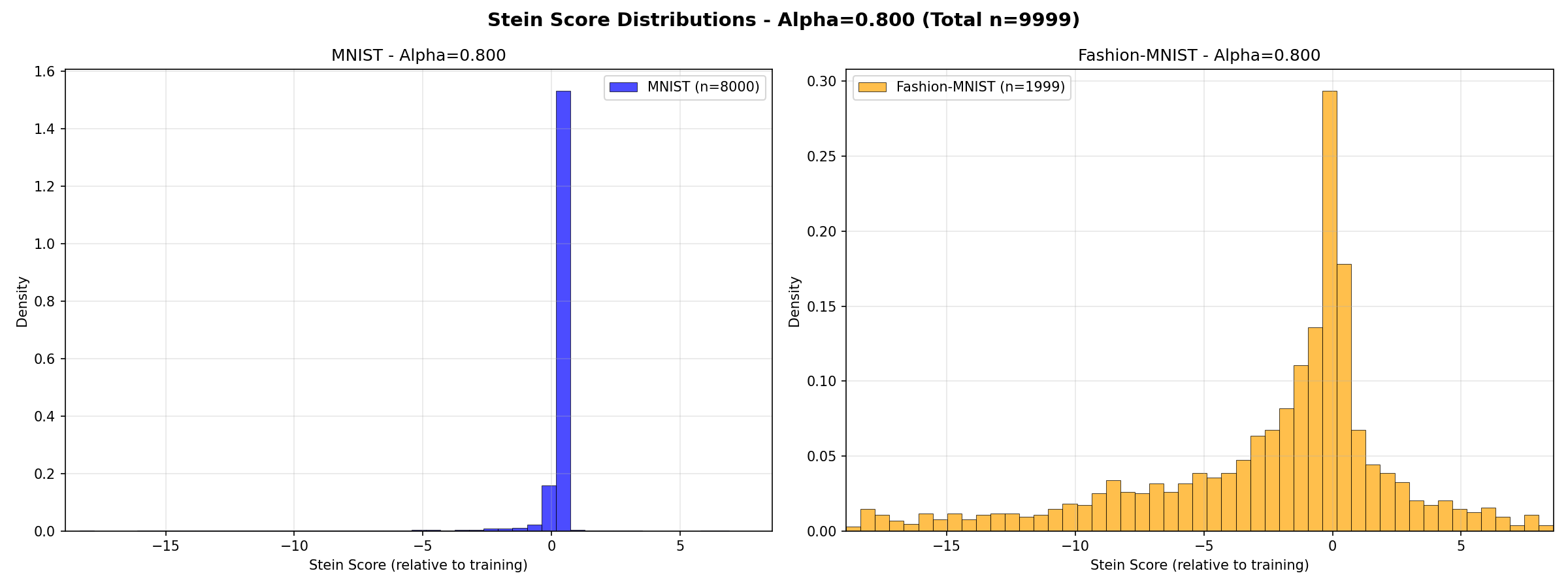}
  \caption{
    \textbf{Distribution of the per-sample TASTE residuals.} The left panel shows evaluation on clean (MNIST) samples, while the right presents the scores generated from corrupted data. The considerable distribution shift of the per-sample TASTE residuals allows for per-input anomaly detection. 
  }
  \label{fig:stein-histogram}
\end{figure}

\textbf{Per-sample separation.}
For a fixed corruption level (e.g.\ $\beta=0.8$), the distribution of
$r_f(x)$ is sharply concentrated near zero for MNIST and displays heavy tails
for Fashion-MNIST (Figure~\ref{fig:stein-histogram}).  This separation enables
a simple calibrated test: using held-out MNIST data, estimate the $95$th
percentile $\tau_{0.95}$ of $|r_f(x)|$ and declare a test point OOD if
$|r_f(x)|>\tau_{0.95}$.  The resulting test achieves high power across all
corruption levels (Figure~\ref{fig:hypothesis-testing}) despite requiring no
supervision on Fashion-MNIST.

\begin{figure}[t]
  \centering
  \includegraphics[width=\linewidth]{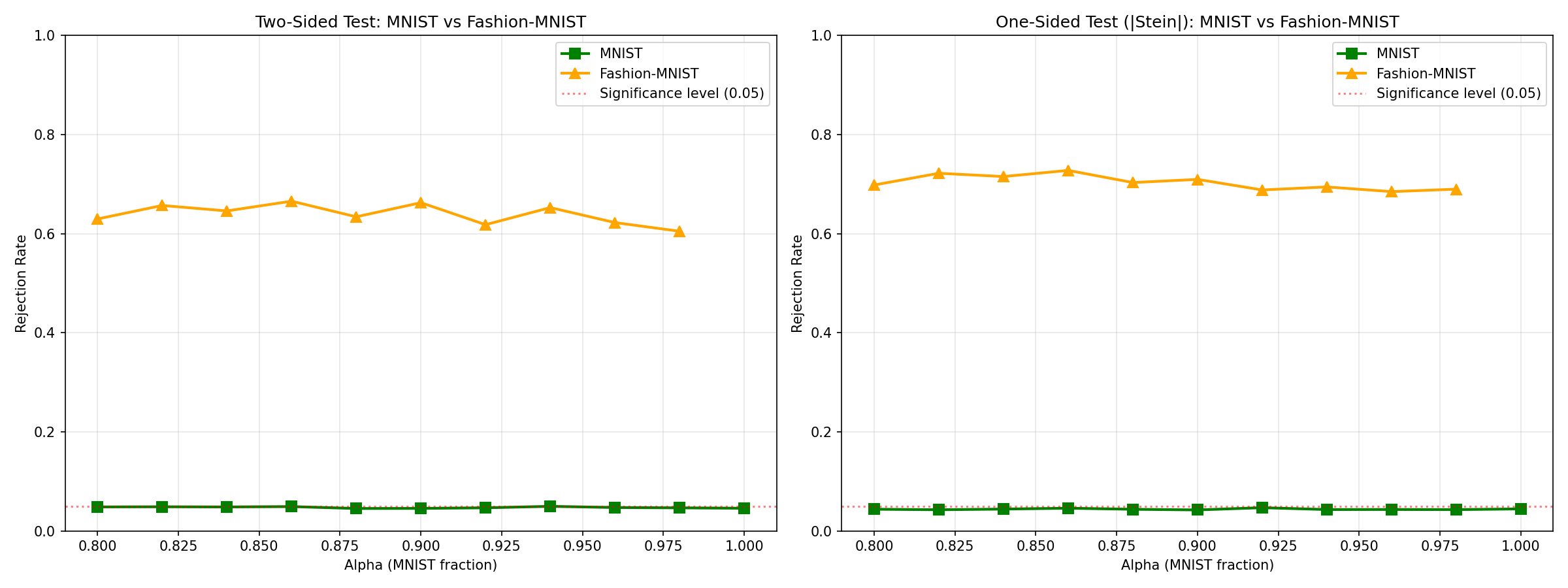}
  \caption{
    \textbf{Stein score-based test power comparison as a function of test set corruption level $\alpha$.} The left panel presents results for two-sided tests, while a one-sided (arguably more powerful) test is presented on the right. Also, see \textbf{Appendix~C} for power analysis based on a first-order operator
  }
  \label{fig:hypothesis-testing}
\end{figure}

\section{First-Order Stein Operator Diagnostics}
\label{app:first-order-stein}

For a function $f \in \mathcal{F}(p)$, instead of the Langevin operator $\mathcal{L}_p f(x)=\Delta f(x)+\nabla\log p(x)^\top\nabla f(x) \in \R$, consider a first-order Stein operator defined as
\[
\mathcal{A}_p f(x) = \nabla f(x) + f(x)s_p(x) \in \R^d.
\]
To obtain a scalar diagnostic one typically projects this vector onto a fixed
direction $v$, which we denote by
\begin{equation}\label{eq:first-order-stein-eq}
      \mathcal{A}^v_p f(X)
  =
  v^\top \nabla f(x) + f(x) v^\top s_p(x)
\end{equation}

It is easy to show that $\mathbb{E}_p[\mathcal{A}^v_p f(X)] = 0$ under mild assumptions on $f$. 
Indeed, let $X \sim p$ and assume $f \in \mathcal{F}(p)$ is such that the following integration by parts is valid (e.g., $f p$ vanishes at infinity).
Then
\begin{equation}\label{eq:first-order-stein-proof}
\mathbb{E}_p\!\left[\mathcal{A}^v_p f(X)\right]
=
\mathbb{E}_p\!\left[v^\top \nabla f(X)\right]
+
\mathbb{E}_p\!\left[f(X)\,v^\top s_p(X)\right].
\end{equation}

Since $s_p=\nabla \log p$, we have $f\,s_p = \nabla(fp) / p - \nabla f$, and thus
\begin{align*}
\mathbb{E}_p\!\left[f(X)\,v^\top s_p(X)\right]
= &
\int v^\top f(x)\nabla \log p(x)\,p(x)\,dx \\
= & 
\int v^\top \nabla (f(x)p(x))\,dx
-
\int v^\top \nabla f(x)\,p(x)\,dx.
\end{align*}
Under the assumed boundary conditions, the first integral vanishes, yielding
\[
\mathbb{E}_p\!\left[f(X)\,v^\top s_p(X)\right]
=
-
\mathbb{E}_p\!\left[v^\top \nabla f(X)\right].
\]
Substituting back into \eqref{eq:first-order-stein-proof}, the two terms cancel and we obtain
\[
\mathbb{E}_p[\mathcal{A}^v_p f(X)] = 0.
\]

\medskip
We now consider $q$, another smooth density on $\R^d$, absolutely continuous with respect to $p$. Mirroring the argument from the main text, we obtain

\[
\E_{q}[\mathcal{A}^v_p f(X)]
  =-\,v^\top \E_{q}
   \!\left[
      f(X)
      \nabla \log\frac{q(X)}{p(X)}
    \right].
\]

Indeed, taking expectation under $q$ yields
\[
\E_q[\mathcal{A}^v_p f(X)]
= v^\top \E_q[\nabla f(X)] + v^\top \E_q[f(X)\nabla\log p(X)].
\]
Using integration by parts under $q$ and the identity
\[
\E_q[\nabla f(X)] = -\,\E_q\!\left[f(X)\nabla\log q(X)\right],
\]
which holds for $f$ in the Stein class, we obtain
\[
\E_q[\mathcal{A}^v_p f(X)]
= -\,v^\top \E_q\!\left[f(X)\bigl(\nabla\log q(X)-\nabla\log p(X)\bigr)\right].
\]
Rearranging terms gives
$$
\E_{q}[\mathcal{A}^v_p f(X)]
  = -\,v^\top \E_q
   \!\left[
      f(X)\nabla\log\frac{q(X)}{p(X)}
    \right],
$$ 
as claimed.

Thus this functional detects only \emph{value-based alignment}: between the output $f(X)$ and the score mismatch $\nabla\log(q/p)(X)$. In
contrast to the Langevin operator, which probes geometric distortions via
$\nabla f(X)$ and $\Delta f(X)$, the first-order functional reacts primarily
to changes in the distribution of the scalar outputs $f(X)$. This might lead to unexpected behaviour.

\subsection{Directional blind spots: the 2D Gaussian illustration}

To illustrate the potential blind spots, consider the two-dimensional Gaussian mean-shift
model introduced in Section~\ref{sec:2d-rotation-experiment} of the main text
\[
p(x)=\mathcal{N}(0,I_2),
\qquad
q(x)=\mathcal{N}(\varepsilon u, I_2),
\qquad
u=(\cos\theta,\sin\theta).
\]
For the task $f(x)=x_2-x_1$, the scalar first-order functional admits the exact
expression
\[
\mathbb{E}_q[\mathcal{A}^v_p f(X)]
= -\varepsilon^2\bigl(v_1 \cos\theta + v_2 \sin\theta\bigr)\,(\sin\theta - \cos\theta).
\]
For $v=(1,1)$ this expression simplifies to 
\[
\mathbb{E}_q[\mathcal{A}^v_p f(X)]
=
\varepsilon^2 \cos(2\theta)
\]
which corresponds to taking the component-wise average of the operator.

Thus the response is \emph{quadratic} in the shift magnitude and vanishes for
all $\theta=\pi/4,3\pi/4,\dots$, creating entire undetectable families of
distributional shifts (see Figure~\ref{fig:first-order-metrics-vs-angle}).  By contrast, the Langevin Stein operator responds
linearly in $\varepsilon$ and exhibits no such degeneracies. 

\begin{figure}[H]
  \centering
  \includegraphics[width=\linewidth]{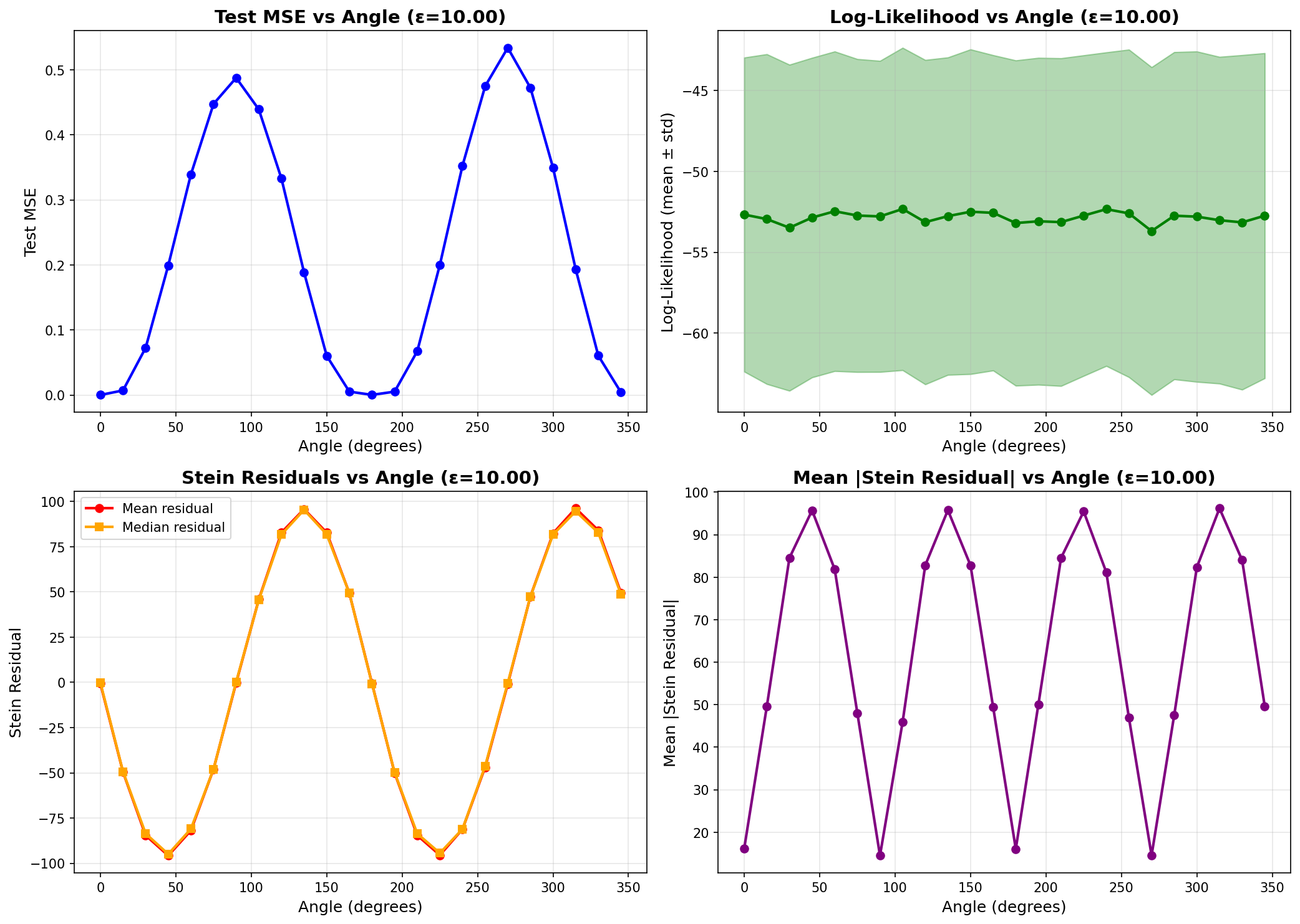}
  \caption{
    \textbf{Directional evaluation of model sensitivity and OOD scoring under the first-order Stein operator \eqref{eq:first-order-stein-eq}.}
    This figure should be compared with \textbf{Figure~\ref{fig:rotation-experiment-lines}}. The signal generated by the (analog of the) TASTE functional for the first-order Stein operator no longer aligns with performance.
  }
\label{fig:first-order-metrics-vs-angle}
\end{figure}

\subsection{The $L^2$ 
norm of the first-order Stein operator (no blind spots, but not a Stein operator)}

Another way to obtain a scalar quantity - and perhaps a more natural one - is to consider the squared $L^2$ norm of the operator under the 
test distribution \(q\):
\begin{equation}\label{eq:L2-full-first-order}
  \mathcal{T}_f(p,q)
  :=
  \E_{q}\!\left[\,
      \big\|\mathcal{A}_p f(X)\big\|_2^2
  \right]
  =
  \E_q\!\left[
     \|\nabla f(X) + f(X)s_p(X)\|_2^2
  \right].
\end{equation}

\textbf{No blind spots.}
Unlike the linear functional $\E_q[\mathcal{A}^v_p f]$, which can vanish for
large families of shifts (e.g.\ whole cones of directions in the 2D Gaussian
example), the quadratic quantity \eqref{eq:L2-full-first-order} is much harder
to drive to zero. Thus the $L^2$ norm of the full first-order field does not suffer from the
same directional ``blind spot'' issues as the projected linear functional.

\textbf{Nonlinearity and the loss of the Stein property.}
A key drawback is that $\|\mathcal{A}_p f\|_2^2$ is \emph{not} linear in $f$ or
in the operator. Therefore,
\[
  \E_{p}\!\left[\|\mathcal{A}_p f(X)\|_2^2\right]
  \neq 0
\]
in general. In fact, for any nontrivial $f$ in the Stein class of $p$,
\[
  \E_{p}\!\left[\|\mathcal{A}_p f(X)\|_2^2\right]
  =
  \Var_p\!\left(\mathcal{A}_p f(X)\right)
  > 0.
\]
Thus $\mathcal{T}_f(p,q)$ is \emph{not} a Stein operator: it does not vanish under
the null hypothesis that $q=p$ and therefore cannot be used directly for discrepancy testing. To convert it into a usable shift-sensitive statistic, we apply
the same principle used for correcting misspecified score models: subtract the
empirical baseline under the training distribution. Under $q=p$, this corrected statistic concentrates near zero, although this is more of a heuristic and does not have the same theoretical properties as a  proper Stein operator.   
Under $q\neq p$, we see that it typically becomes positive because the magnitude of the
first-order field changes whenever the distributional shift interacts with
either $f$ or $s_p$. 

Informally, unlike the linear first-order TASTE functional, the squared-norm
        statistic captures \emph{both magnitude and directional} changes in
        $\nabla f + f s_p$, thus eliminating blind spots.


\end{document}